\documentclass[leqno]{imsart}

\usepackage{fullpage}
\usepackage{amsfonts,amsmath,amsthm,amssymb,bm}
\usepackage[colorlinks,citecolor=blue,urlcolor=blue]{hyperref}
\usepackage[authoryear,round]{natbib}
\usepackage{derivative} 
\usepackage{graphicx}
\usepackage{subcaption}
\graphicspath{{figures/}}
\usepackage{enumitem}
\usepackage{float}

\renewcommand{\theenumi}{{\rm(\roman{enumi})}}
\usepackage{cleveref}
\Crefname{assumption}{Assumption}{Assumptions}
\usepackage[textsize=small]{todonotes}
\usepackage{mathrsfs}
\usepackage{booktabs}
\usepackage{appendix}

\usepackage{mathtools}
\DeclarePairedDelimiter\norm{\lVert}{\rVert}
\DeclarePairedDelimiter\fnorm{\lVert}{\rVert_{\rm{F}}}
\DeclarePairedDelimiter\opnorm{\lVert}{\rVert_{\rm{op}}}

\def\defas{\stackrel{\text{\tiny def}}{=}}

\def\Rem{{\mathrm{Rem}}}

\def\sign{{\rm sign}}
\def\var{{\rm var}}

\def\Rem{{\rm Rem}}

\newcommand{\limd}{{\smash{\xrightarrow{~\mathrm{d}~}}}}
\newcommand{\limp}{{\smash{\xrightarrow{~\mathrm{p}~}}}}

\def\argmin{\mathop{\rm arg\, min}}
\DeclareMathOperator{\prox}{prox}
\DeclareMathOperator{\soft}{soft}

\DeclareMathOperator{\trace}{Tr}

\let\vec\relax
\DeclareMathOperator{\vec}{\mathbf{vec}}

\DeclareMathOperator{\supp}{supp}

\usepackage{xspace} 
\def\iid{i.i.d.\@\xspace}
\def\eg{e.g.\@\xspace}
\def\ie{i.e.\@\xspace}

\def\R{{\mathbb{R}}}
\def\E{{\mathbb{E}}}
\def\N{{\mathbb{N}}}
\def\P{{\mathbb{P}}}

\def\mathbold{\boldsymbol} 
\def\cal{\mathcal} 

\def\ba{\mathbold{a}}

\def\bA{\mathbold{A}}
\def\hbA{\smash{\widehat{\mathbf A}}}

\def\bb{\mathbold{b}}
\def\hbb{{\widehat{\bb}}}\def\tbb{{\widetilde{\bb}}}
\def\bB{\mathbold{B}}
\def\hbB{{\widehat{\bB}}}

\def\hbC{{\widehat{\bfC}}}

\def\bC{\mathbold{C}}
\def\hbC{{\widehat{\bC}}}

\def\bD{\mathbold{D}}\def\calD{{\cal D}}

\def\bE{\mathbold{E}}

\def\bff{\mathbold{f}}

\def\bF{\mathbold{F}}

\def\bg{\mathbold{g}}
\def\tbg{{\widetilde{\bg}}}
\def\bG{\mathbold{G}}

\def\bH{\mathbold{H}}

\def\bI{\mathbold{I}}

\def\bJ{\mathbold{J}}
\def\calJ{{\cal J}}

\def\bK{\mathbold{K}}

\def\bL{\mathbold{L}}

\def\bM{\mathbold{M}}
\def\calM{{\cal M}}

\def\bN{\mathbold{N}}

\def\bP{\mathbold{P}}

\def\bQ{\mathbold{Q}}\def\tbQ{{\widetilde{\bQ}}}

\def\bR{\mathbold{R}}

\def\bS{\mathbold{S}}

\def\bT{\mathbold{T}}

\def\be{\mathbold{e}}

\def\bu{\mathbold{u}}

\def\bU{\mathbold{U}}

\def\bv{\mathbold{v}}
\def\tbv{{\widetilde{\bv}}}
\def\bV{\mathbold{V}}

\def\bw{\mathbold{w}}

\def\bW{\mathbold{W}}\def\hbW{{\widehat{\bW}}}

\def\bx{\mathbold{x}}

\def\bX{\mathbold{X}}
\def\calX{{\cal X}}

\def\by{\mathbold{y}}

\def\bY{\mathbold{Y}}

\def\bz{\mathbold{z}}
\def\tbz{{\widetilde{\bz}}}
\def\bZ{\mathbold{Z}}

\def\bGamma{\mathbold{\Gamma}}

\def\bDelta{\mathbold{\Delta}}

\def\ep{\varepsilon}\def\eps{\epsilon}
\def\bep{ {\mathbold{\ep} }}

\def\bzeta{\mathbold{\zeta}}

\def\bfeta{\mathbold{\eta}}

\def\btheta{\mathbold{\theta}}

\def\hbtheta{{\widehat{\btheta}}}

\def\lam{\lambda}

\def\bmu{\mathbold{\mu}}

\def\bXi{\mathbold{\Xi}}

\def\bPi{\mathbold{\Pi}}

\def\brho{\mathbold{\rho}}

\def\bSigma{\mathbold{\Sigma}}

\def\b0{{\mathbf{0}}}
\def\bd1{{\mathbf{1}}}

\newcommand{\boldzero}{\mathbf{0}}

\def\I{{\mathbb{I}}}

\usepackage{constants}

\startlocaldefs
\theoremstyle{plain}

\newtheorem{theorem}{Theorem}[section]
\newtheorem{lemma}[theorem]{Lemma}
\newtheorem{corollary}[theorem]{Corollary}
\newtheorem{proposition}[theorem]{Proposition}
\theoremstyle{remark}
\newtheorem{definition}{Definition}[section]
\newtheorem{assumption}{Assumption}[section]

\endlocaldefs

\AtBeginDocument{%
	\def\MR#1{}  
}
\theoremstyle{remark} 
\newtheorem{remark}{Remark}[section]

\numberwithin{equation}{section}

\begin{document}
\def\doi{}

\begin{frontmatter}
    \title{Uncertainty quantification for iterative algorithms 
    in 
    linear models
    with application to early stopping}
    \runtitle{Uncertainty quantification for iterative algorithms}
    
    \begin{aug}
    \author[A]{\fnms{Pierre C.}~\snm{Bellec}\ead[label=e1]{pierre.bellec@rutgers.edu}},
    \author[A]{\fnms{Kai}~\snm{Tan}\ead[label=e2]{kai.tan@rutgers.edu}},
    \address[A]{Department of Statistics,
    Rutgers University\printead[presep={,\ }]{e1,e2}}
    
    \end{aug}

    \begin{abstract}
        This paper investigates the 
        iterates $\hbb^1,\dots,\hbb^T$ obtained from iterative algorithms in high-dimensional linear regression problems, in the regime where the feature dimension $p$ is comparable with the sample size $n$, i.e., $p \asymp n$. 
        The analysis and proposed estimators are applicable
        to Gradient Descent (GD), proximal GD and their accelerated
        variants such as Fast Iterative Soft-Thresholding (FISTA).
        The paper proposes novel estimators for the generalization error
        of the iterate $\hbb^t$ for any fixed iteration $t$
        along the trajectory. These estimators are proved to be
        $\sqrt n$-consistent under Gaussian designs.
        Applications to early-stopping are provided:
        when the generalization error of the iterates is a U-shape
        function of the iteration $t$, the estimates allow to select
        from the data an iteration $\hat t$ that achieves the smallest
        generalization error along the trajectory.
        Additionally, we provide a technique for developing debiasing 
        corrections and valid confidence intervals for the components of the true coefficient vector from the iterate $\hbb^t$ at any finite iteration $t$. Extensive simulations on synthetic data illustrate the theoretical results.
\end{abstract}
    

    
    \begin{keyword}
    \kwd{high-dimensional models}
    \kwd{risk estimate}
    \kwd{Lasso}
    \kwd{generalization error}
    \end{keyword}
    
    \end{frontmatter}

\maketitle
\setcounter{tocdepth}{2}
\tableofcontents

\section{Introduction}
Consider the linear model
\begin{equation}\label{eq:linear-model}
    \by = \bX \bb^* + \bep,
\end{equation}
where $\by\in \R^n$ is the response vector, $\bX\in \R^{n\times p}$ is the design matrix with \iid rows with
covariance $\bSigma\in \R^{p\times p}$,
and $\bep\in \R^n$ is the noise vector with \iid entries from $\mathsf{N}(0,\sigma^2)$. 
To estimate the coefficient vector in the linear model with
dimension $p$ larger or comparable to the sample size $n$,
a common strategy is to consider
the penalized least-squares estimator
\begin{equation}\label{eq:b-hat}
    \hbb 
    \in \argmin_{\bb\in \R^p} \frac{1}{2n}\|\by - \bX \bb\|^2 + \rho(\bb).
\end{equation}
This estimator could be ordinary least-squares estimate (if $\rho(\bb) = 0$), 
Ridge regression, Lasso \citep{tibshirani1996regression}, Elastic Net \citep{zou2005regularization}, group Lasso \citep{yuan2006model}, nuclear norm
penalization \citep{koltchinskii2011nuclear} or
Slope \citep{bogdan2015slope}, to name a few convex examples.
Concave penalty functions $\rho$ in \eqref{eq:b-hat} have also been considered,
including
SCAD \citep{fan2001variable} and MCP \citep{zhang10-mc+}. 
The properties of the penalized estimator $\hbb$ in \eqref{eq:b-hat} are now well understood in high-dimensional regimes where the dimension
$p$ is larger than $n$. 
For instance, optimal L1 and L2 bounds of Lasso estimator $\hbb$ have been studied
extensively \citep[among others]{Wainwright09,bickel2009simultaneous,ye2010rate,bayati2012lasso,dalalyan2017prediction,bellec2016slope}.
Beyond risk bounds, confidence intervals
for single entries of $\bb^*$ have been developed in \cite{zhang2014confidence,JavanmardM14a,GeerBR14}, and more recently in
\cite{cai2017confidence,javanmard2018debiasing,bellec_zhang2019debiasing_adjust}.
We refer readers to \cite{buhlmann2011statistics,van2016estimation} for a comprehensive review of such results. 
In most results cited above, $p\gg n$ and the L2 risk of \eqref{eq:b-hat}
typically converges to 0, \ie, $\hbb$ is consistent.
Consistency, on the other hand, is not granted in the regime
where dimension and sample size are of the same order, \ie,
\begin{equation}
    \label{eq:proportional_regime}
    p \asymp n,
    \qquad \text{ or }
    \qquad 
     \lim (p/n)
    \text{ exists and is finite}.
\end{equation}
In this regime, \cite{bayati2012lasso,miolane2018distribution}
showed that the Lasso L2 risk,
under suitable assumptions, converges to a finite positive constant
that can be characterized by a system of nonlinear equations.
Regime \eqref{eq:proportional_regime} will be in force in the present paper.
Since the early works
\citep{bayati2012lasso,el_karoui2013robust,donoho2016high}
on this proportional regime,
significant progress was made
to develop a unified theory explaining the behaviors of M-estimators
and penalized estimators when \eqref{eq:proportional_regime} holds, see for instance \cite{thrampoulidis2018precise,celentano2020lasso,loureiro2021capturing}.
Compared with the early debiasing literature \citep{ZhangSteph14,GeerBR14,JavanmardM14a},
new debiasing techniques are needed to construct confidence intervals
for components of $\bb^*$ in this regime, in particular with the need
to account for the degrees-of-freedom of the estimator \eqref{eq:b-hat}
for the debiasing correction
\citep{bellec2022biasing,celentano2020lasso,bellec2023debiasing}.

\subsection{Iterative algorithms}
As we have seen,
the properties of the estimator $\hbb$ in \eqref{eq:b-hat}
are now well understood, both in the classical high-dimensional regime with $p\gg n$  
and in the proportional regime \eqref{eq:proportional_regime}.
In practice, an exact minimizer $\hbb$ of \eqref{eq:b-hat} is
typically unavailable in closed-form, and numerical approximations are used
instead.
To solve the minimization problem \eqref{eq:b-hat} numerically, the practitioner
resorts to iterative algorithms
developed in the optimization literature
to obtain an approximate solution to \eqref{eq:b-hat}.

For smooth objective functions,
this includes first-order methods such as
Gradient Descent (GD), 
its accelerated version known as Accelerated Gradient Descent (AGD) \citep{nesterov1983method}, 
or
its randomized variants \citep{bottou2010large}.
For non-smooth objective functions, the practitioner typically
resorts to coordinate descent
\citep{friedman2010regularization} or iterative algorithms
involving the proximal operator of $\rho$ such as
Iterative Soft-Thresholding Algorithm (ISTA)
\citep{Daubechies2004} and its accelerated version Fast Iterative Soft-Thresholding Algorithm (FISTA)
\citep{beck2009fast}.
Such iterative algorithm starts with an initializer $\hbb^1$ and then iteratively produces iterates $\hbb^2, \hbb^3, ...$ by applying a nonlinear
function to a weighted sum of the previous iterate and the gradient of the first term
in \eqref{eq:b-hat} (the smooth part of the objective function) at the previous iterate.
For instance,
for a convex penalty function $\rho$ in \eqref{eq:b-hat} 
the proximal gradient descent algorithm solves \eqref{eq:b-hat} 
with the iterations
\begin{equation}
    \label{proximal_gradient_intro}
    \hbb^{t+1} = \prox[\tfrac1L\rho]\bigl(\hbb^{t} + \tfrac{1}{nL}\bX^\top (\by - \bX\hbb^{t})\bigr){,}
\end{equation}
where the proximal operator is defined as $\prox[\frac1L\rho](\bv)=\argmin_{\bb\in\R^p}\|\bb-\bv\|^2/2 + \frac1L\rho(\bb)$ and
the scalar $L>0$ is usually taken \citep{beck2009fast} as an upper bound
on the Lipschitz constant of the gradient of the first term in the objective
function \eqref{eq:b-hat}, \ie, such that $L\ge \|\bX^\top\bX/n\|_{\rm op}$.
For the Lasso problem with $\rho(\bb) = \lambda \|\bb\|_1$ in \eqref{eq:b-hat},
the iterations \eqref{proximal_gradient_intro} become the ISTA iterations
\begin{equation}
    \label{soft_intro}
    \hbb^{t+1} = \soft_{\lambda/L}\bigl(\hbb^{t} + \tfrac{1}{nL}\bX^\top (\by - \bX\hbb^{t})\bigr) {,}
\end{equation}
where $\soft_{\lambda/L}(\cdot)$  applies the soft-thresholding $u\mapsto \sign(u)(|u|-\lambda/L)_+$ componentwise. 
For AGD and FISTA,
the recursion formula defining $\hbb^{t+1}$ involves the iterates 
$(\hbb^{t},\hbb^{t-1})$ as we will detail in 
\Cref{sec:eg-AGD,sec:eg-FISTA} below.

More generally, the focus of the present paper is on iterative algorithm of
the form
\begin{equation}
    \hbb^{t+1} = \bg_{t+1}
    (\hbb^{t}, \hbb^{t-1}, \bv^t, \bv^{t-1}),
    \quad \text{ where}
    \quad
    \bv^t = \tfrac1n \bX^\top(\by - \bX\hbb^t)
    \label{recursion_intro_2_previous_iterates}
\end{equation}
for nonlinear Lipschitz functions $\bg_{t+1}$, in order to include
accelerated algorithms from the
optimization literature that 
are of the form \eqref{recursion_intro_2_previous_iterates}, i.e.,
using the previous two iterations,
including AGD \citep{nesterov1983method}
and FISTA
\citep{beck2009fast}.
Our theory actually applies to recursions of the form
\begin{equation}
    \hbb^{t+1} = \bg_{t+1}
    (\hbb^{t}, \hbb^{t-1},\dots,\hbb^2,\hbb^1, ~  \bv^t, \bv^{t-1},\dots,\bv^2,\bv^1),
    \label{intro_iterates_general_form}
\end{equation}
although the typical optimization algorithms mentioned in the previous
paragraph only use the previous two iterates, as in 
\eqref{recursion_intro_2_previous_iterates}.

These iterative algorithms are typically stopped when
a stopping criterion is met. 
Existing theories in the optimization literature on iterative algorithms examine their convergence properties regarding the training error, specifically how fast (in $t$) does 
the objective function at the iterate $\hbb^t$ approach the minimal
value in \eqref{eq:b-hat}, or how fast does
the iterate $\hbb^t$ approach an ideal minimizer
$\hbb$ of \eqref{eq:b-hat} as $t\to+\infty$.
For instance, \cite{beck2009fast} showed that FISTA enjoys a faster rate of convergence
compared to ISTA.

\subsection{Statistical properties of iterates for small $t$ or when convergence fails}
If $\hbb^t$ converges to $\hbb$ as $t$ increases, it is reasonable to expect that the properties of $\hbb$ (\eg, risk bounds, or debiased confidence intervals for components of $\bb^*$) are also valid for $\hbb^t$ provided $t$ is sufficiently large.
One motivation of the present paper is the study of iterates $(\hbb^t)_{t\ge 1}$ for small values of $t$, or when convergence of these iterates 
fail for the optimization problem \eqref{eq:b-hat}.
In such cases, because $\hbb^t$ and an ideal solution $\hbb$ to \eqref{eq:b-hat} significantly differ from each other, we should not expect that
the statistical properties of $\hbb$ are inherited by $\hbb^t$.
This calls for a statistical analysis of $\hbb^t$ itself, rather
than ideal minimizers $\hbb$ of the optimization problem \eqref{eq:b-hat}.
Two concrete situations for which convergence fails are the following.

\begin{itemize}
    \item The dataset $(\bX,\by)$ is so large that a single iteration
        of the recursion \eqref{intro_iterates_general_form}
        is slow and/or computationally costly.
        Running the iterations \eqref{intro_iterates_general_form}
        until convergence ($t\to+\infty$) is then not realistic.
        Still,
        in such situation, the practitioner may only run a few iterations and use
        $\hbb^t$ for a small $t$ (\eg, $t=5$ or $t=10$) for statistical purposes,
        including making predictions
        at new test data points $\bx_{new}$.
    \item If the optimization problem \eqref{eq:b-hat} is non-convex,
    as is the case with SCAD \citep{fan2001variable} 
        or MCP \citep{zhang10-mc+}, convergence of the iterates
        to a local or global solution to \eqref{eq:b-hat} is a subtle
        problem. Iterative algorithms may converge to a local optima.
        Specific algorithms have been developed to ensure good
        statistical properties under
        suitable assumptions, such as Restricted Eigenvalue conditions
        \citep{zhang10-mc+,feng2017sorted}.
        However, such conditions 
        are often not verifiable in practice \citep{bandeira2012certifying}.
        Despite the absence of any convergence guarantees, 
        the algorithm used to solve \eqref{eq:b-hat} for
        SCAD or MCP is still implementable, yielding iterates in the form of 
        \eqref
        {intro_iterates_general_form}.
        Despite the lack of verifiable conditions to guarantee convergence, 
        at a given $t$
        the iterate $\hbb^t$ can still be used for downstream tasks,
        for instance predictions at test data points $\bx_{new}$.
\end{itemize}

In both situations, convergence does not occur and statistical
properties of $\hbb^t$ at a finite $t$ are expected to significantly departs
from those of $\hbb$.
Using $\hbb^t$ for statistical purposes, instead of waiting for
convergence, is also appealing in order to save computational resources.
This calls for a statistical theory of $\hbb^t$ itself, rather
than ideal minimizers $\hbb$ of \eqref{eq:b-hat}.
The present makes contributions in this direction with two major applications on the
properties of the iterate $\hbb^t$ at a fixed $t$:
assessing the predictive performance of $\hbb^t$, and constructing debiased confidence
intervals for components of $\bb^*$ based on $\hbb^t$.

\subsection{Early stopping}
It is important to recognize that the predictive performance of $\hbb^t$ does not necessarily improve as $t$ grows. 
In fact, early-stopped gradient descent has been found to be effective in several machine learning applications, including non-parametric regression in a reproducing kernel Hilbert space \citep{raskutti2014early}, boosting \cite{Buhlmann2003boosting,Rosset2003Boosting}
and ridge regression \citep{ali2019continuous}, to name a few. 
The intuition is that early stopping provides a form of regularization
that prevents overfitting
(see, e.g., \cite{ali2019continuous} and the references therein).
In these settings, the population risk curve versus the iteration number is U-shaped (first decreasing and then increasing); we will observe instances of this
phenomenon in Figures~\ref{fig:LSE-GD}-\ref{fig:MCP-LQA-1500}.
For practical purposes, the major question in such situations is to
infer from the data if early stopping is beneficial, and when to stop
the iterations. The practitioner typically wishes to stop the iteration as soon
as the population risk starts to increase; this is only possible
if an estimator of the population risk of $\hbb^t$ is available.
Our results below provide such an estimator that can be effectively computed
from the data $(\bX,\by)$.

\subsection{Contributions}
The aforementioned situations motivate studying the properties of the iterates $\hbb^t$ at each step $t$, in particular for smaller values of $t$ where $\hbb^t$ is not yet close to the optimizer $\hbb$. 
Our focus centers on two key questions:
\begin{itemize}
    \item[(Q1)] Can we quantify the predictive performance of $\hbb^t$ for each $t$ along the trajectory?
    \item[(Q2)] Can we use the iterates $\hbb^t$ to construct confidence intervals for the entries of $\bb^*$, 
        for instance by deriving asymptotic normality results around $\hbb^t$?
\end{itemize}
This paper provides affirmative answers to these questions.
Our contributions are summarized as follows:
\begin{itemize}
    \item[i)] We introduce a reliable method for estimating the generalization error of each iterate $\hbb^t$ obtained from widely-used algorithms that solve \eqref{eq:b-hat}. This estimator takes the form
        $\frac1n\|\sum_{s\le t}\hat w_{s,t}(\by - \bX\hbb^s)\|^2$,
        \ie,
        the mean square error of the weighted average of the residual vector $\{\by - \bX\hbb^s\}_{s\le t}$ with
        carefully chosen data-driven weights $\hat w_{s,t}$ defined in \Cref{thm:rt} below.
    These weights are algorithm-specific and are computed from the data. 
    The theory also allows us to derive estimators for the generalization
    error of weighted averages of the iterates. 
        For instance, our approach enables the estimation of the generalization error for the average $\frac{1}{T}\sum_{t=1}^T\hbb^t$. 
        The corresponding theory is presented in \Cref{thm:generalization-error}.
    \item[ii)] The proposed estimator of the generalization error serves as a practical proxy for minimizing the true generalization error. To identify the iteration index $t$ such that $\hbb^t$ has the lowest generalization error, one can choose the $t$ that minimizes the estimated generalization error. This approach offers a viable method for pinpointing the optimal stopping point in the algorithm's trajectory. This application is presented in \Cref{cor:early}.
    \item[iii)] We develop a method for constructing valid confidence intervals for the elements of $\bb^*$ based on a given iteration $t$ and
    iterate $\hbb^t$. Consequently, one needs not wait for the convergence of the algorithm to construct valid confidence intervals for $\bb^*$. In cases where early-stopping is beneficial for achieving smaller generalization error, confidence intervals based on the early-stopped iterate $\hbb^{\hat t}$ (where $\hat t$ is chosen by minimizing the estimated risk) are shorter than the confidence interval based on the fully converged minimizer $\hbb$.
The theory on asymptotic normality
and confidence intervals 
is provided in \Cref{thm:debias}.
    \item[iv)] The proposed method is applicable to a broad spectrum of regression techniques and algorithms. 
        We provide detailed expressions of the weights $\hat w_{s, t}$ for several popular algorithms in \Cref{sec:examples}.
    It offers a practical solution for determining the optimal stopping iteration of the algorithm to enjoy the smallest generalization error as 
    a function of $t$ along the
    trajectory.
    Extensive simulation studies on synthetic data corroborate our theoretical findings. 
    These studies demonstrate the efficiency of our proposed risk estimator for $\hbb^t$, and confirm the validity of the asymptotic normality results for the entries of $\bb^*$ derived from $\hbb^t$.
    The simulations are presented and discussed in 
    of \Cref{sec:examples}.
\end{itemize}

\subsection{Related literature}
For regression problems in the proportional regime \eqref{eq:proportional_regime}, the most extensively studied iterative algorithm is Approximate Message Passing (AMP) \citep{donoho2009message,bayati2011dynamics,bayati2012lasso}. In the current framework, AMP can be formulated as $
\hbb^{t+1}=\bfeta_{t+1}(\hbb^t + \sum_{s\le t}w_{t,s}^{\textsc{amp}} \bv^s)$, where $w_{t,s}^{\textsc{amp}}$ are specific scalar weights. These weights ensure that at each step, the input of $\bfeta_{t+1}$ is approximately normally distributed.
This construction enables tracking the evolution of the error of $\hbb^t$ as a function of $t$ using a simple recursion, referred to as state evolution.
AMP has found numerous applications, including for instance
in high-dimensional statistics \citep{bayati2012lasso,celentano2019fundamental},
spiked models/low-rank matrix denoising
\citep[and references therein]{montanari2021estimation}
or sampling
\citep{el2022sampling}.
A proof of the state evolution for non-separable
$\bfeta_{t}$ is given in \cite{berthier2020state}.
The focus of the present is to study algorithms from the optimization 
literature such as
proximal gradient \eqref{proximal_gradient_intro}, AGD or ISTA/FISTA
that do not use the special AMP weights.

\cite{celentano2020estimation,montanari2022statistically} exhibit
fundamental limits for the performance of iterations that encompass \eqref{intro_iterates_general_form}. These works show that
iterates can be viewed as a Lipschitz function $g_*$ of some AMP algorithm. They 
characterize the fundamental lower bound obtained by taking
the infimum over all $g_*$ and corresponding AMP algorithms,
and show that the lower bound is attained by a specific AMP algorithm, termed Bayes-AMP.
Our goal in this paper is different: to develop a data-driven estimate of the error of the iterates $\hbb^t$.

Still in the proportional regime,
\cite{chandrasekher2022alternating,chandrasekher2023sharp,lou2024hyperparameter} 
consider iterations that,
in order to compute the $(t+1)$-th iterate,
use a fresh random batch
$(\bx_i^{t+1},y_i^{t+1})_{i \in N^{t+1}}$ independent of all the past
(here, the past refers to $\hbb^t,\hbb^{t-1},...\hbb^1$ as well as the random samples used to compute
these).
These works characterize the evolution of the iterates' performance (measured by mean square error, or inner product with the ground truth) by a sequence of deterministic scalars defined by
simple recursions, and they show that the iterates' performance
is close to the deterministic scalars.
If samples are reused across iterations, as in the present paper,
it is still possible to characterize deterministic equivalents of the
performance of the iterates in some cases
(see \cite{celentano2021high} for continuous-time iterates,
\cite{gerbelot2022rigorous} for discrete-time ones).
The deterministic equivalents of the iterates' performance are necessarily more complex than in the fresh random batch setting. This complexity is captured by limiting Gaussian processes indexed by $t \in [T]$ \citep{celentano2021high,gerbelot2022rigorous}.

For optimization of the least-squares problem, with $\rho(\cdot)=0$
in \eqref{eq:b-hat}, \cite{ali2020implicit,ali2019continuous,sheng2022accelerated} (among others)
study equivalences between early stopping and L2 regularization:
stopping the iterations of GD or AGD
at a fixed $t$ implicitly performs shrinkage
of a similar nature to Ridge regression.
For gradient descent iterations solving this least-squares problem,
\cite{patil2024failures} develops estimates of the generalization error
of the iterates using leave-one-out cross-validation, and show 
consistency of these leave-one-out estimates along the GD trajectory.
While leave-one-out estimates are close to the goal of the present paper,
\ie, to develop estimates of the generalization error along the
trajectory of the algorithm, they are
not computationally practical without further modifications since they require
to running $n$ algorithms in parallel to obtain the leave-one-out
estimate.

Outside of the proportional regime \eqref{eq:proportional_regime},
\cite{luo2023iterative} develop an efficient approximation of leave-one-out 
cross validation for
iterates $\hbb^t$ solving empirical risk minimization problems with additive regularization.
The normalized Hessian of the objective function is assumed to be Lipschitz
and non-singular at the iterates, which precludes settings
of interest for the present paper with $p>n$
and non-smooth penalty in \eqref{eq:b-hat}.
For uncertainty quantification of the iterates, 
\cite{hoppe2023uncertainty} 
extends the classical debiased lasso estimator \citep{ZhangSteph14,JavanmardM14a,van2014asymptotically} to the iterates of a variant of ISTA. 
The asymptotic normality in 
\cite{hoppe2023uncertainty}
requires a bound on the size of the supports of
$\hbb^t$ and $\bb^*$, as well as $\norm{\hbb^t - \bb^*}$ to vanish as $n\to\infty$.
This is not the case in the regime \eqref{eq:proportional_regime} of interest here where $p$ and $n$ are of the same order.
The present paper departs from this analysis of a variant ISTA
on several fronts. First, our analysis is grounded in a more general iteration formula (see 
\eqref{recursion_intro_2_previous_iterates} and  \eqref{intro_iterates_general_form}), which encompasses a broader spectrum of algorithms. 
Second, our results establish asymptotic normality of $\hbb^t$ for 
any $t$, and do not require $\norm{\hbb^t - \bb^*}$ to vanish as $n\to\infty$.
In contrast to these studies, the present paper introduces new estimator
of the generalization error of algorithm iterates that solve the regression problem \eqref{eq:b-hat} for $p>n$ and with non-smooth penalty functions.
Additionally, we propose a method to construct confidence intervals for the entries of $\bb^*$ by debiasing the iterate $\hbb^t$.

\subsection{Notation}
Regular variables like $a, b, ...$ refer to scalars, bold lowercase letters such as $\ba, \bb, ...$ represent vectors, and bold uppercase letters like $\bA, \bB, ...$ indicate matrices.
For an integer $n\ge 1$, we use the notation
$[n] = \{1, \ldots, n\}$. 
The vectors $\be_i\in \R^n, 
\be_j\in\R^p, \be_t\in\R^T$ denote the canonical basis vector of the corresponding index. 
For clarity, we always use implicit index $i$ to loop or sum over $[n]$,
index $j$ to loop over $[p]$, and 
indices $t,t',s$ to loop over $[T]$.
For a real vector $\ba \in \R^p$,  
$\norm*{\ba}$ 
denotes its Euclidean norm. 
For a matrix $\bA$,  
$\fnorm*{\bA}$, 
$\opnorm*{\bA}$, 
$\norm*{\bA}_*$ 
denote its Frobenius, operator and nuclear norm, respectively. 
For a matrix $\bA$, we denote its maximum
and minimum singular values by $\sigma_{\max}(\bA)$ and $\sigma_{\min}(\bA)$, respectively. If $\bA$ is symmetric,
$\lambda_{\max}(\bA)$ and $\lambda_{\min}(\bA)$ are its maximum and minimum eigenvalues.
Let $\bA\otimes \bB$ be the Kronecker product of matrices $\bA$ and $\bB$. 
Let ${\bf 1}_n$ denote the all-ones vector in $\R^n$, and $\bI_n$ denote the identity matrix of size $n$. 
For an event $E$, $\I(E)$ denotes the indicator function of $E$. It takes the value 1 if the event $E$ occurs and 0 otherwise. 
Let $\mathsf{N}(\mu, \sigma^2)$ denote the Gaussian distribution with mean $\mu$ and variance $\sigma^2$, and $\mathsf{N}_k(\bmu, \bSigma)$ denote the $k$-dimensional Gaussian distribution with mean $\bmu$ and covariance matrix $\bSigma$.
For a random sequence $\xi_n$, we write $\xi_n = O_P(a_n)$ if $\xi_n/a_n$ is stochastically bounded,
and $\limp$ for convergence in probability and
$\limd$ for convergence in distribution.
We reserve the letters $c$ and $C$ to denote generic constants.
Additionally, we use $C(\zeta, T, \gamma, \kappa)$ to denote a positive constant that depends only on $\zeta, T, \gamma, \kappa$.
The exact values of these constants may vary from place to place.

\section{Main results}
\label{sec:main}
Throughout the paper, given iterates $\hbb^1,...,\hbb^t,...$, define 
$\bv^t\in\R^p$ by
\begin{equation}
    \label{v_t}
    \bv^t = n^{-1} \bX^\top (\by - \bX\hbb^t) \qquad \text{ for } t\ge 1. 
\end{equation}

\subsection{Iterates and derivatives}
Our first task is to establish some notation
(namely, matrices $\calJ,\calD\in\R^{pT\times pT}$) that will be
used to construct the proposed estimators. 
As a warm-up, we will first introduce this notation for proximal gradient 
descent (\Cref{subsubsec:PGD})
and iterates constructed from the two previous ones
(\Cref{subsubsec:two}).
Notation for the general form \eqref{intro_iterates_general_form}
will be given in \Cref{subsubsec:general}.

\subsubsection{Proximal gradient descent}
\label{subsubsec:PGD}
Consider the proximal gradient descent iterates
\eqref{proximal_gradient_intro}.
With this notation and starting with $\hbb^1=\bf0$,  the proximal gradient descent
\eqref{proximal_gradient_intro} iterations can be rewritten as 
\begin{align}\label{eq:pgd-iterates}
    \hbb^{t} = \bg_t(\hbb^{t-1},\bv^{t-1}),
    \quad \text{ where }
    \bg_t(\bb^t,\bv^t) = \prox[\tfrac1L\rho](\bb^t + \tfrac1L \bv^t).
\end{align}
The Jacobian of the function $\bg_t$ is a matrix of size
$p\times 2p$ and can be partitioned into the two $p\times p$ blocks
\begin{equation}\label{eq:pgd-J-D}
    \bJ_{t,t-1}= \pdv{\bg_{t}}{\bb^{t-1}}(\hbb^{t-1}, \bv^{t-1}), \quad
    \bD_{t,t-1} = \pdv{\bg_{t}}{\bv^{t-1}}(\hbb^{t-1}, \bv^{t-1}).
\end{equation}
Next, define $\calJ\in\R^{pT\times pT}$ and $\calD \in \R^{pT\times pT}$ 
with $T\times T$ blocks of size $p\times p$ as follows
\begin{equation}\label{eq:pgd-cal-J-D}
\calJ = 
\begin{bmatrix}
    \b0 & \b0 & \cdots & \cdots & \cdots& \b0\\
    \bJ_{2,1} & \b0 & \ddots & \cdots & \cdots& \vdots\\
    \b0 & \bJ_{3,2} & \b0 & \ddots & \cdots& \vdots\\
    \b0 & \b0 & \bJ_{4,3} & \b0 & \ddots& \vdots\\
    \vdots & \ddots & \ddots & \ddots & \ddots& \b0\\
    \b0 & \cdots &\b0 & \b0 & \bJ_{T, T-1} &\b0\\
\end{bmatrix}
,~
\calD = 
\begin{bmatrix}
    \b0 & \b0 & \cdots & \cdots & \cdots& \b0\\
    \bD_{2,1} & \b0 & \ddots & \cdots & \cdots& \vdots\\
    \b0 & \bD_{3,2} & \b0 & \ddots & \cdots& \vdots\\
    \b0 & \b0 & \bD_{4,3} & \b0 & \ddots& \vdots\\
    \vdots & \ddots & \ddots & \ddots & \ddots& \b0\\
    \b0 & \cdots &\b0 & \b0 & \bD_{T, T-1} &\b0
\end{bmatrix}.
\end{equation}
Each $\boldzero$ block in the above $\calJ,\calD$ is the zero matrix of
size $p\times p$.
For iterations depending not only on $\hbb^{t-1},\bv^{t-1}$ but
also on previous iterates, the lower triangular blocks
of $\calD,\calJ$ will be filled with the corresponding non-zero
blocks of the Jacobian of $\bg_t$, as shown in the next subsections.

\subsubsection{Combining two previous iterates}
\label{subsubsec:two}
To incorporate more general algorithms beyond proximal gradient descent, we consider the iteration formula: 
\begin{align}\label{eq:general-form}
    \hbb^{t} = \bg_{t}(\hbb^{t-1}, \hbb^{t-2}, \bv^{t-1}, \bv^{t-2})
    \qquad \text{ for } t\ge 2. 
\end{align}
Here, $\hbb^1$ serves as an initial vector, and we set $\hbb^0 = \bv^0 = \b0_p$.
The iteration functions $\bg_t(\cdot): \R^{4p} \to \R^p$ are determined by user-specified algorithms. 
We include $\hbb^{t-2}$ and $\bv^{t-2}$ into the argument of \eqref{eq:general-form} to encompass optimization algorithms that update iterates by considering information from the two preceding steps. 
This general formulation accommodates popular algorithms such as Nesterov's accelerated gradient (AGD) method \citep{nesterov1983method,nesterov2003introductory} and the Fast Iterative Shrinkage-Thresholding Algorithm (FISTA) \citep{beck2009fast}. 
We provide detailed expressions of $\bg_t$ for several important algorithms in \Cref{sec:examples}.

For algorithms with iteration function $\bg_t$ in \eqref{eq:general-form}
depending on the previous two iterates,
the Jacobian of $\bg_t$ at a point where
$\bg_t$ is differentiable is a matrix of size $\R^{p \times 4p}$.
This Jacobian matrix in $\R^{p \times 4p}$ is naturally partitioned
into the four $p\times p$ blocks:
\begin{equation}\label{eq:D-J} 
\begin{aligned}
    \bJ_{t,t-1}&= \pdv{\bg_{t}}{\bb^{t-1}}(\hbb^{t-1}, \hbb^{t-2}, \bv^{t-1}, \bv^{t-2}), \quad
    \bD_{t,t-1} = \pdv{\bg_{t}}{\bv^{t-1}}(\hbb^{t-1}, \hbb^{t-2}, \bv^{t-1}, \bv^{t-2}),\\
    \bJ_{t,t-2} &= \pdv{\bg_{t}}{\bb^{t-2}}(\hbb^{t-1}, \hbb^{t-2}, \bv^{t-1}, \bv^{t-2}), \quad 
    \bD_{t,t-2} = \pdv{\bg_{t}}{\bv^{t-2}}(\hbb^{t-1}, \hbb^{t-2}, \bv^{t-1}, \bv^{t-2}).
\end{aligned}
\end{equation}
We further define larger matrices $\calJ, \calD \in \R^{pT\times pT}$ by setting the $(t,t')$ block of $\calJ$ as $\bJ_{t,t'}$ and that of $\calD$ as $\bD_{t,t'}$ for all $t,t'\in [T]$ for all $t,t'$ corresponding to matrices
in \eqref{eq:D-J},
with zeros everywhere else:
\begin{equation} \label{eq:calD-J}
\calJ = 
\begin{bmatrix}
    \b0 & \b0 & \cdots & \cdots & \cdots& \b0\\
    \bJ_{2,1} & \b0 & \ddots & \cdots & \cdots& \vdots\\
    \bJ_{3,1} & \bJ_{3,2} & \b0 & \ddots & \cdots& \vdots\\
    \b0 & \bJ_{4,2} & \bJ_{4,3} & \b0 & \ddots& \vdots\\
    \vdots & \ddots & \ddots & \ddots & \ddots& \b0\\
    \b0 & \cdots &\b0 & \bJ_{T, T-2} & \bJ_{T, T-1} &\b0\\
\end{bmatrix}
,~
\calD = 
\begin{bmatrix}
    \b0 & \b0 & \cdots & \cdots & \cdots& \b0\\
    \bD_{2,1} & \b0 & \ddots & \cdots & \cdots& \vdots\\
    \bD_{3,1} & \bD_{3,2} & \b0 & \ddots & \cdots& \vdots\\
    \b0 & \bD_{4,2} & \bD_{4,3} & \b0 & \ddots& \vdots\\
    \vdots & \ddots & \ddots & \ddots & \ddots& \b0\\
    \b0 & \cdots &\b0 & \bD_{T, T-2} & \bD_{T, T-1} &\b0
\end{bmatrix}.
\end{equation}
An equivalent definition
of $\calJ,\calD$ using Kronecker products is
\begin{equation*}
\calJ = \sum_{t=2}^T 
\sum_{s=t-2}^{t-1}
(\be_t \be_s^\top) \otimes
\bJ_{t,s}, \qquad
\calD = \sum_{t=2}^T
\sum_{s=t-2}^{t-1}
(\be_t \be_s^\top) \otimes
\bD_{t,s}, 
\end{equation*}
where $\be_t,\be_s$ are the $t$-th
and $s$-th canonical basis vectors
in $\R^T$.
\subsubsection{General form: combining all previous iterates}
\label{subsubsec:general}
More generally, our theory applies to iterations that depend on
all previous iterates, \ie,
iterates of the form
\begin{equation}
    \hbb^{t} = \bg_{t}
    (\hbb^{t-1}, \hbb^{t-2},\dots,\hbb^2,\hbb^1, ~  \bv^{t-1}, ,\dots,\bv^2,\bv^1)
    \label{eq:iterates_all_previous}
\end{equation}
with $\bv^s=\hbb^s = \boldzero_p$ for all $s\le 0$ by convention.
With the more general iterations \eqref{eq:iterates_all_previous}
depending on all previous iterates, the Jacobian of $\bg_t$ in
\eqref{eq:iterates_all_previous} is now a matrix in $\R^{p\times 2p(t-1)}$
partitioned into blocks
\begin{equation}
    \bJ_{t,s}
    = \frac{\partial \bg_t}{\partial \bb^s}\Bigl(
    \hbb^{t-1},...,\hbb^1,
    \bv^{t-1},...,\bv^1
    \Bigr),
    \quad
    \bD_{t,s}
    = \frac{\partial \bg_t}{\partial \bv^s}\Bigl(
    \hbb^{t-1},...,\hbb^1,
    \bv^{t-1},...,\bv^1
    \Bigr) 
    \label{J_ts_D_ts_all_previous}
\end{equation}
for all $s=1,...,t-1$.
The large matrices $\calJ, \calD\in\R^{pT\times pT}$ are then defined as
\begin{equation}
\calJ = \sum_{t=2}^T 
\sum_{s=1}^{t-1}
(\be_t \be_s^\top) \otimes
\bJ_{t,s}, \qquad
\calD = \sum_{t=2}^T 
    \sum_{s=1}^{t-1}
(\be_t \be_s^\top) \otimes
\bD_{t,s}.
\label{eq:all_previous_calD_calJ}
\end{equation}
This generalizes \eqref{eq:calD-J}
when $\hbb^t$ depends on all previous
iterates.
The matrices $\calJ$ and $\calD$
are always lower triangular block matrices, as in \eqref{eq:calD-J}.

\subsection{Lipschitz assumption and the chain rule}
Throughout, we assume the following.
\begin{assumption}\label{assu:Lipschitz}
    The algorithm starts with $\hbb^1 =\b0$. 
    The iteration function $\bg_t$ in \eqref{eq:iterates_all_previous}
    is $\zeta$-Lipschitz and satisfies 
    $\bg_t(\b0) = \b0$.
\end{assumption}

The notation in \eqref{eq:D-J} and \eqref{J_ts_D_ts_all_previous} naturally arises from the application of the chain rule in the recursive computation of derivatives for the iterates
\eqref{eq:general-form}
and
\eqref{eq:iterates_all_previous}.
While assuming that $\bg_t$ is Lipschitz (\Cref{assu:Lipschitz}) implies
that the iterates are differentiable almost everywhere
with respect to $(\bX,\by)$ by Rademacher's theorem
(a locally Lipschitz function of $(\by,\bX)$ is differentiable
almost everywhere), the chain rule may fail for composition
of multivariate Lipschitz functions.
This technical issue is clearly apparent 
with the example 
\begin{equation}
    \label{chainrule_lipschitz_example}
    \phi^*(u,v) = (u,u),
    \quad
    \phi^{**}(u, v) = \max(u, {v}),
    \quad
    \phi^{**}\circ \phi^*(u, v) = u.
\end{equation}
Although $\frac{\partial}{\partial u}\bigl(\phi^{**}\circ \phi^*\bigr)(u, v) = 1$
clearly holds,
the chain rule is undefined because
$\phi^{**}$ is not differentiable
at $\phi^*(u,v)$ for all $(u,v)\in\R^2$. 
A fix to this issue for the composition of locally Lipschitz functions
is given in Corollary 3.2 of
\cite{ambrosio1990general}: the chain rule for
for the composition of two Lipschitz functions $g:\R^k\to\R^q,u:\R^q\to\R^p$
holds almost surely in the modified
form $\nabla(g\circ u)(x)=\nabla(g|_{T_x^u})\nabla u(x)$
where $g|_{T_x^u}$ is the restriction of $g$ to the affine space defined 
in \cite[Corollary 3.2]{ambrosio1990general}.
Thus, in order to leverage the chain rule in our proofs, we will assume that
the matrices in \eqref{eq:D-J} or \eqref{J_ts_D_ts_all_previous} are all
modified, if necessary, in order to grant the chain rule as given in
\cite[Corollary 3.2]{ambrosio1990general}.
In many practical examples,
the Lipschitz (but not-necessarily differentiable everywhere)
nonlinear function used to recursively obtain the iterates 
is separable, \ie, $\bg^{sep}:\R^p\to\R^p$ of the form 
\begin{equation*}
    \bg^{sep}(\bx)_j=g_j(x_j)
    \quad
    \text{ for all }\bx\in\R^p,j\in[p]
    \quad
    \text{for some functions } g_j:\R\to\R.
\end{equation*}
This is the case for the soft-thresholding in \eqref{soft_intro}
as well as its variants, FISTA and LQA, detailed in
\Cref{sec:eg-FISTA,sec:eg-LQA}. In this case
by \cite[Theorem 2.1.11]{ziemer2012weakly}, the chain rule holds
almost everywhere for the composition of Lipschitz functions 
of the form $g_j\circ h$ for $h:\R^k\to\R$ and $g_j:\R\to\R$.
Consequently, the usual (unmodified) chain rule is granted almost everywhere
for compositions of Lipschitz functions
that are either separable or
everywhere differentiable.

\subsection{Memory matrix}

Equipped with the above notation in \eqref{J_ts_D_ts_all_previous} and \eqref{eq:all_previous_calD_calJ}, we now introduce the memory matrix. The memory matrix
is a key ingredient in the construction of our estimator of the prediction
error of the iterate $\hbb^t$ obtained from the iterations
\eqref{eq:general-form} or \eqref{eq:iterates_all_previous}.

\begin{definition}
For iterates $\{\hbb^1, \hbb^2, ..., \hbb^T\}$ obtained from the iteration \eqref{eq:iterates_all_previous},
with $\calD,\calJ$ defined in 
\eqref{eq:all_previous_calD_calJ},
define the memory matrix $\hbA\in \R^{T\times T}$ as 
\begin{equation}
    \label{eq:hat-A}
    \hbA 
    = \frac{1}{n}\sum_{i=1}^n 
    \Bigl(\bI_T \otimes (\be_i^\top \bX)\Bigr)
    \Bigl(\bI_{pT} + \calD (\bI_T \otimes \tfrac{\bX^\top \bX}{n}) - \calJ\Bigr)^{-1} 
    \calD \Bigl(\bI_T \otimes (\bX^\top \be_i)\Bigr)
\end{equation}
where $\be_i\in\R^n$ is the $i$-th
canonical basis vector.
\end{definition}
By definition, the memory matrix $\hbA$ only depends on the data $(\bX, \by)$, the iterates $\hbb^t$, the vectors $\bv^t$ in
\eqref{v_t}, and the derivative matrices in \eqref{J_ts_D_ts_all_previous}
evaluated at $(\hbb^{t-1},...,\hbb^1, \bv^{t-1},...,\bv^1)$.
The specific computation of $\hbA$ thus depends on the iteration function $\bg_t$ through the derivative matrices in \eqref{J_ts_D_ts_all_previous}
that appear as blocks of $\calD$ and $\calJ$.
The matrix $\hbA$ will differ significantly for different choices
of nonlinear functions $\bg_t$.
For the same nonlinear functions $\bg_t$, the matrix $\hbA$
will also be different for different realizations of the data $(\bX,\by)$,
although we observe in practice that the entries of $\hbA$ concentrate
and have small variance with respect to the randomness of $(\bX,\by)$.

In \Cref{sec:examples}, we provide detailed expressions of $\calJ$ and $\calD$ for several important algorithms in the optimization literature. 
 While definition \eqref{eq:hat-A} involves the inverse of a matrix
of size $\R^{pT\times pT}$, due to its specific structure,
computation of $\hbA$ requires
much fewer resources that performing a matrix inversion
or solving a linear system in $\R^{pT\times pT}$.
We provide efficient methods to compute $\hbA$ row by row in \Cref{sec:computation}.

An equivalent definition of the matrix $\hbA$ is to consider the matrix
\begin{equation}
    \frac{1}{n}
    \Bigl(\bI_T \otimes  \bX\Bigr)
    \Bigl(\bI_{pT} + \calD (\bI_T \otimes \tfrac{\bX^\top \bX}{n}) - \calJ\Bigr)^{-1} 
    \calD \Bigl(\bI_T \otimes (\bX^\top)\Bigr)
    \label{large_nT_nT}
\end{equation}
in $\R^{nT\times nT}$ as
a matrix by block, with $T\times T$
blocks of size $n\times n$. The matrix $\hbA\in\R^{T\times T}$ is then
obtained by taking the trace of each $n\times n$ block.

Since $\calD$ and $\calJ$ are lower triangular matrices with diagonal
blocks equal to $\boldzero_{p\times p}$, both
$\bI_{pT} + \calD(\bI_T\otimes \bX^\top\bX/n) - \calJ$
and its inverse are lower triangular with diagonal blocks equal to $\bI_p$.
Consequently, \eqref{large_nT_nT} is block lower triangular with
diagonal blocks equal to $\boldzero_{n\times n}$. By taking the trace
of each $n\times n$ block in \eqref{large_nT_nT}, we see that
$\hbA\in\R^{T\times T}$ is lower triangular with
zero diagonal entries.
It thus always holds that the matrix $\bI_T - \hbA/n$ is a lower triangular with all diagonal entries equal to 1, invertible, with
$(\bI_T - \hbA/n)^{-1}$ also lower triangular with all diagonal entries equal to 1.
The matrix 
\begin{equation*}
    (\bI_T - \hbA/n)^{-1}
\end{equation*}
plays a critical role in our
estimator of the generalization error of each iteration $\hbb^t$ in
\Cref{thm:rt,thm:generalization-error} below.

\subsection{Probabilistic assumptions and proportional regime}
Our main theorems below hold under the following assumptions.
\begin{assumption}\label{assu:design}
    The design matrix $\bX$ has \iid rows from $\mathsf{N}_p(\bf0, \bSigma)$ for some positive definite matrix $\bSigma$ satisfying 
    $0< \lambda_{\min}(\bSigma)\le 1 \le \lambda_{\max}(\bSigma)$ and 
    $\opnorm{\bSigma} \|\bSigma^{-1}\|_{\rm op} \le \kappa$. 
\end{assumption}

\begin{assumption}\label{assu:noise}
The noise $\bep$ is independent of $\bX$ and 
has \iid entries from $\mathsf{N}(0, \sigma^2)$. 
\end{assumption}

\begin{assumption}\label{assu:regime}
The sample size $n$ and predictor dimension $p$ satisfy $p/n\le \gamma$ for a constant $\gamma \in (0, \infty)$.
\end{assumption}

Recall that $\zeta$ is the Lipschitz constant in \Cref{assu:Lipschitz}.
In the results below, $(\zeta,T,\gamma,\kappa)$ can be thought of as 
constant problem parameters as $n,p\to+\infty$; our upper bounds involve constants $C(\zeta,T,\gamma,\kappa)$
that only depend on these four quantities.

\subsection{Estimation of the generalization error of iterates}
\label{subsec:r_t}
We are interested in quantifying the performance of each iterate $\hbb^t$. 
We assess the performance of the iterate $\hbb^t$ using the out-of-sample prediction error (also referred to as prediction risk or generalization error)
\begin{equation}
    \label{eq:rt}
    r_t
    \defas 
    \E \Bigl[\bigl(y_{new} - \bx_{new}^\top \hbb^t\bigr)^2 \mid (\bX, \by)\Bigr]
    = \|\bSigma^{1/2} (\hbb^t - \bb^*)\|^2 + \sigma^2,
\end{equation}
where $(\bx_{new}, y_{new})$ is a new test sample that has the same distribution as $(\bx_1, y_1)$, the first row of $(\bX,\by)$.
Since $\hbb^t$ is measurable with respect to $(\bX,\by)$, the last equality
above follows from \Cref{assu:design,assu:noise}.
Note that $(\bx_{new}, y_{new})$ is only used as a mathematical
device to define $r_t$ in \eqref{eq:rt}; we do not assume that independent
test samples are available.
Our first result concerns the estimation of prediction risk $r_t$ for each iterate $\hbb^t$. 
\begin{theorem}[Estimation of prediction risk]\label{thm:rt}
    Let \Cref{assu:design,assu:noise,assu:regime,assu:Lipschitz} be fulfilled. 
    For each $t\in [T]$, 
    define the estimate $\hat r_t$ of $r_t$ by
    \begin{equation}
        \label{eq:hat-rt}
        \hat r_t
        = \frac1n\Big\|\sum_{s=1}^t \hat w_{t,s} \bigl(\by - \bX \hbb^s\bigr) \Big\|^2,
    \end{equation}
    where $\hat w_{t,s} = \be_t^\top (\bI_T - \hbA/n)^{-1}\be_s$.
    We have for any $t\in [T]$,
    \begin{equation}
        \E\big|\hat r_t - r_t\big| 
        \le n^{-\frac 12} 
        \var(y_1)
        C(\zeta, T, \gamma, \kappa) 
        .
    \end{equation}
    Here 
    $\var(y_1) = \norm{\bSigma^{1/2}\bb^*}^2 + \sigma^2$, and 
    $C(\zeta, T, \gamma, \kappa)$ is a constant depending only on 
    $\zeta, T, \gamma, \kappa$.
\end{theorem}
\Cref{thm:rt} establishes that $\hat r_t$ is $\sqrt n$-consistent when
$\var(y_1),\zeta,T,\gamma,\kappa$ are constant as $n,p\to+\infty$.
The risk estimate $\hat r_t$ for $\hbb^t$ is determined by both the current residual vector $\by - \bX \hbb^t$ and all preceding residual vectors $\{\by - \bX \hbb^s\}_{s=1, ..., t-1}$ through the weighted sum
inside the mean squared norm in \eqref{eq:hat-rt}.
These weights
$(\hat w_{t,1}, \ldots, \hat w_{t,t})$ are the first $t$ entries of the $t$-th row of the matrix $(\bI_T - \hbA/n)^{-1}$.
The proof of \Cref{thm:rt} is given in \Cref{sec:proof-thm-rt}.

By construction of the matrices $\hbA$
and $(\bI_T - \hbA/n)^{-1}$,
the definition of these weights are not influenced by the total number of iterations $T$ (\eg, the weights $(\hat w_{s,t})_{s\le t}$ for any $t\le T$
are the same if the total number of iterations is $T$ or any $T'>T$).
In other words, opting for a larger $T$ simply enlarges the dimensions of the matrix $\hbA$ and the inverse matrix $(\bI_T - \hbA/n)^{-1}$, yet the first $t$ entries in the $t$-th row of $(\bI_T - \hbA/n)^{-1}$ remain unchanged.
The necessity of a total number of iteration, $T$, is arguably
artificial; an equivalent presentation would increase the size
of $\hbA$ and $(\bI_T - \hbA/n)^{-1}$ by one row and one column
at every new iteration.
We opted for a fixed total number of iterations $T$ with
matrices $\hbA$ and $(\bI_T - \hbA/n)^{-1}$ having fixed size $T\times T$
for clarity,
to avoid dealing with matrices changing in size.

\begin{remark}[Risk of initialization]
    For $t=1$, $\hat w_{t,s} = \I(s=1)$ since $(\bI_T - \hbA/n)^{-1}$ is lower triangular
    with diagonal entries all equal to 1.
    As a sanity check,
    we obtain $\hat r_1 = n^{-1} \|\by - \bX\hbb^1\|^2$
    from \eqref{eq:hat-rt}, which is an unbiased estimate
    of the prediction error of the initialization $\hbb^1$
    since $\hbb^1$ is independent of $(\bX,\by)$.
\end{remark}

Besides the individual iterate $\hbb^t$, one might also consider the estimate $\bar\bb$ defined as the average of $m$ consecutive iterates starting from $\hbb^{t_0}$, \ie, $\bar\bb = \frac{1}{m}\sum_{t=t_0}^{t_0+m-1} \hbb^t$. 
The prediction risk of
the estimate 
$\bar\bb$ is then 
\begin{equation*}
    \sigma^2 + \|\bSigma^{1/2}(\bar\bb-\bb^*)\|^2
    =
    \frac{1}{m^2}
    \sum_{t=t_0}^{t_0+m-1} 
    \sum_{t'=t_0}^{t_0+m-1} 
    \Bigl[\sigma^2 + (\hbb^t-\bb^*)^\top\bSigma(\hbb^{t'}-\bb^*)\Bigr],
\end{equation*}
simply by expanding the square.
Estimation of the generalization error of $\bar\bb$ is thus possible provided
that we can estimate the terms
\begin{equation}
    \label{eq:rtt}
    r_{tt'} \defas \sigma^2 + (\hbb^t-\bb^*)^\top\bSigma(\hbb^{t'}-\bb^*)
\end{equation}
for each $t,t'$ inside the double sum.
The following result derives a $\sqrt n$-consistent estimate
of $r_{tt'}$, using the same weights as in \Cref{thm:rt}.

\begin{theorem}[Proof is given in \Cref{sec:proof-thm-generalization-error}]
    \label{thm:generalization-error}
    Let \Cref{assu:design,assu:noise,assu:regime,assu:Lipschitz} be fulfilled. 
    For two integers \( t,t' \le T \) define the estimate of $r_{tt'}$ as 
    \begin{equation}
        \hat r_{tt'} = 
        \frac1n
        \Bigl(
        \sum_{s=1}^t \hat w_{t,s} \bigl(\by - \bX \hbb^s\bigr)
        \Bigr)^\top
        \Bigl(
            \sum_{s'=1}^{t'} \hat w_{t',s'} \bigl(\by - \bX \hbb^{s'}\bigr)
        \Bigr),
        \label{eq:hat_r_tt'}
    \end{equation}
    where $\hat w_{a,b} = \be_a^\top (\bI_T - \hbA/n)^{-1}\be_b$ for all \( a,b\in [T] \) as in \Cref{thm:rt}.
    Then
    \begin{align*}
        \E |\hat r_{tt'} - r_{tt'}|
        + \var(r_{tt'})^{1/2}
        \le n^{-\frac 12} C(\zeta, T, \gamma, \kappa) \var(y_1).
    \end{align*}
    \end{theorem}
    
\subsection{An oracle inequality for early stopping}
Since \Cref{thm:rt} provides upper bounds of the form
$\E |\hat r_t - r_t| \le n^{-1/2} C(\zeta,T,\gamma,\kappa) \var(y_1)$.
A straightforward application of Markov's inequality
yields the following.

\begin{corollary}[Proof is given in \Cref{proof:cor:early}]
    \label{cor:early}
    Let \Cref{assu:design,assu:noise,assu:regime,assu:Lipschitz} be fulfilled.
    Select an iteration $\hat t\in[T]$ by
    minimizing \eqref{eq:hat-rt}, that is,
    $\hat t = \argmin_{t\in[T]}\|\sum_{s\le t}\hat w_{t,s}(\by-\bX\hbb^s)\|^2/n$.
    Then for any constant $c\in(0,1/2)$,
    \begin{equation*}
        \P\Bigl(
        \|\bSigma^{1/2}(\hbb^{{\hat t}}-\bb^*)\|^2
        \le
        \min_{s\in[T]}
        \|\bSigma^{1/2}(\hbb^{s}-\bb^*)\|^2
        + \frac{\var(y_1)}{n^{1/2-c}}
        \Bigr) \ge 1 - 
        \frac{C(\zeta,\gamma,T,\kappa)}{n^c}
        \to 1.
    \end{equation*}
\end{corollary}

This means that after running $T$ iterations, we can pick the iteration
$\hat t$ that minimizes the criterion \eqref{eq:hat-rt}, and this choice
achieves the smallest prediction error among the first $T$ iterates
up to a negligible error term.
This is appealing for the settings illustrated in \Cref{fig:LSE-GD,fig:LSE-AGD-1500,fig:Lasso-ISTA-1500,fig:Lasso-FISTA-1500,fig:MCP-LQA-1500}
where the generalization error of the iterates $\hbb^t$ is first decreasing
in $t$ up to some $t_*$, before increasing for $t\ge t_*$.
The above selection of $\hat t$ guarantees that the risk of $\hbb^{\hat t}$ is close to that of $\hbb^{t_*}$.

Given $K_n$ candidate sequences of nonlinear functions, say
$(\bg_{t}^{k})_{t\in [T]}$ for each $k\in[K_n]$, a similar argument
using Markov's inequality yields
$$
\|\bSigma^{1/2}(\hbb^{{\hat t}}-\bb^*)\|^2
\le
\min_{k\in [K_n]}
\min_{s\in[T]}
\|\bSigma^{1/2}(\hbb^{s}-\bb^*)\|^2
+ \var(y_1) o_P(1)
$$
if $(\zeta,\gamma,T,\kappa)$ are constants as $n,p\to+\infty$
and $K_n=o(\sqrt n)$.

\subsection{Alternative weights and covariance-adaptive weights}
The proofs of \Cref{thm:rt,thm:generalization-error} reveal
that if $\bSigma=\bI_p$, the weights $\hat w_{t,s}$
in \eqref{eq:hat-rt} and \eqref{eq:hat_r_tt'} can be replaced
with the alternative weights
\begin{equation*}
\begin{aligned}
    \check w_{t,s} &= 
    \mathbb I\{t=s\}
    + {\frac1n} \trace\Bigl[(\be_t^\top\otimes \bI_p)\bigl(\bI_{pT} + \calD (\bI_T \otimes \tfrac{\bX^\top \bX}{n}) - \calJ\bigr)^{-1} \calD (\be_s\otimes \bI_p)\Bigr]
                 \\&=
    \mathbb I\{t=s\}
    + {\frac1n}
    \sum_{j=1}^p
    (\be_t^\top \otimes \be_j^\top)\bigl(\bI_{pT} + \calD (\bI_T \otimes \tfrac{\bX^\top \bX}{n}) - \calJ\bigr)^{-1}\calD (\be_s\otimes \be_j)
\end{aligned}
\text{ (for $\bSigma=\bI_p$).}
\end{equation*}
Indeed, these alternative weights
$\check w_{t,s}$ are the $(t,s)$ entries of the matrix $\bI_T + \hbC/n$ in the proof of \Cref{thm:generalization-error} with $\hbC\in \R^{T\times T}$ defined in
\eqref{hbC}.
These weights satisfy
\begin{equation}
    \E\Bigl[\sum_{s=1}^T 
        \frac 1 n
    \Big\|
    \sum_{s=1}^t (\hat w_{t,s} - \check w_{t,s}) (\by-\bX\hbb^s)
    \Big\|^2
    \Bigr] \le C(\zeta,T,\gamma) n^{-1/2},
    \qquad
    \text{ when }\bSigma= \bI_p,
\end{equation}
by \eqref{eq:bound_alternative_weights} in the appendix. This means that
$\hat w_{t,s}$
and $\check w_{t,s}$ can be used interchangeably
as weights for
$(\by-\bX\hbb^s)_{s\le T}$ in the definition of $\hat r_t$ when $\bSigma = \bI_p$.

This interchangeability is lost as soon as $\bSigma\ne \bI_p$.
In this case, the alternative weights $\check w_{t,s}$ have expression
\begin{equation}
    \check w_{t,s} = 
    \mathbb I\{t=s\}
    + {\tfrac1n}\trace\bigl[(\be_t^\top\otimes \bSigma^{1/2})
    \bigl(\bI_{pT} + \calD (\bI_T \otimes \tfrac{\bX^\top \bX}{n}) - \calJ\bigr)^{-1}
    \calD (\be_s\otimes \bSigma^{1/2})\bigr]
\end{equation}
using the correspondence \eqref{relation_M_M^*_D_D^*}.
The alternative weights $\check w_{t,s}$ thus require the knowledge
of $\bSigma$. This is the main reason we presented the main results
using the weights $\hat w_{t,s}$ and the memory matrix $\hbA$:
The expressions of $\hat w_{t,s}$ and $\hbA$ do not 
require prior knowledge of  $\bSigma$ or estimating $\bSigma$ from the data.
The weights $\hat w_{t,s}$ are thus preferable, and more broadly
applicable than $\check w_{t,s}$.
This duality between two interchangeable scalar adjustments,
one requiring the knowledge of $\bSigma$ and one
not requiring it,
was already observed in regularized M-estimation
\cite[Section 5]{bellec2021derivatives}.

\subsection{Iterative debiased estimation}
This subsection focuses on asymptotic normality and 
statistical inference for the entries of $\bb^*$ using iterate $\hbb^t$.
The next theorem shows that the following debiased estimate,
\begin{equation}
    \label{debiased_j}
    \hbb^{t, \rm{debias}}_j:=
    \underbrace{\hbb^t_j}_{\text{iterate}}
    ~+~
    \underbrace{n^{-1} 
    \sum_{s=1}^t \hat w_{t,s} \bigl(\by - \bX \hbb^s\bigr)^\top
    \bX \bSigma^{-1} \be_j}_{\text{bias correction}}
    ,
\end{equation}
is approximately normally distributed and centered at $\bb_j^*$.
Above, the observable
weights $\hat w_{t,s}$ are the same as in \Cref{thm:rt}.
The approximate normality below is quantified using the 2-Wasserstein distance
between probability distributions in $\R^T$, defined as 
$W_2(\mu, \nu)=\inf_{(\bu,\bw)}\E[\|\bu-\bv\|^2]^{1/2}$
where the infimum is taken over all couplings $(\bu,\bw)$ of the two probability measures $(\mu,\nu)$.
We refer to \cite[Definition 6.8]{villani2009optimal} for several
characterizations of convergence with respect to $W_2$.

The following result is understood asymptotically as $n,p\to+\infty$.
Implicitly, we assume that a sequence of regression problems
\eqref{eq:linear-model} is given, together with nonlinear functions $\bg_t$ in \eqref{eq:iterates_all_previous}. It is implicit that $n$ serves as the index of the sequence,
and $p=p^{(n)}$, $\bb^*{}^{(n)}$, $\bg_t{}^{(n)}$ all implicitly depend on $n$
and may change values as $n$ increases, as long as 
\Cref{assu:design,assu:Lipschitz,assu:noise,assu:regime} are 
fulfilled for every $n$. The constants $(T,\zeta,\gamma,\kappa,\sigma^2)$ 
do not depend on $n$.

\begin{theorem}
    \label{thm:debias}
    Let \Cref{assu:design,assu:Lipschitz,assu:noise,assu:regime}
    be fulfilled and assume that both
    $T$ and $\var(y_1)$ are bounded from above
    by a fixed constant as \( n,p\to+\infty \).
    Then there exists a random vector \( \bzeta_j \)
        in $\R^T$ with
        \begin{equation}
            \label{convergence_Zeta_j_W2_to_0}
            \max_{j\in[p]}
            W_2\Bigl( \mathsf{Law}(\bzeta_j), ~~ \mathsf{N}\Bigl(\boldzero_T, \bigl(\E[r_{tt'}]\bigr)_{(t,t')\in[T]\times [T]}\Bigr) \Bigr)
        \to 0
        \end{equation}
        such that for each $t \in[T]$,
        \begin{equation}
    \sum_{j=1}^p
    \E\Bigl[
    \Bigl(
        \sqrt{\frac{n}{\|\bSigma^{-1/2} \be_j\|^2}}
        \Bigl(\hbb^{t, \rm{debias}}_j - \bb^*_j\Bigr)
        - \bzeta_{jt}
    \Bigr)^2
    \Bigr] \le {C(\zeta, T, \kappa, \gamma) \var(y_1)}.
        \label{eq:debias-normality}
        \end{equation}
\end{theorem}
The proof of \Cref{thm:debias} is given in \Cref{sec:proof-thm-debias}.
If $\var(y_1)$ stays bounded as $n,p\to+\infty$, the sum over $p$ entries
in the left-hand side of \eqref{eq:debias-normality} stays bounded, and at most a constant number
of entries may not converge to 0. This is made precise in the following corollary.

\begin{corollary} \label{cor:debias}
    Let \Cref{assu:design,assu:Lipschitz,assu:noise,assu:regime}
    be fulfilled and assume that 
    $T$, $\var(y_1)$ are bounded from above
    by a fixed constant as \( n,p\to+\infty \).
    Let $a_p\to+\infty$ be a slowly increasing sequence (\eg, $a_p=\log p$).
    There exists a subset $J_{n,p}$ of size at least $p- a_p$ 
    such that
    \begin{equation}
        \max_{j\in J_{n,p}}
        W_2\Bigl(
        \sqrt{\frac{n}{\|\bSigma^{-1/2} \be_j\|^2}}
        \bigl(\hbb^{t, \rm{debias}}_j - \bb^*_j\bigr)
        ,
        ~~
        \mathsf{N}\bigl(0,\E[r_t]\bigr)
        \Bigr)
    \to 0
    \label{eq:Zj}.
        \end{equation}
\end{corollary}
The proof of \Cref{cor:debias} is given in \Cref{sec:proof_cor_debias}. 
Convergence in 2-Wasserstein distance implies convergence in distributions 
\citep[Def. 6.8 and Theorem 6.9]{villani2009optimal}. If $\var(y_1)/\sigma^2$ is bounded as $n,p\to+\infty$,
\Cref{thm:generalization-error} further shows that
$\hat r_t$ consistently estimates $\E[r_t]$, so that
by Slutsky's theorem and \eqref{eq:Zj},
the z-score 
\begin{equation}
    \label{eq:z-score}
    \frac{\sqrt{n}(\hbb^{t, \rm{debias}}_j - \bb^*_j)}
    {
        \|\bSigma^{-1/2} \be_j\|
    \sqrt{\hat r_t}}
\end{equation}
converge to $\sf N(0,1)$ in distributions uniformly
over $j\in J_{n,p}$.
For those overwhelming majority
of components $j\in J_{n,p}$, \Cref{thm:generalization-error} and 
\Cref{cor:debias} thus provide the $1- \alpha$ confidence interval for $\bb_j^*:$
$$
\Big[\hbb_j^{t, \rm{debias}} 
- z_{\alpha/2}\sqrt{\frac{\hat r_t}{n}(\bSigma^{-1})_{jj}}, 
\quad 
\hbb^{t, \rm{debias}}
+ z_{\alpha/2}\sqrt{\frac{\hat r_t}{n}(\bSigma^{-1})_{jj}}
\Big],
$$
where $z_{\alpha/2}$ is the standard normal quantile defined by $\P(|\mathsf{N}(0,1)|\ge z_{\alpha/2})=\alpha$.

Asymptotic normality ``on average'' over coordinates, 
or only over a large subsets of coordinates, is typical in asymptotic
normality results in the proportional regime, see for instance
\cite{bayati2012lasso,sur2018modern,lei2018asymptotics,berthier2020state,celentano2020lasso,bellec2023debiasing}. 
Studying the remaining coordinates ($j\notin J_{n,p}$ above)
remains a challenge, and in some situations the remaining coordinates
exhibit a ``variance spike'' where the standard deviation
estimate in \eqref{eq:Zj} is incorrect and
the asymptotic variance of \eqref{eq:z-score} is bounded away from 1
\cite[Section 3.7]{bellec2023debiasing}.

One application of \Cref{cor:debias} is early rejection in hypothesis
testing. In order to perform a hypothesis test of the form
\begin{equation}
    H_{0j}:
    \bb_j^* = 0
    \qquad
    \text{ against }
    \qquad
    H_{1j}:
    \bb_j^* \ne 0,
\end{equation}
for the $j$-th component of the true regression vector $\bb^*$ in \eqref{eq:linear-model}.
Since the asymptotic normality of the z-score \eqref{eq:z-score}
is maintained throughout the iterations, one may perform
early tests, without waiting for the convergence of $\hbb^t$: 
reject $H_{0j}$ at the $t$-th iteration if the z-score under $H_{0j}$
is larger than the two-sided quantile of the normal distribution.
We may for instance perform this test at $t=10,10^2, 10^3$
with an appropriate Bonferroni multiple testing correction.
While the estimates $\hat r_{t}$ in \eqref{eq:hat-rt} and $\hat r_{tt'}$ in \eqref{eq:hat_r_tt'} do not require the knowledge of the covariance
matrix $\bSigma$, the construction of the debiased estimate
\eqref{debiased_j} requires
the knowledge of $\bSigma^{-1}\be_j$ or an estimate of it.
If unlabeled samples $(\bx_i)_{i=n+1,...,N}$ are available for instance,
$\bSigma^{-1}\be_j$ can be estimated by regressing the $j$-th column
onto the others.
We emphasize that the knowledge of $\bSigma^{-1}\be_j$ is only used
in the construction of debiased estimate \eqref{debiased_j},
and that the estimate $\hat r_t$ of the generalization error
is readily usable without any knowledge of $\bSigma$.

\section{Concrete examples}\label{sec:examples}
In this section, we present the analysis of a few popular algorithms aimed
at solving minimization problems of the form \eqref{eq:b-hat},
including the least-squares problem, Lasso and MCP. 
For the five algorithms present in the five subsections below,
we present explicit expressions for iteration functions $\bg_t$
as well as for the corresponding derivative matrices $\calD, \calJ$ and the memory matrix $\hbA$. For each algorithm, these expressions are then used to compute
 the risk estimate $\hat r_t$ in \eqref{eq:hat-rt} and the z-score in \eqref{eq:z-score}. 
To illustrate our theoretical results, we conduct extensive simulation studies comparing the proposed risk estimator $\hat r_t$ with the actual risk $r_t$ at each iteration $t$ and constructing normal quantile-quantile plots for  z-score defined in \eqref{eq:z-score} compared to the standard normal distribution.
We start this section with the simulation settings. 

\textbf{Simulation setup.}
We consider three distinct scenarios based on the relationship between the sample size and feature dimension, denoted as $(n, p)$, to generate
datasets from the linear model \eqref{eq:linear-model}. These scenarios are:
\begin{itemize}
    \item Over-parametrized regime: $(n, p) = (1200, 1500)$, in which the number of features surpasses the number of samples.
    \item Equal-parametrized regime: $(n, p) = (1200, 1200)$, characterized by an equal number of samples and features.
    \item Under-parametrized regime: $(n, p) = (1200, 500)$, where the number of samples exceeds the number of features.
\end{itemize}

For each scenario, the design matrix $\bX$ is generated from a multivariate normal distribution with zero mean and a covariance matrix $\bSigma = (\bSigma_{jk})_{j,k\in[p]}$, where $\bSigma_{jk} = 0.5^{|j-k|}$. The noise vector $\bep$ follows the standard normal distribution, namely, $\sigma^2 = 1$. 
The coefficient vector $\bb^*$ is chosen with $p/20$ nonzero entries, set to a constant value such that the signal-to-noise ratio $\|\bSigma^{1/2} \bb^*\|^2/\sigma^2$ equals 5.
For each algorithm in the following subsections, the initial vector is set to $\hbb^1 = \b0$, and the algorithm runs for $T = 500$ iterations,
except for GD, where $T = 3000$ is necessary to achieve convergence.
For each iterate $(\hbb^t)_{t\in[T]}$, we compute both the actual risk $r_t$ following \eqref{eq:rt} and our proposed risk estimator $\hat r_t$ as per \eqref{eq:hat-rt}, along with the proposed z-score defined in \eqref{eq:z-score} with $j=1$ (\ie, the first coordinate). 
Each simulation is replicated 100 times. 
We present the simulation results for $(n,p) = (1200, 1500)$ for each algorithm in the subsequent subsections, and leave the other $(n,p)$ pair scenarios to \Cref{sec:app-simulation} since their results are very similar to $(n,p) = (1200, 1500)$. Each simulation we tried confirmed the accuracy of 
the estimated risk $\hat r_t$ for estimating $r_t$ as well as the 
closeness of the z-scores \eqref{eq:z-score} to the standard normal distribution.

\subsection{Gradient descent}\label{sec:eg-GD}
Consider estimating $\bb^*$ by minimizing the squared loss function
$
\bff(\bb) = \frac{1}{2n} \|\by - \bX \bb\|^2. 
$
The gradient descent (GD) method finds the minimizer of $\bff(\bb)$ by iterations 
$$
\hbb^{t} = \hbb^{t-1} - \eta \nabla \bff(\hbb^{t-1}) = \hbb^{t-1} + \frac{\eta}{n} \bX^\top (\by - \bX \hbb^{t-1}), 
$$
with an initialization $\hbb^1\in \R^p$. 
Since the function 
$\bff$ is an $L$-smooth with $L= \opnorm{\bX}^2/n$, one can take fixed step size $\eta = \frac{1}{L}$. 
Using the definition $\bv^t = \frac{1}{n} \bX^\top (\by - \bX \hbb^{t})$, the iteration of GD can be written as 
\begin{equation}\label{eq:iteration-GD}
    \hbb^t = \bg_{t}(\hbb^{t-1}, \bv^{t-1})= \hbb^{t-1} + \eta \bv^{t-1} \text{ for } t\ge 2.
\end{equation}
Therefore, the function $\bg_t$ is Lipschitz continuous, and the matrices in \eqref{eq:pgd-J-D} are given by
\begin{equation}
    \bJ_{t,t-1} = \bI_p, 
    \quad 
    \bD_{t,t-1} = \eta \bI_p.
\end{equation}
Hence, for GD with iteration \eqref{eq:iteration-GD}, the expressions of $\calJ$ and $\calD$ in \eqref{eq:calD-J} become 
\begin{equation*}
    \calJ = 
    \begin{bmatrix}
        \b0 & \b0 & \cdots & \cdots & \b0 \\
        \bI_p & \b0 & \ddots & & \vdots \\
        \b0 & \bI_p & \b0 & \ddots & \vdots \\
        \vdots & \ddots & \ddots & \ddots & \b0 \\
        \b0 & \cdots & \b0 & \bI_p & \b0 \\
    \end{bmatrix}     
    = \bL \otimes \bI_p
, \qquad
\calD = 
\eta
\calJ =
\eta 
    (\bL \otimes \bI_p),
\end{equation*}
where $\bL\in\R^{T\times T}$ is the strictly lower triangular matrix
$\bL = \sum_{t=2}^{T} \be_t \be_{t-1}^\top$.
Now we proceed to derive the expression for each entry of $\hbA$ in \eqref{eq:hat-A} for GD. 
First, note that $\calD = \eta\calJ$, so  that
\begin{align*}
    \calJ - \calD (\bI_T \otimes \tfrac{\bX^\top \bX}{n}) 
    =
    \bL \otimes (\bI_p - \eta \bX^\top\bX /n).
\end{align*}
Since $\bL$ is strictly lower triangular,
the above display is also lower triangular. 
With $\bGamma = \bI_p - \tfrac{\eta \bX^\top \bX}{n}$,
we have 
\begin{align*}
    \Bigl(\bI_{pT} + \calD (\bI_T \otimes \tfrac{\bX^\top \bX}{n}) - \calJ\Bigr)^{-1}
    = 
    \Bigl(\bI_{pT} - (\bL \otimes \bGamma)\Bigr)^{-1}
    = \sum_{k=0}^{\infty} \bigl(\bL \otimes \bGamma\bigr)^k
    = \sum_{k=0}^{T-1} \bL^k \otimes \bGamma^k
\end{align*}
thanks to $(\bL \otimes \bGamma)^k = \bL^k \otimes \bGamma^k$ by the mixed-product
property and since $\bL^k = \b0$ for $k\ge T$. 
In this case, by the definition \eqref{eq:hat-A} of $\hbA$, we get
\begin{align*}
    \hbA 
    & = 
    \sum_{k=0}^{T-1} \bL^{k+1} \trace\Bigl[\frac{\eta}{n}\bX\bGamma^k \bX^\top\Bigr]
    =
    \sum_{k=0}^{T-1} \bL^{k+1} \trace\Bigl[(\bI_p-\bGamma)\bGamma^k\Bigr]
    =
    \sum_{k=0}^{T-2} \bL^{k+1} \trace\Bigl[(\bI_p-\bGamma)\bGamma^k\Bigr].
\end{align*}
More visually, since each incremental power of $\bL$ moves 
the diagonal one step towards the bottom left, we have
\begin{align*}
    \hbA & = 
    \begin{bmatrix}
        0 & 0 & 0 & 0 & 0 & \\
        \trace(\bI_p-\bGamma) & 0 & 0 & 0 & 0 & \\
        \trace((\bI_p-\bGamma) \bGamma) & \trace(\bI_p-\bGamma) & 0 & 0 & 0 & \\
        \vdots & \ddots & \ddots & \ddots & 0 & \\
        \trace((\bI_p-\bGamma) \bGamma^{T-2}) & \ldots & \trace((\bI_p-\bGamma)\bGamma ) & \trace(\bI_p-\bGamma) & 0 & 
    \end{bmatrix}.
\end{align*}
The above simplifications are specific to GD and will typically not be
possible for examples involving nonlinear transformations such as
soft-thresholding. 
Computation of $\hbA$ is straightforward once the eigenvalues of $\bGamma$ is computed. 

\begin{remark}
    A closely related estimator to the GD iterates \eqref{eq:iteration-GD} is the minimum $\ell_2$ norm least-squares estimator, is defined as
\begin{align}\label{eq:tbb}
    \tbb = \argmin \big\{\|\bb\|: \bb \text{ minimizes } \|\by - \bX \bb\|^2\big\}.
\end{align}
For $n> p$, $\tbb = (\bX^\top  \bX)^{-1} \bX^\top \by$, and for $n\le p$, $\tbb = \bX^\top (\bX \bX^\top)^{-1} \by$.
It is equivalent to write
$\tbb = (\bX^\top  \bX)^{\dagger} \bX^\top \by $ 
where $\dagger$ denotes the Moore-Penrose inverse. 
The GD iterates $\hbb^t$ in \eqref{eq:iteration-GD}
converge to the min-norm least-squares solution $\tbb$  as $t\to\infty$ \cite[Proposition 1]{Hastie2022surprises}. We observe this phenomenon in
the simulations below where the risk of $\tbb$ is represented by
an horizontal green curve.
\end{remark}

\textbf{Simulation results.}
Define $r_{\infty} := \norm{\bSigma^{1/2} (\hbb^{\infty} - \bb^*)} + \sigma^2$, where $\hbb^{\infty}:= \lim_{t\to\infty} \hbb^t$ is the min-norm least-squares estimator \eqref{eq:tbb}, and has closed-form expression $\hbb^{\infty} = \tbb = (\bX^\top \bX)^{\dagger} \bX^\top  \by$.
At each iteration $t\in[T]$, we calculate and present the average estimated risk $\hat r_t$, the actual risk $r_t$, and the limiting risk $r_{\infty}$ over 100 repetitions, including 2-standard error bars, in \Cref{fig:LSE-GD-risk-1500}. 
We also provide the quantile-quantile (Q-Q) plot of the z-score \eqref{eq:z-score} at different iterations in \Cref{fig:LSE-GD-zscore-1500}.

\begin{figure}[H]
    \centering
    \begin{subfigure}{0.5\linewidth}
      \centering
      \includegraphics[width=\linewidth]{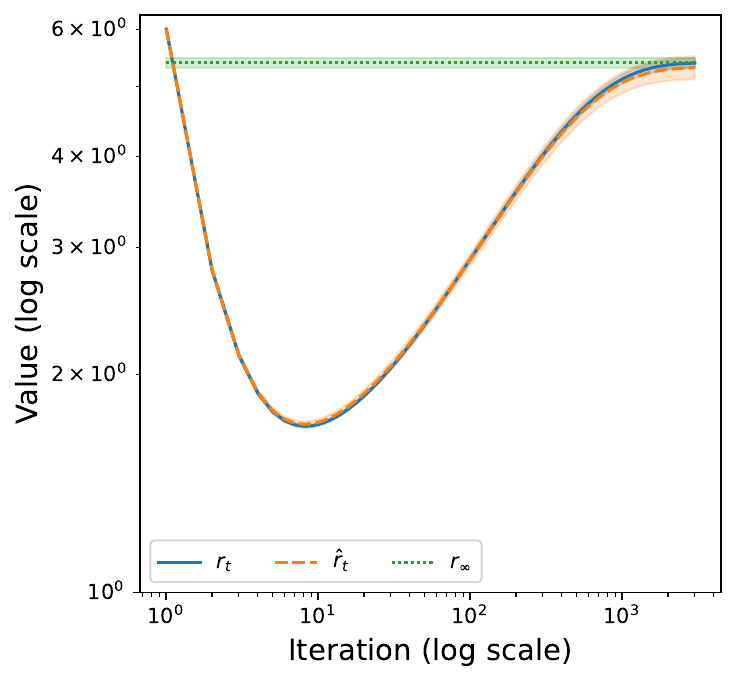}
      \caption{Risk curves versus iteration number.}
      \label{fig:LSE-GD-risk-1500}
    \end{subfigure}
    \hfill
    \begin{subfigure}{0.49\linewidth}
      \centering
      \includegraphics[width=\linewidth]{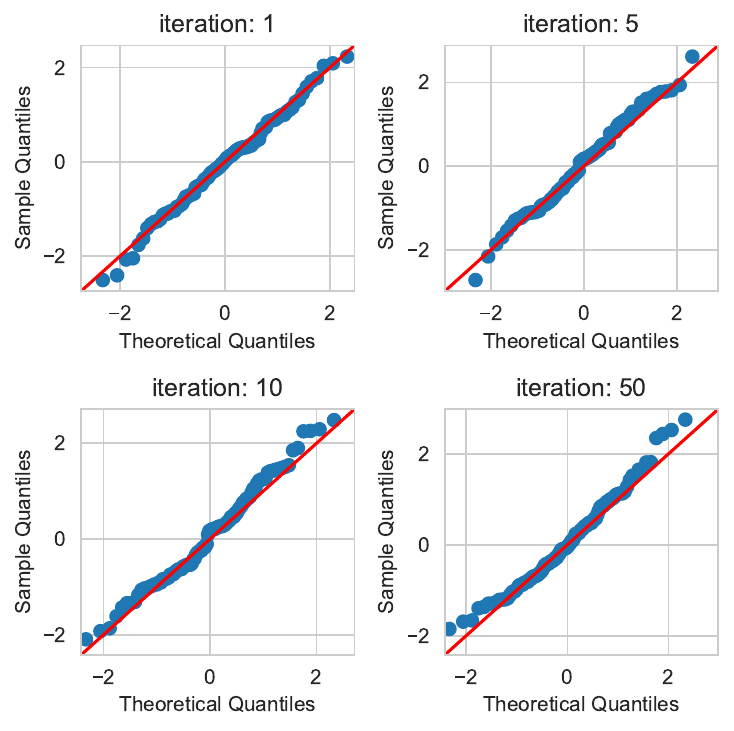}
      \caption{Q-Q plots of z-score at different iterations.}
      \label{fig:LSE-GD-zscore-1500}
    \end{subfigure}
    \caption{Risk curves and qq-plots of z-score of GD for $(n,p)=(1200, 1500)$.}
    \label{fig:LSE-GD}
  \end{figure}

  \Cref{fig:LSE-GD-risk-1500} shows a strong alignment between our risk estimate $\hat r_t$ and the actual risk $r_t$. 
  Notably, both the risk and the estimated risk reach their minimum at iteration 8, suggesting that $\hbb^{8}$ yields the lowest out-of-sample prediction risk. 
  This observation indicates that it is beneficial to terminate the algorithm as early as iteration 8, rather than continuing with additional iterations, because the risk will blow up quickly after iteration 8.
  As $t$ increases, we observe from \Cref{fig:LSE-GD-risk-1500} that $\hat r_t$ converges to $r_{\infty}$, indicating the effectiveness of our risk estimator for large $t$.
  Furthermore, \Cref{fig:LSE-GD-zscore-1500} reveals that the quantiles closely match the 45-degree line (shown in red). This alignment supports the conclusion that the z-score \eqref{eq:z-score} closely approximates a standard normal distribution.

\subsection{Nesterov's accelerated gradient descent (AGD)} \label{sec:eg-AGD}
The Nesterov's accelerated gradient method \citep{nesterov1983method} is a remarkable extension of the gradient descent by utilizing the momentum from previous iterates. 
It is well known that AGD enjoys quadratic convergence rates, which is faster than the linear convergence rate of the gradient descent. 
To describe the iteration of AGD, define the sequence of scalars
\begin{equation}\label{eq:a}
    a_0 = 0, \quad
a_t = \tfrac{1 + \sqrt{1 + 4 a_{t-1}^2}}{2}, \quad
w_t = \tfrac{1 - a_t}{a_{t+1}}\quad \text{ for all } t\ge 1.
\end{equation}
From some initialization $\hbb^1 \in \R^p$,
AGD iterates are defined as the weighted sum
\begin{align}\label{eq:iteration-AGD}
    \hbb^{t} 
    = (1 - w_{t-1}) (\hbb^{t-1} + \eta \bv^{t-1}) + w_{t-1} (\hbb^{t-2} + \eta \bv^{t-2}),
    \qquad\text{ for all } t\ge 2. 
\end{align}
Hence, the Jacobian matrices defined in \eqref{eq:D-J} are given by
\begin{equation*}
    \begin{aligned}
        \bJ_{t,t-1} &= (1 - w_{t-1}) \bI_p, \quad & \bD_{t,t-1} &= \eta (1 - w_{t-1}) \bI_p, \\
        \bJ_{t,t-2} &= w_{t-1} \bI_p, \quad & \bD_{t,t-2} &= \eta w_{t-1} \bI_p.
    \end{aligned}
\end{equation*}
It follows that the expressions of $\calJ$ and $\calD$ in \eqref{eq:calD-J} become
\begin{equation*}
    \calJ = 
\begin{bmatrix}
    \b0 & \b0 & \cdots & \cdots & \cdots& \b0\\
    (1-w_1)\bI_p & \b0 & \ddots & \cdots & \cdots& \vdots\\
    w_2\bI_p & (1-w_2)\bI_p & \b0 & \ddots & \cdots& \vdots\\
    \b0 & w_3\bI_p & (1-w_3)\bI_p & \b0 & \ddots& \vdots\\
    \vdots & \ddots & \ddots & \ddots & \ddots& \b0\\
    \b0 & \cdots &\b0 & w_{T-1}\bI_p & (1-w_{T-1})\bI_p &\b0
\end{bmatrix} 
,\quad
\calD = 
\eta \calJ.
\end{equation*}

\begin{remark}
    By Proposition 1 in \cite{Hastie2022surprises}, the
AGD iterate $\hbb^t$ in \eqref{eq:iteration-AGD} also converges to the min-norm least-squares estimator $\tbb$ as $t\to\infty$.
\end{remark}

\textbf{Simulation results.} Similar to the plots for gradient descent, we provide results for AGD applied to the least-squares problem. The risk curves are shown in \Cref{fig:LSE-AGD-risk-1500} and Q-Q plots of the z-score \eqref{eq:z-score} in \Cref{fig:LSE-AGD-zscore-1500}.
  The simulation results for AGD are similar to those for GD.
  \Cref{fig:LSE-AGD-risk-1500} clearly shows that the risk estimate $\hat r_t$ closely matches the actual risk $r_t$. 
  In addition, the risk curve suggests stopping the algorithm early at iteration 8 as the estimated risk increases quickly after iteration 8. Both $r_t$ and $\hat r_t$ converge to $r_{\infty}$ as $t$ increases. 
\Cref{fig:LSE-AGD-zscore-1500} again confirms that the z-scores \eqref{eq:z-score} are close to a standard normal distribution.
\begin{figure}[H]
    \centering
    \begin{subfigure}{0.50\linewidth}
      \centering
      \includegraphics[width=\linewidth]{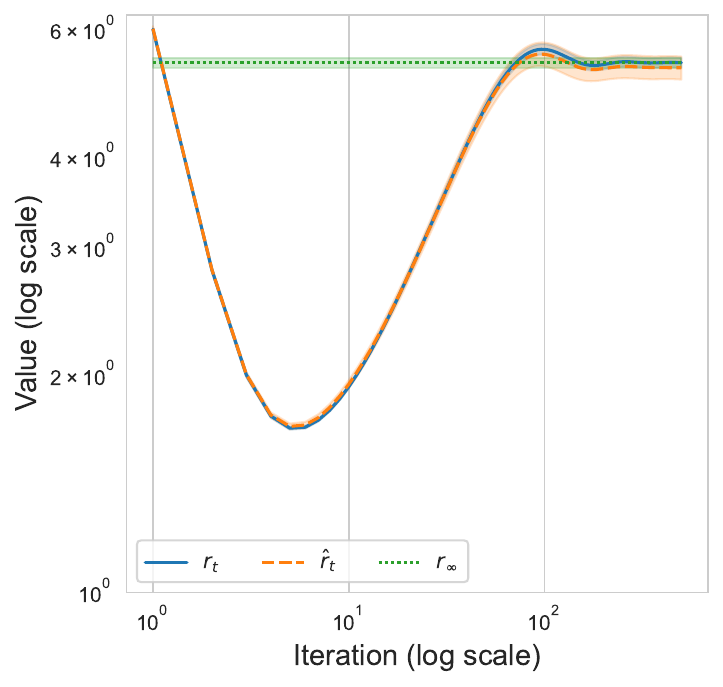}
      \caption{Risk curves versus iteration number.}
      \label{fig:LSE-AGD-risk-1500}
    \end{subfigure}
    \hfill
    \begin{subfigure}{0.47\linewidth}
      \centering
      \includegraphics[width=\linewidth]{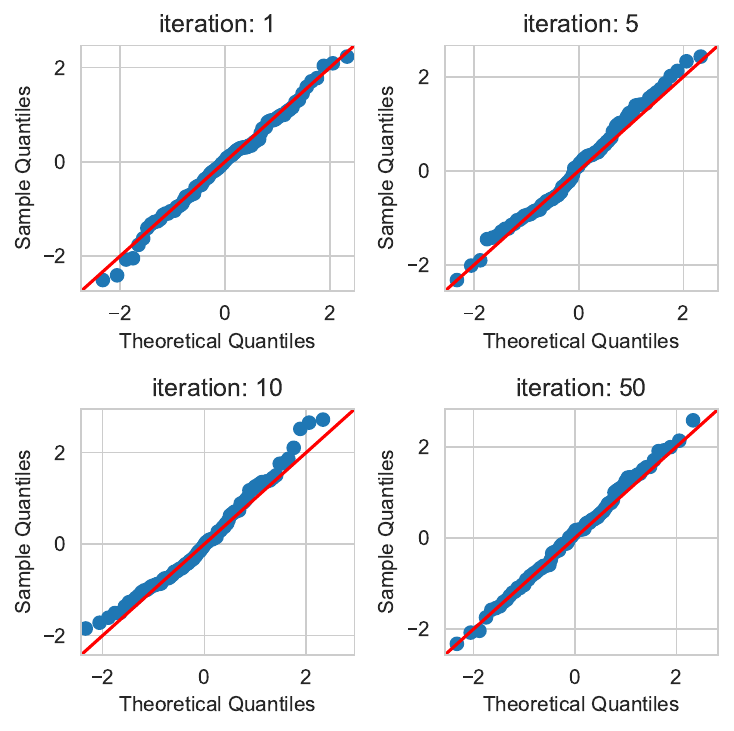}
      \caption{Q-Q plots of z-score at different iterations.}
      \label{fig:LSE-AGD-zscore-1500}
    \end{subfigure}
    \caption{Risk curves and qq-plots of z-score of AGD for $(n,p)=(1200, 1500)$.}
    \label{fig:LSE-AGD-1500}
  \end{figure}

\subsection{Iterative Shrinkage-Thresholding Algorithm (ISTA)} \label{sec:ISTA}
For regression problems with high-dimensional features, penalized regression are useful to achieve sparse solutions. 
Consider the Lasso regression 
\begin{equation}\label{eq:Lasso}
    \hbb \in 
    \argmin_{\bb\in \R^p} \frac{1}{2n}\|\by - \bX \bb\|^2 + \lambda \|\bb\|_1.
\end{equation}
This objective function of the above optimization has two parts, one is the squared loss, the other is an $\ell_1$ penalty. 
ISTA \citep{Daubechies2004} is a simple algorithm to solve \eqref{eq:Lasso} by imposing a soft-thresholding nonlinearity at each iteration.
Concretely,
let $\soft_{\theta}: \R^p \to \R^p$ denote the elementwise soft-thresholding operator, \ie $\soft_{\theta}(\bb)_j = (|\bb_j-\theta|)_+ \text{sgn}(\bb_j).$
ISTA can be viewed as the proximal gradient descent in \Cref{subsubsec:PGD}, and its iteration function $\bg_t$ is given by
\begin{equation}\label{eq:ISTA-update}
    \hbb^{t} 
    = \bg_t(\hbb^{t-1}, \bv^{t-1})
= \soft_{\lambda/L} (\hbb^{t-1} + L^{-1} \bv^{t-1}). 
\end{equation}
The function $\bg_t$ is Lipschitz continuous since the soft-thresholding
is 1-Lipschitz.
Let $S^t = \big\{j\in[p]: |\hbb^{t}_j + L^{-1} \bv^{t}_j|>\lambda/L\big\}$. 
For ISTA, the expressions of $\bJ_{t,t-1}$ and $\bD_{t,t-1}$ in \eqref{eq:pgd-J-D} are the diagonal matrices
\begin{equation*}
    \begin{aligned}
        \bJ_{t,t-1} 
        = \sum_{j\in S^{t-1}} \be_j \be_j^\top
        = \text{diag}\Bigl(
                \bigl(\mathbb I\{j\in S^{t-1}\}
                \bigr)_{j\in[p]}
            \Bigr) ,
        \qquad 
        \bD_{t,t-1} = L^{-1} \bJ_{t,t-1}.
    \end{aligned}
\end{equation*}
Substituting the above into \eqref{eq:pgd-cal-J-D} gives the expressions of $\calJ$ and $\calD$ for ISTA. 

\textbf{Simulation results.}
We apply ISTA to solve the Lasso regression \eqref{eq:Lasso} with two
regularization parameters $\lambda \in \{0.01, 0.1\}$.
The ISTA iterates converge to the Lasso estimator \eqref{eq:Lasso}
(see, e.g., \cite{beck2009fast}).
We compute the Lasso using the Python module \texttt{sklearn.linear\_model.Lasso} and
use this Lasso estimator to evaluate the limiting risk $r_{\infty}$.
\Cref{fig:Lasso-ISTA-risk-1500} showcases the risk estimator $\hat r_t$ and actual risk $r_t$ for each iteration $t$ using ISTA. Furthermore, \Cref{fig:Lasso-ISTA-zscore-1500} displays the Q-Q plots of the z-score \eqref{eq:z-score} for ISTA.
Again, the estimated risk curve closely aligns with the actual risk curve for both values of $\lambda$, and the corresponding z-scores closely approximate the standard normal distribution. 

\begin{figure}[H]
    \centering
    \begin{subfigure}{0.5\linewidth}
      \centering
      \includegraphics[width=\linewidth]{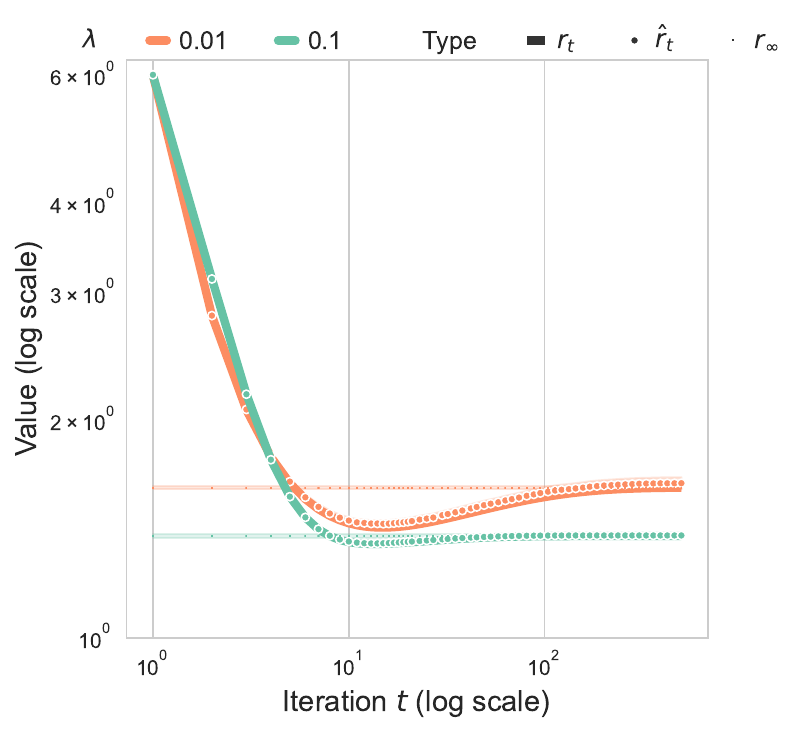}
      \caption{Risk curves versus iteration number.}
      \label{fig:Lasso-ISTA-risk-1500}
    \end{subfigure}
    \hfill
    \begin{subfigure}{0.49\linewidth}
      \centering
      \includegraphics[width=\linewidth]{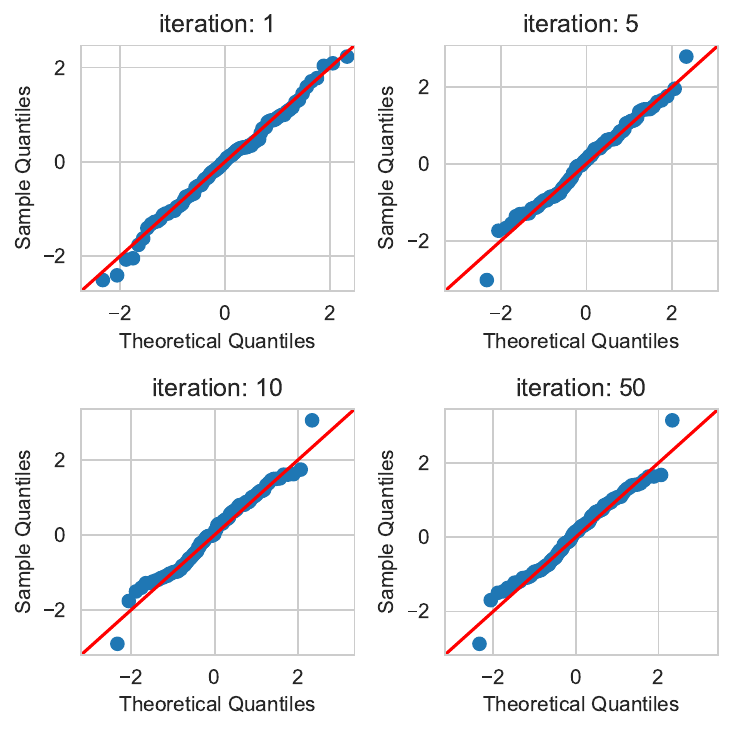}
      \caption{Q-Q plots of z-score for $\lambda=0.1$.}
      \label{fig:Lasso-ISTA-zscore-1500}
    \end{subfigure}
    \caption{Risk curves and qq-plots of z-score of ISTA for $(n,p)=(1200, 1500)$.}
    \label{fig:Lasso-ISTA-1500}
  \end{figure}

\subsection{Fast Iterative Shrinkage-Thresholding Algorithm (FISTA)}\label{sec:eg-FISTA}
Similar to the extension from GD to AGD, 
FISTA \citep{beck2009fast} is an accelerated version of ISTA,
incorporating momentum with the weights \eqref{eq:a}.
One advantage of FISTA is that it enjoys faster convergence rate than ISTA \citep{beck2009fast}. 
Using the same definitions of $a_t, w_t$ in \eqref{eq:a}, FISTA iterates the following steps with some initialization $\hbb^1 \in \R^p$:
\begin{align*}
    \hbb^{t} 
    &= \bg_t(\hbb^{t-1}, \hbb^{t-2},\bv^{t-1}, \bv^{t-2})\\
    &= (1 - w_{t-1})\soft_{\lambda/L}(\hbb^{t-1} + L^{-1} \bv^{t-1}) + w_{t-1} \soft_{\lambda/L}(\hbb^{t-2} + L^{-1} \bv^{t-2})  \text{ for } t\ge 2. 
\end{align*}
Using $S^t = \big\{j\in[p]: |\hbb^{t}_j + L^{-1} \bv^{t}_j|>\lambda/L\big\}$,
the matrices in \eqref{eq:D-J} are the diagonal matrices
\begin{equation*}
    \begin{aligned}
        \bJ_{t,t-1} 
        &= (1 - w_{t-1}) \sum_{j\in S^{t-1}} \be_j \be_j^\top,
        \quad
        &\bD_{t,t-1} &= \frac{1}{L} \bJ_{t,t-1},\\
        \bJ_{t,t-2} &= w_{t-1} \sum_{j\in S^{t-2}} \be_j \be_j^\top, 
        \quad 
        & \bD_{t, t-2} &= \frac{1}{L} \bJ_{t,t-2}.
    \end{aligned}
\end{equation*}
We obtain the expressions of $\calD$ and $\calJ$ for FISTA by substituting the above into \eqref{eq:calD-J}. 
Similarly to the simulation results for ISTA, we present the risk curves and Q-Q plots of the z-score \eqref{eq:z-score} for FISTA in \Cref{fig:Lasso-FISTA-1500}. 
\begin{figure}[H]
    \centering
    \begin{subfigure}{0.5\linewidth}
      \centering
      \includegraphics[width=\linewidth]{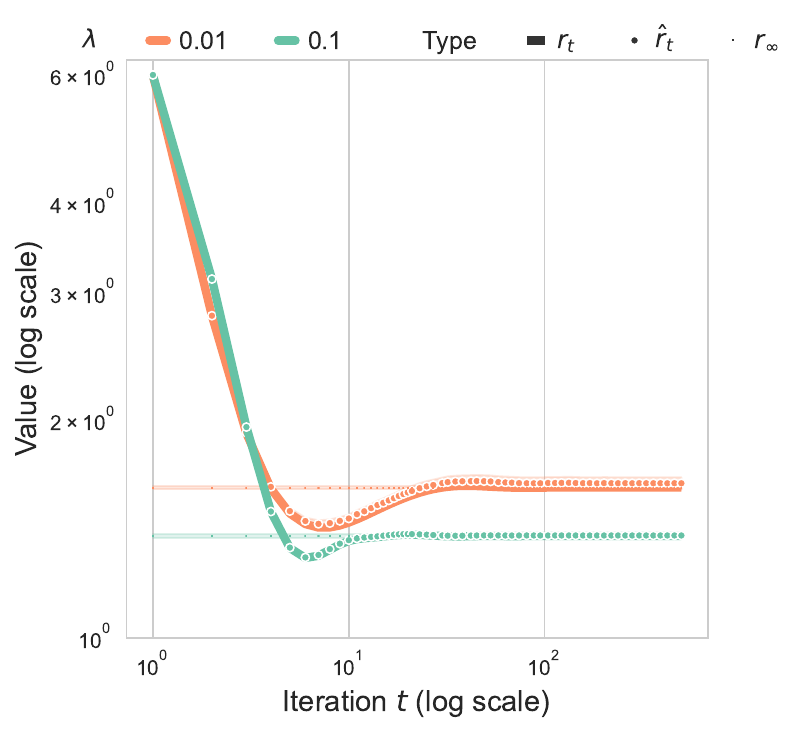}
      \caption{Risk curves versus iteration number.}
      \label{fig:Lasso-FISTA-risk-1500}
    \end{subfigure}
    \hfill
    \begin{subfigure}{0.49\linewidth}
      \centering
      \includegraphics[width=\linewidth]{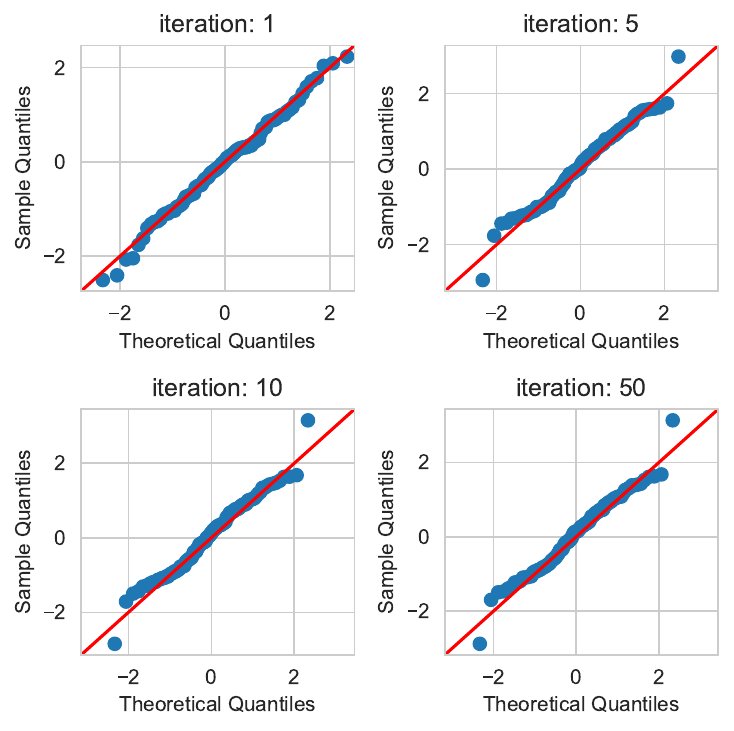}
      \caption{Q-Q plots of z-score for $\lambda=0.1$.}
      \label{fig:Lasso-FISTA-zscore-1500}
    \end{subfigure}
    \caption{Risk curves and qq-plots of z-score of FISTA for $(n,p)=(1200, 1500)$.}
    \label{fig:Lasso-FISTA-1500}
  \end{figure}
  
\subsection{Local quadratic approximation in non-convex penalized regression}\label{sec:eg-LQA}
In this subsection, we consider the folded-concave penalized least-squares,
\begin{equation}\label{eq:concave}
    \min_{\bb\in \R^p} \frac{1}{2n}\|\by - \bX \bb\|^2 + \sum_{j=1}^p \rho(|\bb_j|),
\end{equation}
where $\rho:\R_+\to \R$ is a concave penalty function. 
This encompasses SCAD \citep{fan2001variable} and MCP \citep{zhang10-mc+}.
For simulations, we set $\rho(\cdot)$ to be the MCP penalty 
with two positive tuning parameters $(\lambda, \tau)$, defined as
\begin{align}\label{eq:MCP}
    \rho(x;\lambda, \tau) = 
    \begin{cases}
        \lambda x - \frac{1}{2\tau} x^2& \text{if~} x \le \tau\lambda\\
        \frac12 \tau \lambda^2 & \text{if~} x > \tau\lambda
    \end{cases}.
\end{align}
For simplicity, we will omit the parameters $(\lambda, \tau)$ in the notation, using $\rho(x)$ instead of $\rho(x;\lambda, \tau)$. Consequently, the derivative of $\rho$ becomes $\rho'(x) = (\lambda - x/\tau) \I(x\le \tau\lambda)$.
As $\tau \to \infty$, $\rho(x)$ becomes $\lambda x$, and the optimization problem \eqref{eq:concave} then coincides with the Lasso \eqref{eq:Lasso}. 

In order to solve the non-convex penalized regression \eqref{eq:concave}, we consider the local isotropic quadratic approximation (LQA) to the least-squares loss 
$f(\bb):= \frac{1}{2n}\|\by - \bX \bb\|^2$ at a vector $\hbb^{t-1}$, namely
\begin{equation}\label{eq:Quad_Approx}
    f(\bb) 
\approx f(\hbb^{t-1}) + (\bb-\hbb^{t-1})^\top \nabla f(\hbb^{t-1}) 
+ (L/2) \|\bb - \hbb^{t-1}\|^2, 
\end{equation}
where $\nabla f(\hbb^{t-1})$ is the gradient of $f$ evaluated at $\hbb^{t-1}$, and
$L = n^{-1}\opnorm{\bX}^2$ as in previous sections.
Applying the above quadratic approximation \eqref{eq:Quad_Approx} to the least-squares in \eqref{eq:concave}
and ignoring the constant term yields the optimization problem
\begin{equation}\label{eq:MCP-quad-update}
    \hbb^{t} 
    = \argmin_{\bb\in \R^p} 
    \frac{1}{2} \|\bb - (\hbb^{t-1} + (nL)^{-1} \bX^\top (\by- \bX\hbb^{t-1}))\|^2
    + \frac1L \sum_{j=1}^p \rho(|\bb_j|).
\end{equation}
Let $\hbb_j^t$ and $\bv_j^t$ be the $j$-th entry of $\hbb^t$ and $\bv^t$, respectively.
For $\tau \ge 1/L$, the optimization problem \eqref{eq:MCP-quad-update} admits the closed-form solution
\begin{equation}\label{eq:LQA-update}
\begin{aligned}
    \hbb^{t}_j &=
\begin{cases}
    \soft_{\lambda/L} \Bigl(\hbb^{t-1}_j + \frac{1}{L} \bv^{t-1}_j\Bigr) (1 - \tfrac{1}{\tau L})^{-1} & \text{if } |\hbb^{t-1}_j + \tfrac{1}{L}  \bv^{t-1}_j| \le \tau\lambda, \\
    \hbb^{t-1}_j + \frac{1}{L} \bv^{t-1}_j & \text{otherwise}.
\end{cases}.
\end{aligned}
\end{equation}
In this case, $\bJ_{t,t-1}$ and $\bD_{t,t-1}$ in \eqref{eq:pgd-J-D}
are the diagonal matrices
\begin{align*}
    \bJ_{t,t-1}&= 
    \sum_{j=1}^p \be_j \be_j^\top \Bigl[ (1 - \tfrac{1}{\tau L})^{-1} \I\Bigl(|\hbb^{t-1}_j + \tfrac{1}{L} \bv^{t-1}_j|\in [\lambda/L, \tau\lambda]\Bigr) + 
    \I\Bigl(|\hbb^{t-1}_j + \tfrac{1}{L} \bv^{t-1}_j| > \tau\lambda\Bigr)\Bigr]
\end{align*}
and $\bD_{t,t-1} = \tfrac{1}{L} \bJ_{t,t-1}$.

\begin{remark}
    If $\tau = \infty$, the MCP function reduces to
    $\sum_{j=1}^p\rho(|\bb_j|) = \lambda \norm{\bb}_1$, so MCP is the same
    as the Lasso. 
    In this case, the LQA iterations \eqref{eq:LQA-update} become
    $\hbb^{t}= \soft_{\lambda/L}(\hbb^{t-1} + \frac{1}{L} \bv^{t-1})$, which is the same as ISTA iterations \eqref{eq:ISTA-update} in \Cref{sec:ISTA}.

\end{remark}

\textbf{Simulation results.} 
For the MCP penalty function \eqref{eq:MCP}, we consider $\lambda\in \{0.1, 0.2\}$ and $\tau=3$. 
We display the curves for the risk estimator $\hat r_t$ and the actual risk $r_t$ for each iteration $t$ in \Cref{fig:MCP-LQA-risk-1500}. Additionally, the Q-Q plots of the z-score \eqref{eq:z-score} are given in \Cref{fig:MCP-LQA-zscore-1500}.
\Cref{fig:MCP-LQA-risk-1500} shows that the estimated risk accurately
estimates the true risk.
The z-scores \eqref{eq:z-score} closely approximate the
standard normal distribution. 
Similar to observations in the aforementioned algorithms, LQA reaches its lowest risk level at around iteration 10. This suggests that early stopping could be beneficial for LQA for certain tuning parameters to improve generalization performance.
When comparing the lowest points on the risk curves for LQA with other algorithms (GD, AGD, ISTA, and FISTA), LQA achieves the lower risk
among the tested algorithms at the given tuning parameters.
\Cref{fig:MCP-LQA-zscore-1500} provides empirical support for the established asymptotic normality of the debiased LQA iterates in \Cref{thm:debias}. 
\begin{figure}[H]
    \centering
    \begin{subfigure}{0.5\linewidth}
      \centering
      \includegraphics[width=\linewidth]{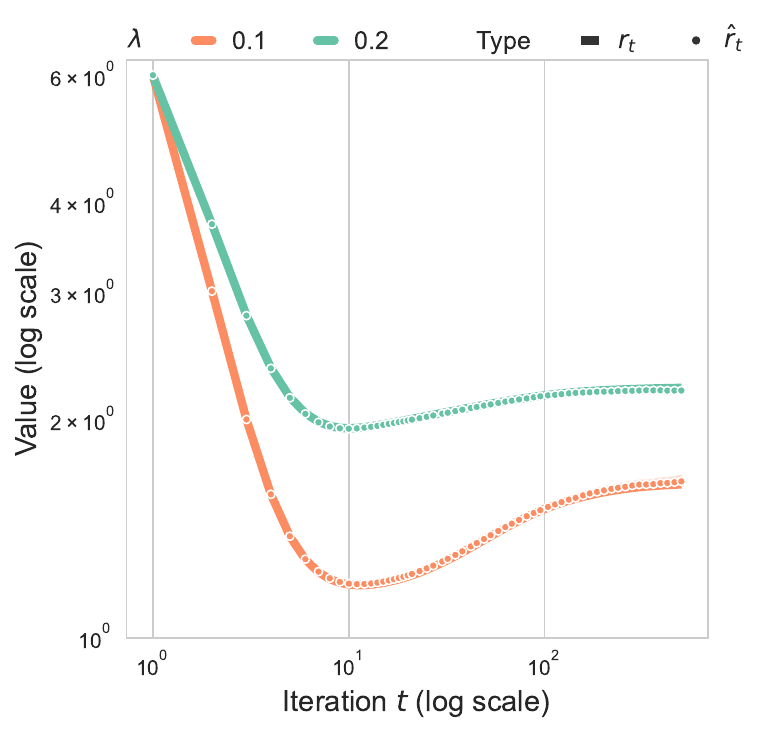}
      \caption{Risk curves versus iteration number.}
      \label{fig:MCP-LQA-risk-1500}
    \end{subfigure}
    \hfill
    \begin{subfigure}{0.49\linewidth}
      \centering
      \includegraphics[width=\linewidth]{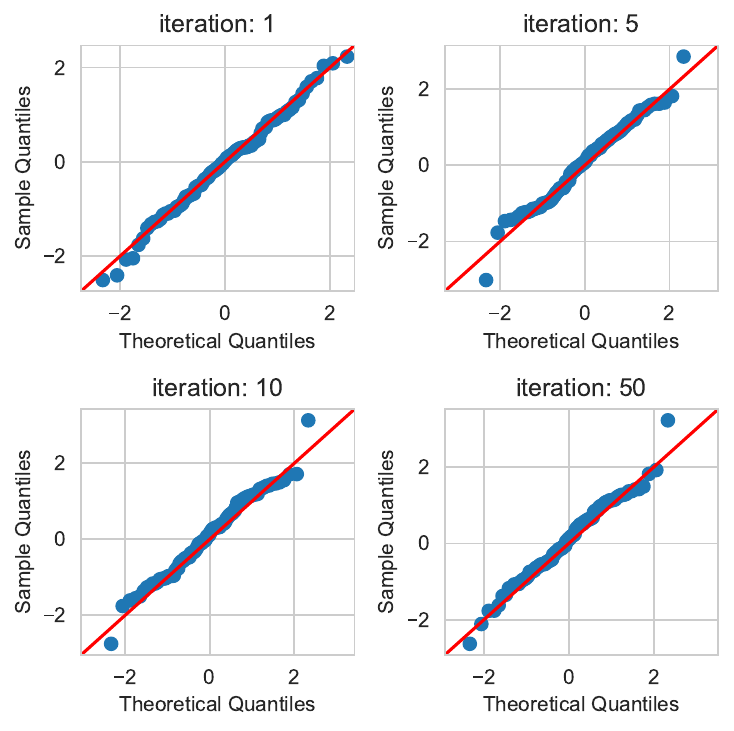}
      \caption{Q-Q plots of z-score for $\lambda=0.1$.}
      \label{fig:MCP-LQA-zscore-1500}
    \end{subfigure}
    \caption{Risk curves and qq-plots of z-score of LQA for $(n,p)=(1200, 1500)$.}
    \label{fig:MCP-LQA-1500}
  \end{figure}

\section{Efficient computation of the memory matrix}
\label{sec:computation}
Recall that the memory matrix $\hbA$
is crucial in the formulae of proposed risk estimator $\hat r_t$ in \eqref{eq:hat-rt} and debiased estimator $\hbb_j^{t, \rm debias}$ in \eqref{debiased_j}. 
We have provided specific expressions of $\hbA$ for various algorithms in \Cref{sec:examples}.
In this section, we provide an efficient way to compute the memory matrix $\hbA$ in \eqref{eq:hat-A} for algorithms with the general iterations \eqref{eq:iterates_all_previous}.

\subsection{Iteratively computing rows of $\hbA$}
Recall the definition of $\hbA$ in \eqref{eq:hat-A} is 
\begin{align*}
    \hbA
    = \frac{1}{n}\sum_{i=1}^n 
    \bigl(\bI_T \otimes (\be_i^\top \bX)\bigr)
    \Bigl(\bI_{pT} + \calD (\bI_T \otimes \tfrac{\bX^\top \bX}{n}) - \calJ\Bigr)^{-1} 
    \calD \bigl(\bI_T \otimes (\bX^\top \be_i)\bigr).
\end{align*}
The first apparent computational hurdle lies in inverting the large matrix $\bI_{pT} + \calD (\bI_T \otimes \tfrac{\bX^\top \bX}{n}) - \calJ$ of size ${pT \times pT}$. 
We now provide an efficient way to compute the memory matrix $\hbA$ without explicitly inverting this large matrix.
Recall that $\calD$ and $\calJ$ are $T\times T$ block lower triangular matrices given in \eqref{eq:calD-J} or \eqref{eq:all_previous_calD_calJ},
where each block is of size $p\times p$.
It follows that $\calJ - \calD (\bI_T \otimes \tfrac{\bX^\top \bX}{n})$ is also a $T\times T$ block lower triangular matrix with zero diagonal blocks.
Let $\bM_t\in\R^{p\times (pT)}$ be the $t$-th block row
of $\mathcal D$. 
Consider the linear system with unknowns $\bR_1\in\R^{p\times pT},
\dots,\bR_T\in\R^{p\times pT}$ given by
\begin{equation*}
    \Bigl(\bI_{pT} 
            -\calJ + \calD(\bI_T \otimes \tfrac{\bX^\top\bX}{n})
    \Bigr)
    \begin{bmatrix}
        \bR_1 \\
        \bR_2 \\
        \vdots \\
        \bR_{T-1} \\
        \bR_T
    \end{bmatrix}
    =
    \begin{bmatrix*}
        \bM_1\\
        \bM_2\\
        \vdots\\
        \bM_{T-1} \\
        \bM_T
    \end{bmatrix*}
    =
    \calD
\end{equation*}
which can be rewritten with $\bP_{t,s} = \bJ_{t,s} - \bD_{t,s}\bX^\top\bX/n$ as
\begin{equation}
    \label{large-system}
    \begin{bmatrix}
        \bI_p &\b0_{p\times p}&\dots &\dots&\b0_{p\times p}\\
        -\bP_{2,1} & \bI_p &\b0_{p\times p}&\empty&\vdots\\
        -\bP_{3,1} & -\bP_{3,2} & \bI_p &\ddots& \vdots\\\
        \vdots& \ddots & \ddots & \ddots &\b0_{p\times p}\\
        -\bP_{T,1} & \empty & \dots  & -\bP_{T,T-1}   &\bI_p
    \end{bmatrix}
    \begin{bmatrix}
        \bR_1 \\
        \bR_2 \\
        \vdots \\
        \bR_{T-1} \\
        \bR_T
    \end{bmatrix}
    =
    \begin{bmatrix}
        \bM_1\\
        \bM_2\\
        \vdots\\
        \bM_{T-1} \\
        \bM_T
    \end{bmatrix}
    =
    \calD.
\end{equation}
The left-most matrix being lower-triangular with identity diagonal blocks,
forward substitution provides the unique solution: starting with
$\bR_1 = \bM_1$ we obtain directly
$\bR_2 = \bM_2 + \bP_{2,1}\bR_1$, and more generally for all $t\ge 1$,
\begin{equation}
    \bR_t = \bM_t + \sum_{s=1}^{t-1}\bP_{t,s}\bR_s
    = \sum_{s=1}^{t-1} \be_s^\top \otimes \bD_{t,s}
    +
    \sum_{s=1}^{t-1}\bP_{t,s}\bR_s
    \label{R_t},
\end{equation}
where the second equality follows from the observation that
$\calD$ is block lower triangular with $t$-th row block
$\bM_t = \sum_{s=1}^{t-1}\be_s^\top \otimes \bD_{t,s}$.
Given
$(\bR_{s})_{s\le t-1}$, compute $\bR_t$ according to 
\eqref{R_t} and set
\begin{equation}
    \label{eq:computation_A}
    \hbA_{t,t'} = \tfrac 1n
    \trace\bigl[
        \bX \bR_t (\be_{t'} \otimes \bX^\top)
    \bigr]
    \quad\text{ for } t'<t,
    \quad \hbA_{t,t'} = 0 \quad \text{ for }t'\ge t.
\end{equation}
If only one or two blocks $\bP_{t,s}$ per row are nonzero
as in \eqref{eq:calD-J}, which is satisfied for all examples of
\Cref{sec:examples}, the recursion
\eqref{R_t} simplifies in this case to
\begin{equation}
    \label{R_t_only_two}
    \bR_t 
    = 
    \sum_{s=t-2}^{t-1} \be_s^\top \otimes \bD_{t,s}
    + \sum_{s=t-2}^{t-1} \bigl(\bJ_{t,s} - \bD_{t,s}\tfrac{\bX^\top\bX}{n}\bigr)\bR_{s},
\end{equation}
followed by \eqref{eq:computation_A}.
Notably, we may compute each $\bR_t$ recursively while
only keeping in memory $\bR_{t-1}$ and $\bR_{t-2}$,
both of size $p\times (pT)$, at each step.

\subsection{Hutchinson's trace approximation}

The above computation of $\hbA$ using \eqref{eq:computation_A} can still be prohibitive, as it requires storing in memory
matrices $\bR_t,\bR_{t-1}$ of size at most $p\times (pT)$ at each step, and perform
matrix-matrix products with dimensions $p\times p$ and $p\times (pT)$
in \eqref{R_t_only_two}.
We now describe an efficient way to approximate the entries of $\hbA$
while avoiding to store intermediate matrices of size $p\times (pT)$.
Let $m\ge 1$ be a small constant integer; in simulations
we have noticed that $m=1,2$ or $3$ already gives accurate estimates.
We propose to approximate the trace \eqref{eq:computation_A}
for the $(t,s)$-th entry of $\hbA$
using  Hutchinson's trace approximation \citep{hutchinson1990stochastic}
$$
\trace(\bM)
\approx \trace(\bW^\top \bM \bW)
$$
for any matrix $\bM\in\R^{n\times n}$,
where 
$\bW\in\R^{n\times m}$ is a random matrix with i.i.d. entries uniformly
distributed in $\{\frac{1}{\sqrt m}, \frac{-1}{\sqrt m}\}$.
The Hanson-Wright inequality ensures that the above approximation
holds with high-probability when $\|\bM\|_{\rm F}$ is negligible
compared to $\trace(\bM)$.
Using this approximation, the computation of the trace of a matrix
of size $n\times n$ is reduced to that of a matrix of size $m\times m$.

In our case, we need to approximate the trace in \eqref{eq:computation_A} using Hutchinson's approximation. 
To this end, 
let 
$\bW\in\R^{n\times m}$ be a random matrix with i.i.d. entries uniformly
distributed in $\{\frac{1}{\sqrt m}, \frac{-1}{\sqrt m}\}$
independently of everything else.
The matrix $\bW$ is only sampled once, and fixed throughout
the following computation.
Instead of computing recursively $\bR_t$ in \eqref{R_t},
we compute $\bar \bR_t = \bR_t (\bI_T \otimes (\bX^\top\bW))$ 
iteratively from \eqref{R_t} with the recursion
\begin{equation}
    \label{bar_R_t}
    \bar\bR_t = 
    \sum_{s=1}^{t-1} \be_s^\top \otimes (\bD_{t,s}\bX^\top\bW)
    + \sum_{s=1}^{t-1} \Bigl(\bJ_{t,s} - \bD_{t,s}\frac{\bX^\top\bX}{n}\Bigr)
    \bar \bR_{s}.
\end{equation}
Once $\bar\bR_t$ is available,
we compute the Hutchinson's approximation for each entry
of the $t$-th row of $\hbA$ in \eqref{eq:computation_A} by
\begin{align*}
    \hbA_{t,t'}^{H}
    &= \frac1n 
    \trace\Bigl[
        \bW^\top \bX \bar\bR_t(\be_{t'}\otimes \bI_m)
    \Bigr]
\end{align*}
for $t'<t$ and $0$ for $t'\ge t$.
For AGD and FISTA or any iterative algorithm
with $\calD,\calJ$ given by \eqref{eq:calD-J},
$\bD_{t,s}$ and $\bJ_{t,s}$ are 0 except for $s=t-1$ and $s=t-2$.
In this case the sums in the recursion for $\bR_t$ in \eqref{R_t} and
$\bar \bR_t$ in \eqref{bar_R_t} are reduced 
to $\sum_{s=t-2}^{t-1}$ with only two terms.
For ISTA and GD, $\bD_{t,s}$ and $\bJ_{t,s}$ are 0 except for $s=t-1$
so the sum is reduced to only one term at $s=t-1$.
In this case, the recursion
\eqref{bar_R_t} only uses $mT$ matrix-vector products
with matrix dimensions smaller than $\max\{n, p, T\}$.
In particular, if $mT \ll \min\{n,p\}$, it is never needed to
perform a matrix-matrix multiplication with two matrices with
both dimensions of order $n$ or $p$.
In terms of memory footprint, in the case \eqref{eq:calD-J},
only the last two $\bar\bR_{t-1}$
and $\bar\bR_{t-2}$ are needed to compute $\bar\bR_t$.
This is the same cost as storing a matrix of size $p\times (2mT)$,
which is negligible compared to storing $\bX\in\R^{n\times p}$
as long as $mT \ll n$.

Finally, remark that $\bar \bR_t$ is 0 except in its first $t-1$ column blocks,
so that only these first $t-1$ columns blocks need to be stored
and computed in the recursion \eqref{bar_R_t}.
More precisely,
let $\bV_{t,t'}\defas \bR_{t}(\be_{t'} \otimes \bX^\top \bW) \in \R^{p\times m}$ for $t'\le t-1$ to be $t'$ column block of $\bar\bR_t$. 
Then Hutchinson's approximation equals
\begin{align}\label{eq:hbA-approx}
    \hbA_{t,t'}^H
    = \tfrac{1}{n} \trace[\bW^\top \bX \bV_{t,t'}].
\end{align}
At step $t$, in order to compute
$\bV_t 
\defas [\bV_{t,1}, .., \bV_{t,t-1}]
\in \R^{p\times m(t-1)}$
in the case \eqref{eq:calD-J} satisfied by all examples of \Cref{sec:examples},
we have
the recursion formula
\begin{align*}
    \bV_t
    = \bP_{t,t-2} 
    \bigl[\bV_{t-2},~ \b0_p, \b0_p\bigr] 
    + \bP_{t,t-1} \bigl[\bV_{t-1},~ \b0_p\bigr] + 
    \bigl[
        \b0_{p\times m(t-3)},
        ~ \bD_{t,t-2}\bX^\top \bW,
        ~ \bD_{t,t-1}\bX^\top \bW
    \bigr],
\end{align*}
where $\bP_{t,s} = \bJ_{t,s} - \bD_{t,s}\bX^\top\bX/n$ as above.
In the case of ISTA and FISTA where the nonlinear functions use the soft-thresholding operator,
these approximations let us compute the iterations \eqref{recursion_intro_2_previous_iterates} and the memory matrix \eqref{eq:hat-A} with problem
dimensions $n=25,000;p=40,000;T=30$
on a laptop with 32GB of RAM within four minutes.

\section{Discussion}
This paper introduces a novel procedure to estimate the out-of-sample prediction error for iterates of gradient descent type algorithms, in the context of high-dimensional regression.
As illustrated in \Cref{sec:examples}, this risk-estimation procedure is applicable to a wide range of algorithms commonly used in optimization, including gradient descent, Nesterov's accelerated gradient descent, iterative shrinkage-thresholding algorithm, fast iterative shrinkage-thresholding algorithm \citep{beck2009fast}, and local quadratic approximation. 
The proposed procedure do not require the knowledge of the design
covariance $\bSigma$ or of the noise level $\sigma^2$.
The estimates allow the statistician to leverage
the benefits of early stopping,
by selecting an early iteration $\hat t\in [T]$ that minimizes
the generalization error among the first $T$ iterations
up to negligible error
(\Cref{cor:early}).
We have further established the asymptotic normality of the entries of the iterates after a debiasing correction, which can be used to construct confidence intervals for the $j$-th of the ground-truth $\bb^*$
when $\bSigma^{-1}\be_j$ is known or can be estimated
(\Cref{thm:debias}).
Extensive numerical simulations in \Cref{sec:examples} demonstrate that
the proposed estimate is accurate,
and that the z-scores defining the confidence intervals
are approximately standard normal.

Here we highlight several directions for further exploration. 
A first avenue for future research is to improve the dependence
of the bounds on the number $T$ of iterations. Currently,
upper bounds in \Cref{thm:generalization-error} and other main results
involve constants that worst than exponential in $T$, while we observe
in simulations that the risk estimate is still accurate over the 
whole trajectory for $T\gg \log n$. Another direction of interest
is to generalize
the estimates of the present paper beyond the square loss
in \eqref{eq:b-hat}, for instance
with the Huber or least-absolute deviation loss in robust regression,
or the logistic loss in classification problems.
Another generalization concerns randomized versions of GD such
as stochastic gradient descent.
Finally, our proofs rely crucially on the Gaussianity of the design,
and it would be of interest to extend the validity of the estimates
proposed here to different design distributions. Recent progress
has been made in this direction
\citep{montanari2022universality,hu2022universality,han2023universality,pesce2023gaussian,dudeja2023universality} regarding the universality
of the training and generalization error of minimizers such as
\eqref{eq:b-hat}. An extra challenge presented here is that the validity
of $\hat r_t$ for estimating $r_t$ requires not only universality of the training and generalization error, but also universality the weights $\hat w_{t,s}$.

\begin{supplement}
    \stitle{Supplement A}
    \sdescription{This supplement contains addtional numerical results and proofs.}
    \end{supplement}
    \begin{supplement}
    \stitle{Supplement B}
    \sdescription{This supplement contains the code and instruction to produce the numerical results.}
    \end{supplement}
\bibliography{main}  

\begin{thebibliography}{79}
\providecommand{\natexlab}[1]{#1}
\providecommand{\url}[1]{\texttt{#1}}
\expandafter\ifx\csname urlstyle\endcsname\relax
  \providecommand{\doi}[1]{doi: #1}\else
  \providecommand{\doi}{doi: \begingroup \urlstyle{rm}\Url}\fi

\bibitem[Ali et~al.(2019)Ali, Kolter, and Tibshirani]{ali2019continuous}
Alnur Ali, J~Zico Kolter, and Ryan~J Tibshirani.
\newblock A continuous-time view of early stopping for least squares
  regression.
\newblock In \emph{The 22nd international conference on artificial intelligence
  and statistics}, pages 1370--1378. PMLR, 2019.

\bibitem[Ali et~al.(2020)Ali, Dobriban, and Tibshirani]{ali2020implicit}
Alnur Ali, Edgar Dobriban, and Ryan Tibshirani.
\newblock The implicit regularization of stochastic gradient flow for least
  squares.
\newblock In \emph{International conference on machine learning}, pages
  233--244. PMLR, 2020.

\bibitem[Ambrosio and Dal~Maso(1990)]{ambrosio1990general}
Luigi Ambrosio and Gianni Dal~Maso.
\newblock A general chain rule for distributional derivatives.
\newblock \emph{Proceedings of the American Mathematical Society}, 108\penalty0
  (3):\penalty0 691--702, 1990.

\bibitem[Bandeira et~al.(2013)Bandeira, Dobriban, Mixon, and
  Sawin]{bandeira2012certifying}
Afonso~S. Bandeira, Edgar Dobriban, Dustin~G. Mixon, and William~F. Sawin.
\newblock Certifying the restricted isometry property is hard.
\newblock \emph{IEEE Trans. Inform. Theory}, 59\penalty0 (6):\penalty0
  3448--3450, 2013.
\newblock ISSN 0018-9448.

\bibitem[Bayati and Montanari(2011)]{bayati2011dynamics}
Mohsen Bayati and Andrea Montanari.
\newblock The dynamics of message passing on dense graphs, with applications to
  compressed sensing.
\newblock \emph{IEEE Trans. Inform. Theory}, 57\penalty0 (2):\penalty0
  764--785, 2011.
\newblock ISSN 0018-9448.

\bibitem[Bayati and Montanari(2012)]{bayati2012lasso}
Mohsen Bayati and Andrea Montanari.
\newblock The {LASSO} risk for {G}aussian matrices.
\newblock \emph{IEEE Trans. Inform. Theory}, 58\penalty0 (4):\penalty0
  1997--2017, 2012.
\newblock ISSN 0018-9448.

\bibitem[Beck and Teboulle(2009)]{beck2009fast}
Amir Beck and Marc Teboulle.
\newblock A fast iterative shrinkage-thresholding algorithm for linear inverse
  problems.
\newblock \emph{SIAM J. Imaging Sci.}, 2\penalty0 (1):\penalty0 183--202, 2009.

\bibitem[Bellec and Shen(2022)]{bellec2021derivatives}
Pierre~C Bellec and Yiwei Shen.
\newblock Derivatives and residual distribution of regularized m-estimators
  with application to adaptive tuning.
\newblock In \emph{Conference on Learning Theory}, pages 1912--1947. PMLR,
  2022.

\bibitem[Bellec and Zhang(2021)]{bellec2021second}
Pierre~C. Bellec and Cun-Hui Zhang.
\newblock Second-order {S}tein: {SURE} for {SURE} and other applications in
  high-dimensional inference.
\newblock \emph{Ann. Statist.}, 49\penalty0 (4):\penalty0 1864--1903, 2021.
\newblock ISSN 0090-5364.

\bibitem[Bellec and Zhang(2022{\natexlab{a}})]{bellec2022biasing}
Pierre~C. Bellec and Cun-Hui Zhang.
\newblock De-biasing the lasso with degrees-of-freedom adjustment.
\newblock \emph{Bernoulli}, 28\penalty0 (2):\penalty0 713--743,
  2022{\natexlab{a}}.
\newblock ISSN 1350-7265.

\bibitem[Bellec and
  Zhang(2022{\natexlab{b}})]{bellec_zhang2019debiasing_adjust}
Pierre~C. Bellec and Cun-Hui Zhang.
\newblock De-biasing the lasso with degrees-of-freedom adjustment.
\newblock \emph{Bernoulli}, 28\penalty0 (2):\penalty0 713--743,
  2022{\natexlab{b}}.
\newblock ISSN 1350-7265.

\bibitem[Bellec and Zhang(2023)]{bellec2023debiasing}
Pierre~C. Bellec and Cun-Hui Zhang.
\newblock Debiasing convex regularized estimators and interval estimation in
  linear models.
\newblock \emph{Ann. Statist.}, 51\penalty0 (2):\penalty0 391--436, 2023.
\newblock ISSN 0090-5364.

\bibitem[Bellec et~al.(2018)Bellec, Lecu\'e, and Tsybakov]{bellec2016slope}
Pierre~C. Bellec, Guillaume Lecu\'e, and Alexandre~B. Tsybakov.
\newblock Slope meets lasso: Improved oracle bounds and optimality.
\newblock \emph{Ann. Statist.}, 46\penalty0 (6B):\penalty0 3603--3642, 2018.
\newblock ISSN 0090-5364.

\bibitem[Berthier et~al.(2020)Berthier, Montanari, and
  Nguyen]{berthier2020state}
Rapha\"{e}l Berthier, Andrea Montanari, and Phan-Minh Nguyen.
\newblock State evolution for approximate message passing with non-separable
  functions.
\newblock \emph{Inf. Inference}, 9\penalty0 (1):\penalty0 33--79, 2020.
\newblock ISSN 2049-8764.

\bibitem[Bickel et~al.(2009)Bickel, Ritov, and
  Tsybakov]{bickel2009simultaneous}
Peter~J. Bickel, Ya'acov Ritov, and Alexandre~B. Tsybakov.
\newblock Simultaneous analysis of lasso and dantzig selector.
\newblock \emph{Ann. Statist.}, 37\penalty0 (4):\penalty0 1705--1732, 08 2009.
\newblock \doi{10.1214/08-AOS620}.

\bibitem[Bogdan et~al.(2015)Bogdan, van~den Berg, Sabatti, Su, and
  Cand\`es]{bogdan2015slope}
Ma\l~gorzata Bogdan, Ewout van~den Berg, Chiara Sabatti, Weijie Su, and
  Emmanuel~J. Cand\`es.
\newblock S{LOPE}---adaptive variable selection via convex optimization.
\newblock \emph{Ann. Appl. Stat.}, 9\penalty0 (3):\penalty0 1103--1140, 2015.
\newblock ISSN 1932-6157.

\bibitem[Bottou(2010)]{bottou2010large}
L{\'e}on Bottou.
\newblock Large-scale machine learning with stochastic gradient descent.
\newblock In \emph{Proceedings of COMPSTAT'2010: 19th International Conference
  on Computational StatisticsParis France, August 22-27, 2010 Keynote, Invited
  and Contributed Papers}, pages 177--186. Springer, 2010.

\bibitem[Boucheron et~al.(2013)Boucheron, Lugosi, and
  Massart]{boucheron2013concentration}
St{\'e}phane Boucheron, G{\'a}bor Lugosi, and Pascal Massart.
\newblock \emph{Concentration inequalities: A nonasymptotic theory of
  independence}.
\newblock Oxford University Press, 2013.

\bibitem[B{\"u}hlmann and Van De~Geer(2011)]{buhlmann2011statistics}
Peter B{\"u}hlmann and Sara Van De~Geer.
\newblock \emph{Statistics for high-dimensional data: methods, theory and
  applications}.
\newblock Springer Science \& Business Media, 2011.

\bibitem[B\"{u}hlmann and Yu(2003)]{Buhlmann2003boosting}
Peter B\"{u}hlmann and Bin Yu.
\newblock Boosting with the {$L_2$} loss: regression and classification.
\newblock \emph{J. Amer. Statist. Assoc.}, 98\penalty0 (462):\penalty0
  324--339, 2003.
\newblock ISSN 0162-1459.

\bibitem[Cai and Guo(2017)]{cai2017confidence}
T.~Tony Cai and Zijian Guo.
\newblock Confidence intervals for high-dimensional linear regression: minimax
  rates and adaptivity.
\newblock \emph{Ann. Statist.}, 45\penalty0 (2):\penalty0 615--646, 2017.
\newblock ISSN 0090-5364.

\bibitem[Celentano and Montanari(2022)]{celentano2019fundamental}
Michael Celentano and Andrea Montanari.
\newblock Fundamental barriers to high-dimensional regression with convex
  penalties.
\newblock \emph{Ann. Statist.}, 50\penalty0 (1):\penalty0 170--196, 2022.
\newblock ISSN 0090-5364.

\bibitem[Celentano et~al.(2020)Celentano, Montanari, and
  Wu]{celentano2020estimation}
Michael Celentano, Andrea Montanari, and Yuchen Wu.
\newblock The estimation error of general first order methods.
\newblock In \emph{Conference on Learning Theory}, pages 1078--1141. PMLR,
  2020.

\bibitem[Celentano et~al.(2021)Celentano, Cheng, and
  Montanari]{celentano2021high}
Michael Celentano, Chen Cheng, and Andrea Montanari.
\newblock The high-dimensional asymptotics of first order methods with random
  data.
\newblock \emph{arXiv preprint arXiv:2112.07572}, 2021.

\bibitem[Celentano et~al.(2023)Celentano, Montanari, and
  Wei]{celentano2020lasso}
Michael Celentano, Andrea Montanari, and Yuting Wei.
\newblock The {L}asso with general {G}aussian designs with applications to
  hypothesis testing.
\newblock \emph{Ann. Statist.}, 51\penalty0 (5):\penalty0 2194--2220, 2023.
\newblock ISSN 0090-5364.

\bibitem[Chandrasekher et~al.(2022)Chandrasekher, Lou, and
  Pananjady]{chandrasekher2022alternating}
Kabir~Aladin Chandrasekher, Mengqi Lou, and Ashwin Pananjady.
\newblock Alternating minimization for generalized rank one matrix sensing:
  Sharp predictions from a random initialization.
\newblock \emph{arXiv preprint arXiv:2207.09660}, 2022.

\bibitem[Chandrasekher et~al.(2023)Chandrasekher, Pananjady, and
  Thrampoulidis]{chandrasekher2023sharp}
Kabir~Aladin Chandrasekher, Ashwin Pananjady, and Christos Thrampoulidis.
\newblock Sharp global convergence guarantees for iterative nonconvex
  optimization with random data.
\newblock \emph{Ann. Statist.}, 51\penalty0 (1):\penalty0 179--210, 2023.
\newblock ISSN 0090-5364.

\bibitem[Dalalyan et~al.(2017)Dalalyan, Hebiri, and
  Lederer]{dalalyan2017prediction}
Arnak~S Dalalyan, Mohamed Hebiri, and Johannes Lederer.
\newblock On the prediction performance of the lasso.
\newblock \emph{Bernoulli}, 23\penalty0 (1):\penalty0 552--581, 2017.

\bibitem[Daubechies et~al.(2004)Daubechies, Defrise, and
  De~Mol]{Daubechies2004}
Ingrid Daubechies, Michel Defrise, and Christine De~Mol.
\newblock An iterative thresholding algorithm for linear inverse problems with
  a sparsity constraint.
\newblock \emph{Comm. Pure Appl. Math.}, 57\penalty0 (11):\penalty0 1413--1457,
  2004.
\newblock ISSN 0010-3640.

\bibitem[Davidson and Szarek(2001)]{DavidsonS01}
Kenneth~R Davidson and Stanislaw~J Szarek.
\newblock Local operator theory, random matrices and banach spaces.
\newblock \emph{Handbook of the geometry of Banach spaces}, 1\penalty0
  (317-366):\penalty0 131, 2001.

\bibitem[Donoho and Montanari(2016)]{donoho2016high}
David Donoho and Andrea Montanari.
\newblock High dimensional robust {M}-estimation: asymptotic variance via
  approximate message passing.
\newblock \emph{Probab. Theory Related Fields}, 166\penalty0 (3-4):\penalty0
  935--969, 2016.
\newblock ISSN 0178-8051.

\bibitem[Donoho et~al.(2009)Donoho, Maleki, and Montanari]{donoho2009message}
David~L Donoho, Arian Maleki, and Andrea Montanari.
\newblock Message-passing algorithms for compressed sensing.
\newblock \emph{Proceedings of the National Academy of Sciences}, 106\penalty0
  (45):\penalty0 18914--18919, 2009.

\bibitem[Dudeja et~al.(2023)Dudeja, M.~Lu, and Sen]{dudeja2023universality}
Rishabh Dudeja, Yue M.~Lu, and Subhabrata Sen.
\newblock Universality of approximate message passing with semirandom matrices.
\newblock \emph{The Annals of Probability}, 51\penalty0 (5):\penalty0
  1616--1683, 2023.

\bibitem[El~Alaoui et~al.(2022)El~Alaoui, Montanari, and
  Sellke]{el2022sampling}
Ahmed El~Alaoui, Andrea Montanari, and Mark Sellke.
\newblock Sampling from the sherrington-kirkpatrick gibbs measure via
  algorithmic stochastic localization.
\newblock In \emph{2022 IEEE 63rd Annual Symposium on Foundations of Computer
  Science (FOCS)}, pages 323--334. IEEE, 2022.

\bibitem[El~Karoui et~al.(2013)El~Karoui, Bean, Bickel, Lim, and
  Yu]{el_karoui2013robust}
Noureddine El~Karoui, Derek Bean, Peter~J Bickel, Chinghway Lim, and Bin Yu.
\newblock On robust regression with high-dimensional predictors.
\newblock \emph{Proceedings of the National Academy of Sciences}, 110\penalty0
  (36):\penalty0 14557--14562, 2013.

\bibitem[Fan and Li(2001)]{fan2001variable}
Jianqing Fan and Runze Li.
\newblock Variable selection via nonconcave penalized likelihood and its oracle
  properties.
\newblock \emph{J. Amer. Statist. Assoc.}, 96\penalty0 (456):\penalty0
  1348--1360, 2001.
\newblock ISSN 0162-1459.

\bibitem[Feng and Zhang(2019)]{feng2017sorted}
Long Feng and Cun-Hui Zhang.
\newblock Sorted concave penalized regression.
\newblock \emph{Ann. Statist.}, 47\penalty0 (6):\penalty0 3069--3098, 2019.
\newblock ISSN 0090-5364.

\bibitem[Friedman et~al.(2010)Friedman, Hastie, and
  Tibshirani]{friedman2010regularization}
Jerome~H. Friedman, Trevor Hastie, and Rob Tibshirani.
\newblock Regularization paths for generalized linear models via coordinate
  descent.
\newblock \emph{J. Stat. Softw.}, 33\penalty0 (1):\penalty0 1–22, 2010.

\bibitem[Gerbelot et~al.(2022)Gerbelot, Troiani, Mignacco, Krzakala, and
  Zdeborova]{gerbelot2022rigorous}
Cedric Gerbelot, Emanuele Troiani, Francesca Mignacco, Florent Krzakala, and
  Lenka Zdeborova.
\newblock Rigorous dynamical mean field theory for stochastic gradient descent
  methods.
\newblock \emph{arXiv preprint arXiv:2210.06591}, 2022.

\bibitem[Han and Shen(2023)]{han2023universality}
Qiyang Han and Yandi Shen.
\newblock Universality of regularized regression estimators in high dimensions.
\newblock \emph{The Annals of Statistics}, 51\penalty0 (4):\penalty0
  1799--1823, 2023.

\bibitem[Hastie et~al.(2022)Hastie, Montanari, Rosset, and
  Tibshirani]{Hastie2022surprises}
Trevor Hastie, Andrea Montanari, Saharon Rosset, and Ryan~J. Tibshirani.
\newblock Surprises in high-dimensional ridgeless least squares interpolation.
\newblock \emph{Ann. Statist.}, 50\penalty0 (2):\penalty0 949--986, 2022.
\newblock ISSN 0090-5364.

\bibitem[Hoppe et~al.(2023)Hoppe, Verdun, Laus, Krahmer, and
  Rauhut]{hoppe2023uncertainty}
Frederik Hoppe, Claudio~Mayrink Verdun, Hannah Laus, Felix Krahmer, and Holger
  Rauhut.
\newblock Uncertainty quantification for learned ista.
\newblock In \emph{2023 IEEE 33rd International Workshop on Machine Learning
  for Signal Processing (MLSP)}, pages 1--6. IEEE, 2023.

\bibitem[Hu and Lu(2022)]{hu2022universality}
Hong Hu and Yue~M Lu.
\newblock Universality laws for high-dimensional learning with random features.
\newblock \emph{IEEE Transactions on Information Theory}, 69\penalty0
  (3):\penalty0 1932--1964, 2022.

\bibitem[Hutchinson(1990)]{hutchinson1990stochastic}
M.~F. Hutchinson.
\newblock A stochastic estimator of the trace of the influence matrix for
  {L}aplacian smoothing splines.
\newblock \emph{Comm. Statist. Simulation Comput.}, 19\penalty0 (2):\penalty0
  433--450, 1990.
\newblock ISSN 0361-0918.

\bibitem[Javanmard and Montanari(2014)]{JavanmardM14a}
Adel Javanmard and Andrea Montanari.
\newblock Confidence intervals and hypothesis testing for high-dimensional
  regression.
\newblock \emph{J. Mach. Learn. Res.}, 15:\penalty0 2869--2909, 2014.
\newblock ISSN 1532-4435.

\bibitem[Javanmard and Montanari(2018)]{javanmard2018debiasing}
Adel Javanmard and Andrea Montanari.
\newblock Debiasing the {L}asso: optimal sample size for {G}aussian designs.
\newblock \emph{Ann. Statist.}, 46\penalty0 (6A):\penalty0 2593--2622, 2018.
\newblock ISSN 0090-5364.

\bibitem[Koltchinskii et~al.(2011)Koltchinskii, Lounici, and
  Tsybakov]{koltchinskii2011nuclear}
Vladimir Koltchinskii, Karim Lounici, and Alexandre~B. Tsybakov.
\newblock Nuclear-norm penalization and optimal rates for noisy low-rank matrix
  completion.
\newblock \emph{Ann. Statist.}, 39\penalty0 (5):\penalty0 2302--2329, 2011.
\newblock ISSN 0090-5364.

\bibitem[Lei et~al.(2018)Lei, Bickel, and El~Karoui]{lei2018asymptotics}
Lihua Lei, Peter~J. Bickel, and Noureddine El~Karoui.
\newblock Asymptotics for high dimensional regression {$M$}-estimates: fixed
  design results.
\newblock \emph{Probab. Theory Related Fields}, 172\penalty0 (3-4):\penalty0
  983--1079, 2018.
\newblock ISSN 0178-8051.

\bibitem[Lou et~al.(2024)Lou, Verchand, and Pananjady]{lou2024hyperparameter}
Mengqi Lou, Kabir~Aladin Verchand, and Ashwin Pananjady.
\newblock Hyperparameter tuning via trajectory predictions: Stochastic
  prox-linear methods in matrix sensing.
\newblock \emph{arXiv preprint arXiv:2402.01599}, 2024.

\bibitem[Loureiro et~al.(2021)Loureiro, Gerbelot, Cui, Goldt, Krzakala, Mezard,
  and Zdeborov{\'a}]{loureiro2021capturing}
Bruno Loureiro, Cedric Gerbelot, Hugo Cui, Sebastian Goldt, Florent Krzakala,
  Marc Mezard, and Lenka Zdeborov{\'a}.
\newblock Learning curves of generic features maps for realistic datasets with
  a teacher-student model.
\newblock \emph{Advances in Neural Information Processing Systems},
  34:\penalty0 18137--18151, 2021.

\bibitem[Luo et~al.(2023)Luo, Ren, and Barber]{luo2023iterative}
Yuetian Luo, Zhimei Ren, and Rina Barber.
\newblock Iterative approximate cross-validation.
\newblock In \emph{International Conference on Machine Learning}, pages
  23083--23102. PMLR, 2023.

\bibitem[Miolane and Montanari(2021)]{miolane2018distribution}
L\'{e}o Miolane and Andrea Montanari.
\newblock The distribution of the {L}asso: uniform control over sparse balls
  and adaptive parameter tuning.
\newblock \emph{Ann. Statist.}, 49\penalty0 (4):\penalty0 2313--2335, 2021.
\newblock ISSN 0090-5364.

\bibitem[Montanari and Saeed(2022)]{montanari2022universality}
Andrea Montanari and Basil~N Saeed.
\newblock Universality of empirical risk minimization.
\newblock In \emph{Conference on Learning Theory}, pages 4310--4312. PMLR,
  2022.

\bibitem[Montanari and Venkataramanan(2021)]{montanari2021estimation}
Andrea Montanari and Ramji Venkataramanan.
\newblock Estimation of low-rank matrices via approximate message passing.
\newblock \emph{Ann. Statist.}, 49\penalty0 (1):\penalty0 321--345, 2021.
\newblock ISSN 0090-5364.

\bibitem[Montanari and Wu(2022)]{montanari2022statistically}
Andrea Montanari and Yuchen Wu.
\newblock Statistically optimal first order algorithms: A proof via
  orthogonalization.
\newblock \emph{arXiv preprint arXiv:2201.05101}, 2022.

\bibitem[Nesterov(2003)]{nesterov2003introductory}
Yurii Nesterov.
\newblock \emph{Introductory lectures on convex optimization: A basic course},
  volume~87.
\newblock Springer Science \& Business Media, 2003.

\bibitem[Nesterov(1983)]{nesterov1983method}
Yurii~Evgen'evich Nesterov.
\newblock A method of solving a convex programming problem with convergence
  rate o$\backslash$bigl(k\^{}2$\backslash$bigr).
\newblock In \emph{Doklady Akademii Nauk}, volume 269, pages 543--547. Russian
  Academy of Sciences, 1983.

\bibitem[Patil et~al.(2024)Patil, Wu, and Tibshirani]{patil2024failures}
Pratik Patil, Yuchen Wu, and Ryan~J Tibshirani.
\newblock Failures and successes of cross-validation for early-stopped gradient
  descent.
\newblock \emph{arXiv preprint arXiv:2402.16793}, 2024.

\bibitem[Pesce et~al.(2023)Pesce, Krzakala, Loureiro, and
  Stephan]{pesce2023gaussian}
Luca Pesce, Florent Krzakala, Bruno Loureiro, and Ludovic Stephan.
\newblock Are gaussian data all you need? the extents and limits of
  universality in high-dimensional generalized linear estimation.
\newblock In \emph{International Conference on Machine Learning}, pages
  27680--27708. PMLR, 2023.

\bibitem[Raskutti et~al.(2014)Raskutti, Wainwright, and Yu]{raskutti2014early}
Garvesh Raskutti, Martin~J. Wainwright, and Bin Yu.
\newblock Early stopping and non-parametric regression: an optimal
  data-dependent stopping rule.
\newblock \emph{J. Mach. Learn. Res.}, 15:\penalty0 335--366, 2014.
\newblock ISSN 1532-4435.

\bibitem[Rosset et~al.(2004)Rosset, Zhu, and Hastie]{Rosset2003Boosting}
Saharon Rosset, Ji~Zhu, and Trevor Hastie.
\newblock Boosting as a regularized path to a maximum margin classifier.
\newblock \emph{J. Mach. Learn. Res.}, 5:\penalty0 941--973, 2004.
\newblock ISSN 1532-4435.

\bibitem[Sheng and Ali(2022)]{sheng2022accelerated}
Yue Sheng and Alnur Ali.
\newblock Accelerated gradient flow: Risk, stability, and implicit
  regularization.
\newblock \emph{arXiv preprint arXiv:2201.08311}, 2022.

\bibitem[Sur and Cand{\`e}s(2019)]{sur2018modern}
Pragya Sur and Emmanuel~J Cand{\`e}s.
\newblock A modern maximum-likelihood theory for high-dimensional logistic
  regression.
\newblock \emph{Proceedings of the National Academy of Sciences}, 116\penalty0
  (29):\penalty0 14516--14525, 2019.

\bibitem[Tan and Bellec(2024)]{tan2023multinomial}
Kai Tan and Pierre~C Bellec.
\newblock Multinomial logistic regression: Asymptotic normality on null
  covariates in high-dimensions.
\newblock \emph{Advances in Neural Information Processing Systems}, 36, 2024.

\bibitem[Tan et~al.(2022)Tan, Romon, and Bellec]{tan2022noise}
Kai Tan, Gabriel Romon, and Pierre~C Bellec.
\newblock Noise covariance estimation in multi-task high-dimensional linear
  models.
\newblock \emph{arXiv preprint arXiv:2206.07256}, 2022.

\bibitem[Thrampoulidis et~al.(2018)Thrampoulidis, Abbasi, and
  Hassibi]{thrampoulidis2018precise}
Christos Thrampoulidis, Ehsan Abbasi, and Babak Hassibi.
\newblock Precise error analysis of regularized {$M$}-estimators in high
  dimensions.
\newblock \emph{IEEE Trans. Inform. Theory}, 64\penalty0 (8):\penalty0
  5592--5628, 2018.
\newblock ISSN 0018-9448.

\bibitem[Tibshirani(1996)]{tibshirani1996regression}
Robert Tibshirani.
\newblock Regression shrinkage and selection via the lasso.
\newblock \emph{Journal of the Royal Statistical Society Series B: Statistical
  Methodology}, 58\penalty0 (1):\penalty0 267--288, 1996.

\bibitem[van~de Geer(2016)]{van2016estimation}
Sara van~de Geer.
\newblock Estimation and testing under sparsity: {\'E}cole d'{\'e}t{\'e} de
  probabilit{\'e}s de saint-flour xlv--2015.
\newblock \emph{Lecture Notes in Mathematics}, 2159, 2016.

\bibitem[van~de Geer et~al.(2014{\natexlab{a}})van~de Geer, B\"{u}hlmann,
  Ritov, and Dezeure]{GeerBR14}
Sara van~de Geer, Peter B\"{u}hlmann, Ya'acov Ritov, and Ruben Dezeure.
\newblock On asymptotically optimal confidence regions and tests for
  high-dimensional models.
\newblock \emph{Ann. Statist.}, 42\penalty0 (3):\penalty0 1166--1202,
  2014{\natexlab{a}}.
\newblock ISSN 0090-5364.

\bibitem[van~de Geer et~al.(2014{\natexlab{b}})van~de Geer, B\"{u}hlmann,
  Ritov, and Dezeure]{van2014asymptotically}
Sara van~de Geer, Peter B\"{u}hlmann, Ya'acov Ritov, and Ruben Dezeure.
\newblock On asymptotically optimal confidence regions and tests for
  high-dimensional models.
\newblock \emph{Ann. Statist.}, 42\penalty0 (3):\penalty0 1166--1202,
  2014{\natexlab{b}}.
\newblock ISSN 0090-5364.

\bibitem[Villani et~al.(2009)]{villani2009optimal}
C{\'e}dric Villani et~al.
\newblock \emph{Optimal transport: old and new}, volume 338.
\newblock Springer, 2009.

\bibitem[Wainwright(2009)]{Wainwright09}
Martin~J. Wainwright.
\newblock Sharp thresholds for high-dimensional and noisy sparsity recovery
  using {$\ell_1$}-constrained quadratic programming ({L}asso).
\newblock \emph{IEEE Trans. Inform. Theory}, 55\penalty0 (5):\penalty0
  2183--2202, 2009.
\newblock ISSN 0018-9448.

\bibitem[Ye and Zhang(2010)]{ye2010rate}
Fei Ye and Cun-Hui Zhang.
\newblock Rate minimaxity of the {L}asso and {D}antzig selector for the
  {$\ell_q$} loss in {$\ell_r$} balls.
\newblock \emph{J. Mach. Learn. Res.}, 11:\penalty0 3519--3540, 2010.
\newblock ISSN 1532-4435.

\bibitem[Yuan and Lin(2006)]{yuan2006model}
Ming Yuan and Yi~Lin.
\newblock Model selection and estimation in regression with grouped variables.
\newblock \emph{J. R. Stat. Soc. Ser. B Stat. Methodol.}, 68\penalty0
  (1):\penalty0 49--67, 2006.
\newblock ISSN 1369-7412.

\bibitem[Zhang(2010)]{zhang10-mc+}
Cun-Hui Zhang.
\newblock Nearly unbiased variable selection under minimax concave penalty.
\newblock \emph{Ann. Statist.}, 38\penalty0 (2):\penalty0 894--942, 2010.
\newblock ISSN 0090-5364.

\bibitem[Zhang and Zhang(2014{\natexlab{a}})]{ZhangSteph14}
Cun-Hui Zhang and Stephanie~S. Zhang.
\newblock Confidence intervals for low dimensional parameters in high
  dimensional linear models.
\newblock \emph{J. R. Stat. Soc. Ser. B. Stat. Methodol.}, 76\penalty0
  (1):\penalty0 217--242, 2014{\natexlab{a}}.
\newblock ISSN 1369-7412.

\bibitem[Zhang and Zhang(2014{\natexlab{b}})]{zhang2014confidence}
Cun-Hui Zhang and Stephanie~S. Zhang.
\newblock Confidence intervals for low dimensional parameters in high
  dimensional linear models.
\newblock \emph{J. R. Stat. Soc. Ser. B. Stat. Methodol.}, 76\penalty0
  (1):\penalty0 217--242, 2014{\natexlab{b}}.
\newblock ISSN 1369-7412.

\bibitem[Ziemer(1989)]{ziemer2012weakly}
William~P Ziemer.
\newblock \emph{Weakly differentiable functions: Sobolev spaces and functions
  of bounded variation}, volume 120.
\newblock Springer-Verlag New York, 1989.

\bibitem[Zou and Hastie(2005)]{zou2005regularization}
Hui Zou and Trevor Hastie.
\newblock Regularization and variable selection via the elastic net.
\newblock \emph{J. R. Stat. Soc. Ser. B Stat. Methodol.}, 67\penalty0
  (2):\penalty0 301--320, 2005.
\newblock ISSN 1369-7412.

\end{thebibliography}
\bibliographystyle{plainnat} 

\newpage
\appendix 
\appendixpage
\addappheadtotoc
\addtocontents{toc}{\protect\setcounter{tocdepth}{0}}
\section{Additional simulation results}\label{sec:app-simulation}
In this section, we provide more simulation results for the other two scenarios $(n,p) = (1200, 500)$ and $(n,p) = (1200, 1200)$, which are similar to the results for $(n,p) = (1200, 1500)$ presented in \Cref{sec:examples}.

\subsection{Under parametrization: $(n,p) = (1200, 500)$}

\subsubsection{GD and AGD}
For GD and AGD applied to solve least-squares problem, when $(n,p) = (1200, 500)$, we know that the iterate $\hbb^t$ converges to the ordinary least-squares estimate $(\bX^\top\bX)^{-1}\bX^\top \by$ as $t\to\infty$. 
Thus we are able to compute the limiting risk $r_{\infty}$ for GD and AGD.
We present the simulation results for GD and AGD in \Cref{fig:LSE-GD-500} and \Cref{fig:LSE-AGD-500}, respectively.

\begin{figure}[H]
    \centering
    \begin{subfigure}{0.5\linewidth}
      \centering
      \includegraphics[width=\linewidth]{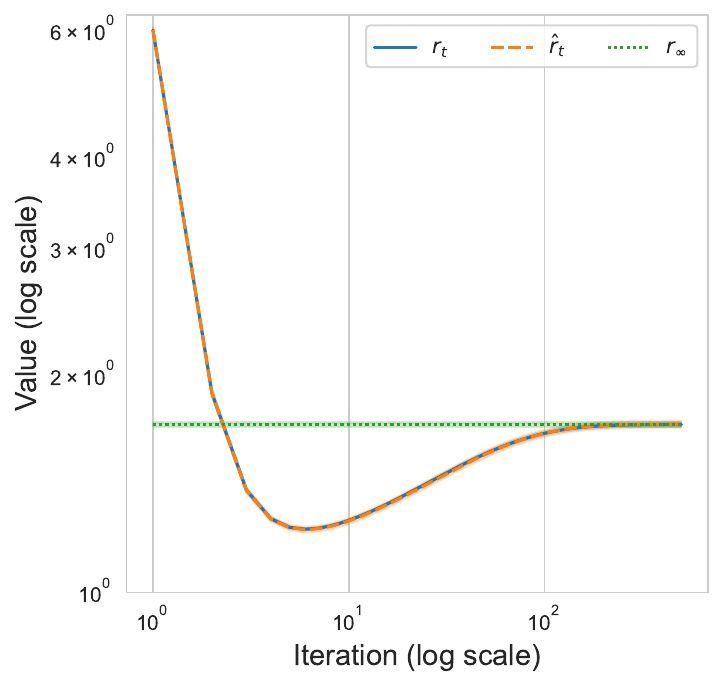}
      \caption{Risk curves versus iteration number.}
      \label{fig:LSE-GD-risk-500}
    \end{subfigure}
    \hfill
    \begin{subfigure}{0.49\linewidth}
      \centering
      \includegraphics[width=\linewidth]{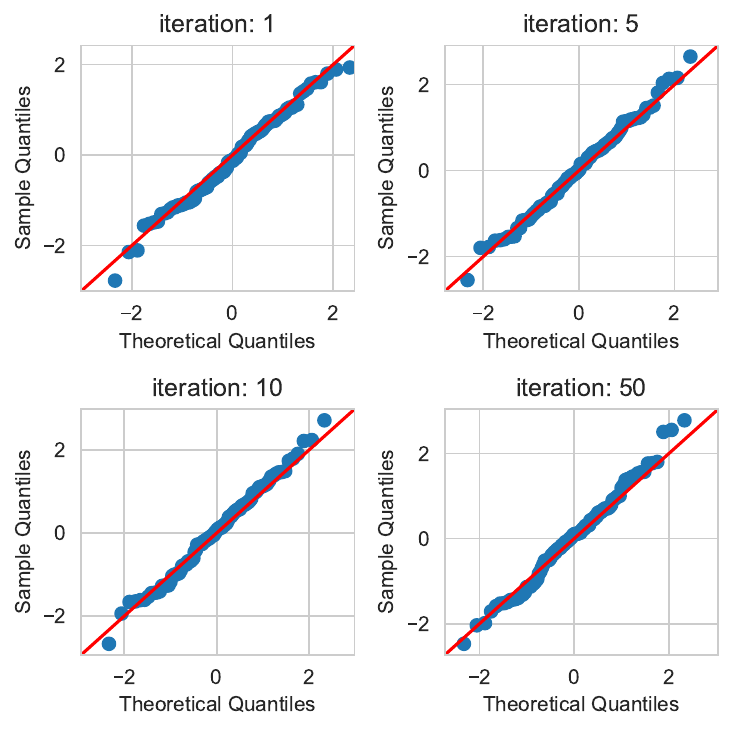}
      \caption{Q-Q plots of z-score for $\lambda = 0.1$.}
      \label{fig:LSE-GD-zscore-500}
    \end{subfigure}
    \caption{Risk curves and qq-plots of z-score of GD for $(n,p)=(1200, 500)$.}
    \label{fig:LSE-GD-500}
  \end{figure}
  \begin{figure}[H]
    \centering
    \begin{subfigure}{0.5\linewidth}
      \centering
      \includegraphics[width=\linewidth]{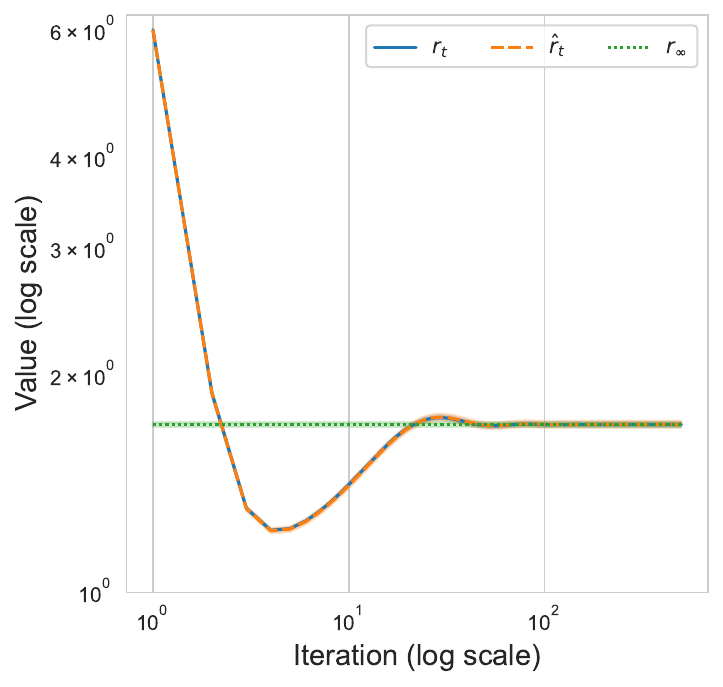}
      \caption{Risk curves versus iteration number.}
      \label{fig:LSE-AGD-risk-500}
    \end{subfigure}
    \hfill
    \begin{subfigure}{0.49\linewidth}
      \centering
      \includegraphics[width=\linewidth]{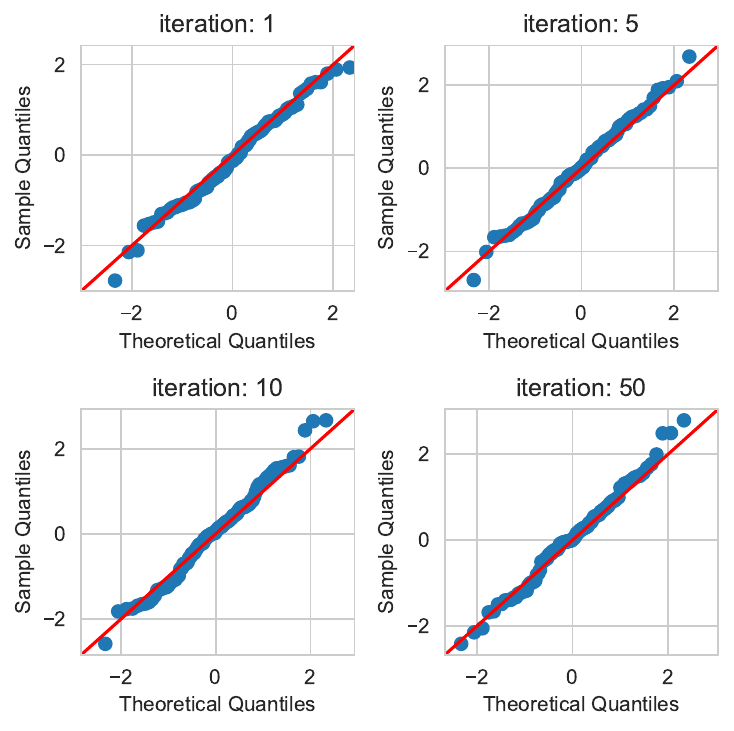}
      \caption{Q-Q plots of z-score for $\lambda = 0.1$.}
      \label{fig:LSE-AGD-zscore-500}
    \end{subfigure}
    \caption{Risk curves and qq-plots of z-score of AGD for $(n,p)=(1200, 500)$.}
    \label{fig:LSE-AGD-500}
  \end{figure}

\subsubsection{ISTA and FISTA}
We present the simulation results for ISTA and FISTA in \Cref{fig:Lasso-ISTA-500} and \Cref{fig:Lasso-FISTA-500}, respectively.

\begin{figure}[H]
    \centering
    \begin{subfigure}{0.5\linewidth}
      \centering
      \includegraphics[width=\linewidth]{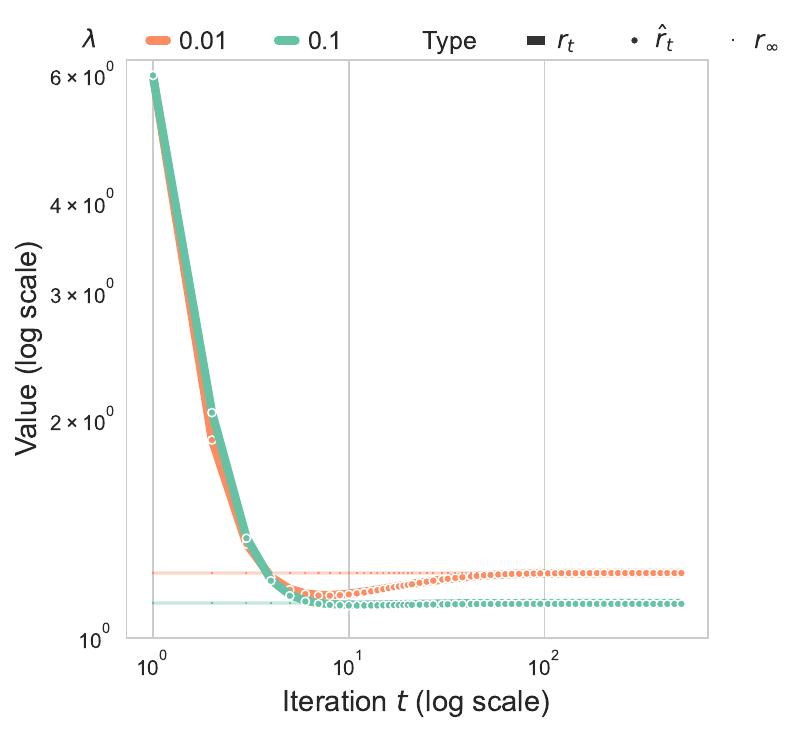}
      \caption{Risk curves versus iteration number.}
      \label{fig:Lasso-ISTA-risk}
    \end{subfigure}
    \hfill
    \begin{subfigure}{0.49\linewidth}
      \centering
      \includegraphics[width=\linewidth]{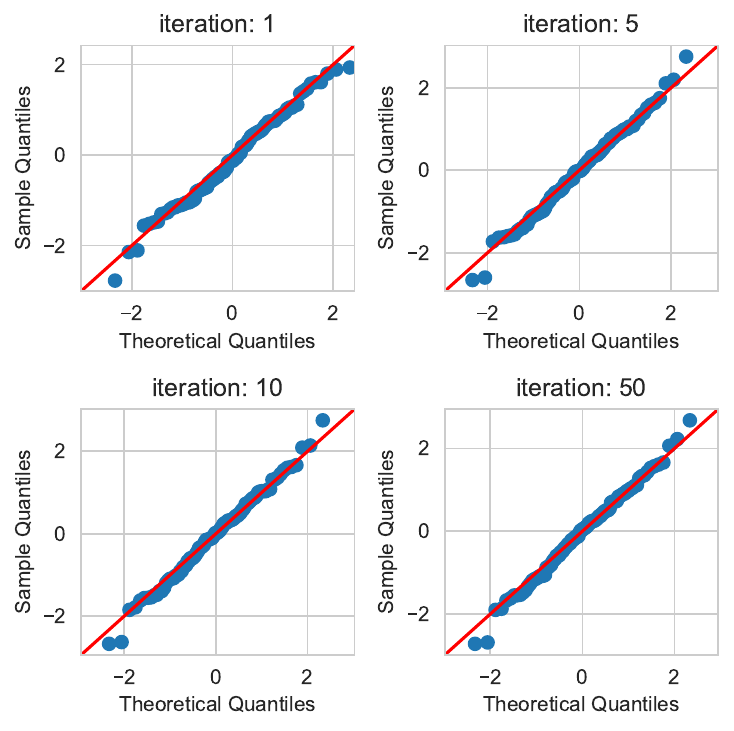}
      \caption{Q-Q plots of z-score for $\lambda=0.1$.}
      \label{fig:Lasso-ISTA-zscore}
    \end{subfigure}
    \caption{Risk curves and qq-plots of z-score of ISTA for $(n,p)=(1200, 500)$.}
    \label{fig:Lasso-ISTA-500}
  \end{figure}
  \begin{figure}[H]
    \centering
    \begin{subfigure}{0.5\linewidth}
      \centering
      \includegraphics[width=\linewidth]{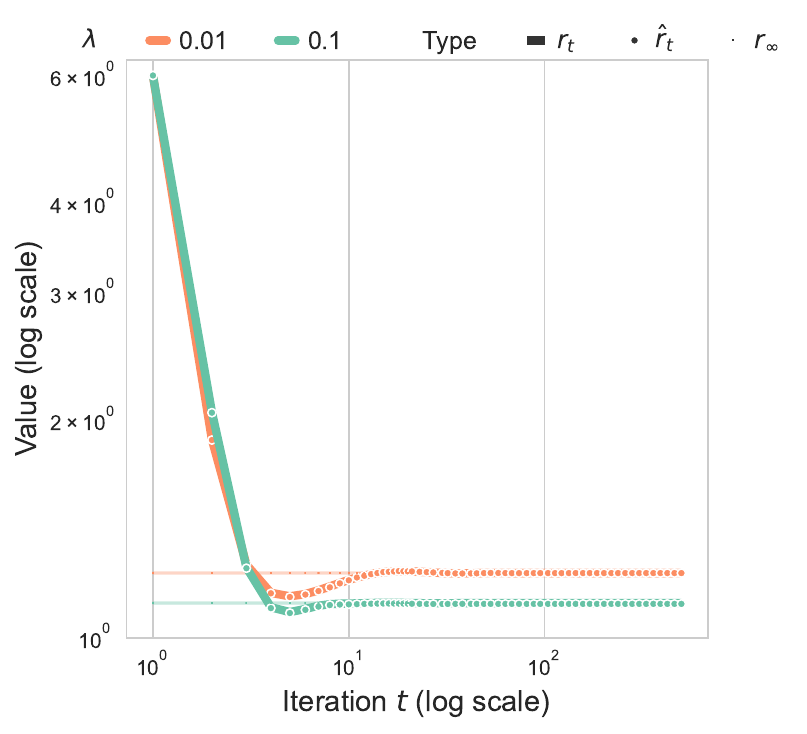}
      \caption{Risk curves versus iteration number.}
      \label{fig:Lasso-FISTA-risk}
    \end{subfigure}
    \hfill
    \begin{subfigure}{0.49\linewidth}
      \centering
      \includegraphics[width=\linewidth]{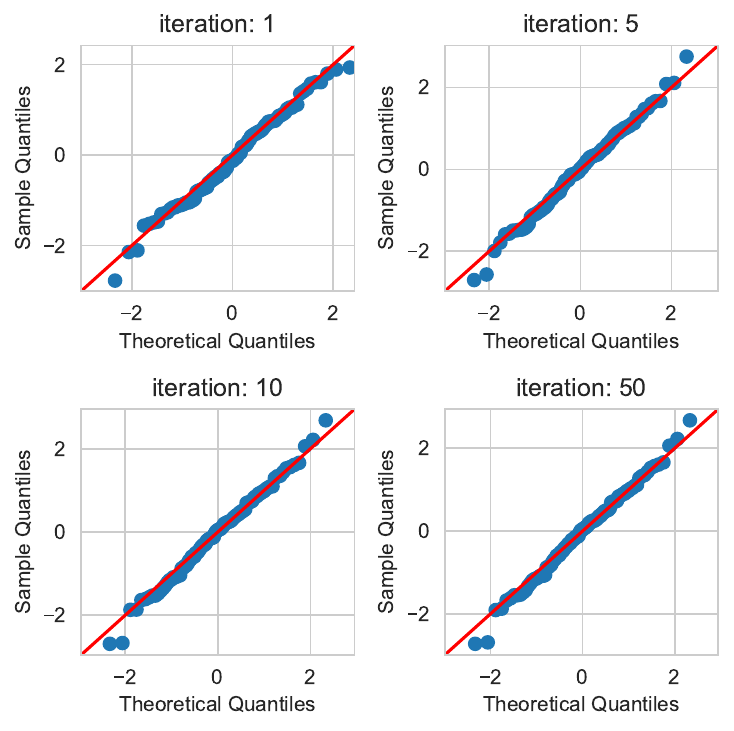}
      \caption{Q-Q plots of z-score for $\lambda = 0.1$.}
      \label{fig:Lasso-FISTA-zscore}
    \end{subfigure}
    \caption{Risk curves and qq-plots of z-score of FISTA for $(n,p)=(1200, 500)$.}
    \label{fig:Lasso-FISTA-500}
  \end{figure}

\subsubsection{LQA}
We present the simulation results for LQA in \Cref{fig:MCP-LQA-500}.
\begin{figure}[H]
    \centering
    \begin{subfigure}{0.5\linewidth}
      \centering
      \includegraphics[width=\linewidth]{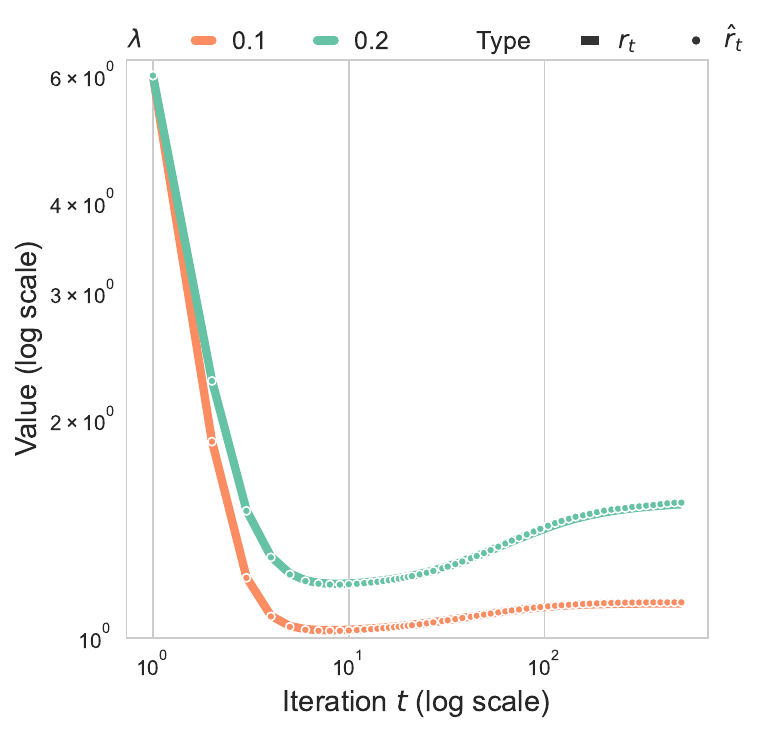}
      \caption{Risk curves versus iteration number.}
      \label{fig:MCP-LQA-risk}
    \end{subfigure}
    \hfill
    \begin{subfigure}{0.49\linewidth}
      \centering
      \includegraphics[width=\linewidth]{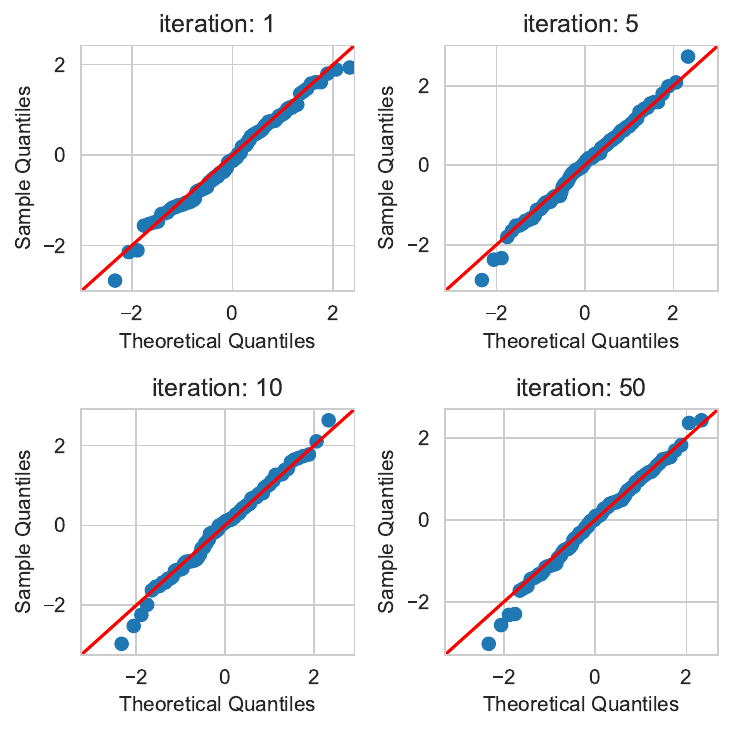}
      \caption{Q-Q plots of z-score for $\lambda = 0.1$.}
      \label{fig:MCP-LQA-zscore}
    \end{subfigure}
    \caption{Risk curves and qq-plots of z-score of LQA for $(n,p)=(1200, 500)$.}
    \label{fig:MCP-LQA-500}
  \end{figure}

\subsection{Equal parametrization: $(n,p) = (1200, 1200)$}
In this section, we present the simulation results for the equal parametrization scenario $(n,p) = (1200, 1200)$.

\subsubsection{GD and AGD}
For GD and AGD, we know that the iterates converge to the min-norm least-squares estimator \eqref{eq:tbb} as $t\to\infty$ \cite[Proposition 1]{Hastie2022surprises}. 
Under $n=p=1200$, we know that the
risk of the min-norm least-squares estimator is infinite \citep{Hastie2022surprises}, \ie, $r_{\infty} = +\infty$. 
So here we do not plot the horizontal line for $r_\infty$.

We present the simulation results for GD and AGD in \Cref{fig:LSE-GD-1200} and \Cref{fig:LSE-AGD-1200}, respectively.

\begin{figure}[H]
    \centering
    \begin{subfigure}{0.5\linewidth}
      \centering
      \includegraphics[width=\linewidth]{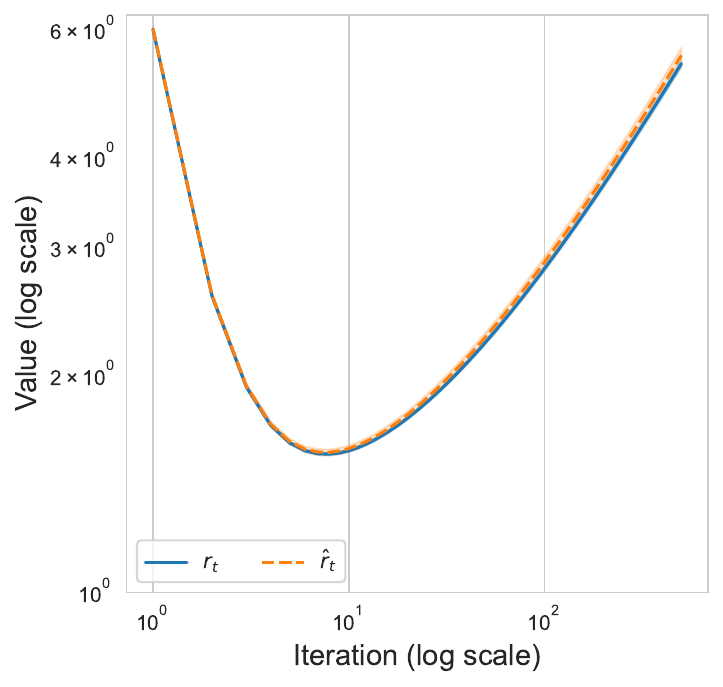}
      \caption{Risk curves versus iteration number.}
      \label{fig:LSE-GD-risk-1200}
    \end{subfigure}
    \hfill
    \begin{subfigure}{0.49\linewidth}
      \centering
      \includegraphics[width=\linewidth]{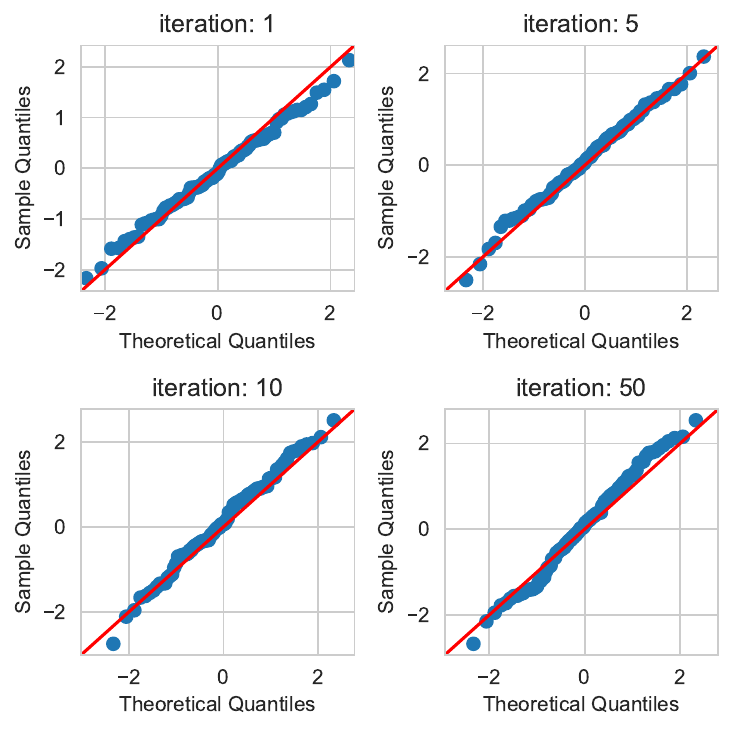}
      \caption{Q-Q plots of z-score for $\lambda = 0.1$.}
      \label{fig:LSE-GD-zscore-1200}
    \end{subfigure}
    \caption{Risk curves and qq-plots of z-score of GD for $(n,p)=(1200, 1200)$.}
    \label{fig:LSE-GD-1200}
  \end{figure}
  \begin{figure}[H]
    \centering
    \begin{subfigure}{0.5\linewidth}
      \centering
      \includegraphics[width=\linewidth]{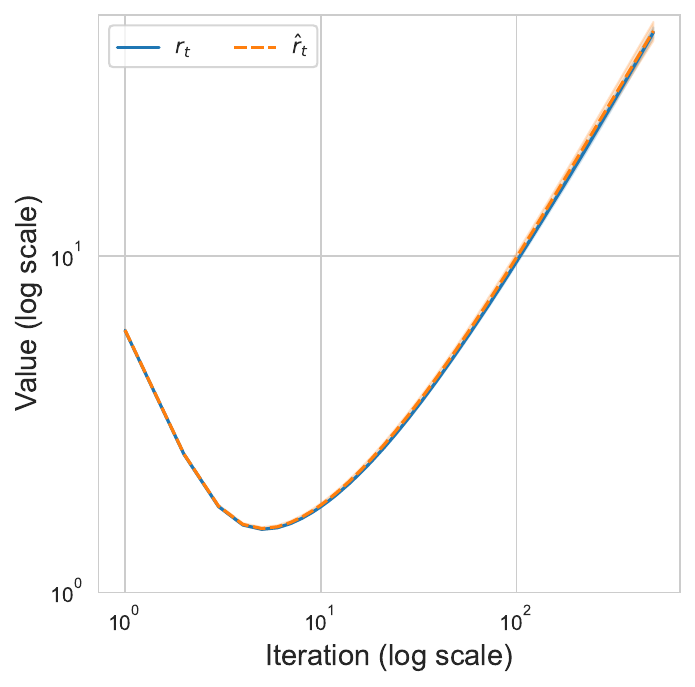}
      \caption{Risk curves versus iteration number.}
      \label{fig:LSE-AGD-risk-1200}
    \end{subfigure}
    \hfill
    \begin{subfigure}{0.49\linewidth}
      \centering
      \includegraphics[width=\linewidth]{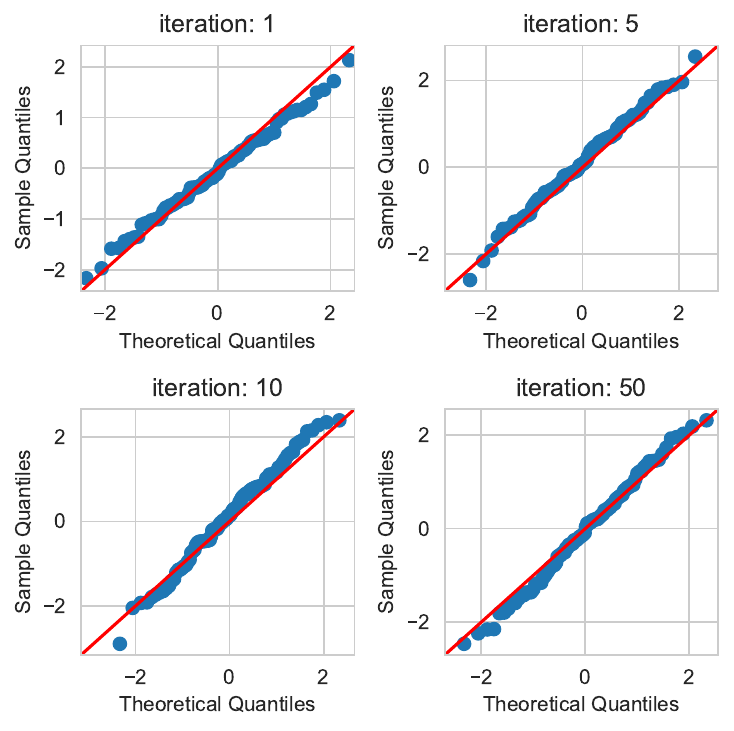}
      \caption{Q-Q plots of z-score for $\lambda = 0.1$.}
      \label{fig:LSE-AGD-zscore-1200}
    \end{subfigure}
    \caption{Risk curves and qq-plots of z-score of AGD for $(n,p)=(1200, 1200)$.}
    \label{fig:LSE-AGD-1200}
  \end{figure}

\subsubsection{ISTA and FISTA}
We present the simulation results for ISTA and FISTA in \Cref{fig:Lasso-ISTA-1200} and \Cref{fig:Lasso-FISTA-1200}, respectively.

\begin{figure}
    \centering
    \begin{subfigure}{0.5\linewidth}
      \centering
      \includegraphics[width=\linewidth]{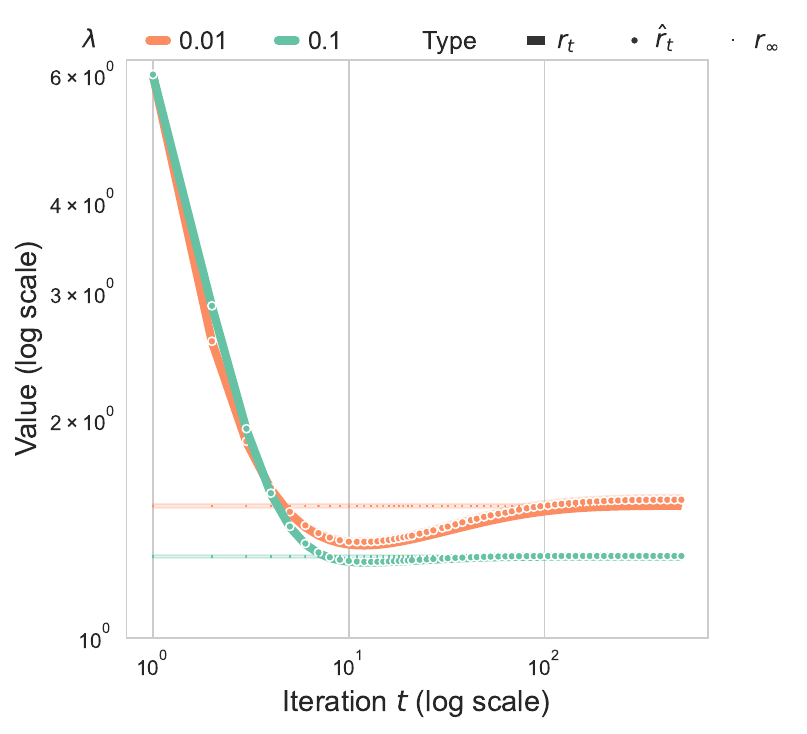}
      \caption{Risk curves versus iteration number.}
      \label{fig:Lasso-ISTA-risk-1200}
    \end{subfigure}
    \hfill
    \begin{subfigure}{0.49\linewidth}
      \centering
      \includegraphics[width=\linewidth]{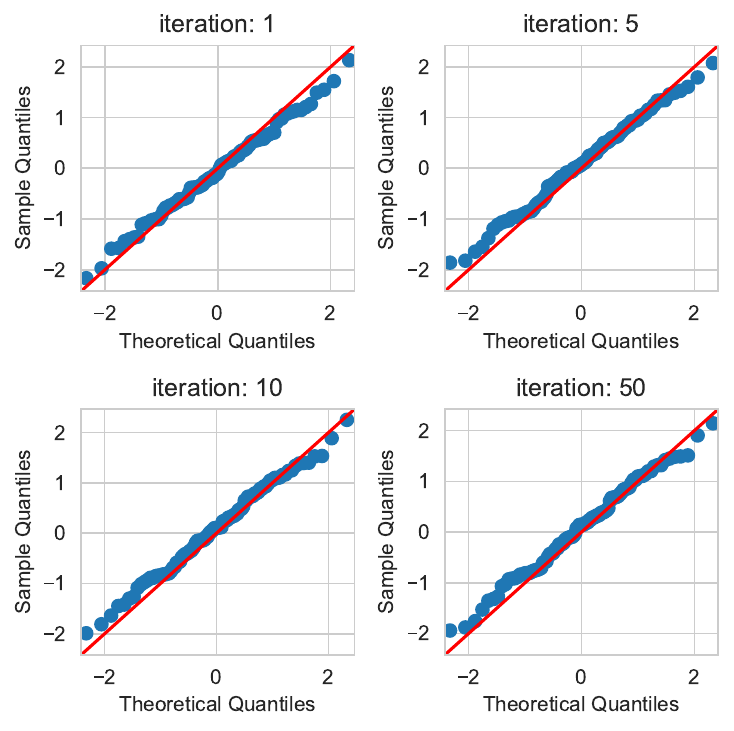}
      \caption{Q-Q plots of z-score for $\lambda = 0.1$.}
      \label{fig:Lasso-ISTA-zscore-1200}
    \end{subfigure}
    \caption{Risk curves and qq-plots of z-score of ISTA for $(n,p)=(1200, 1200)$.}
    \label{fig:Lasso-ISTA-1200}
  \end{figure}
  \begin{figure}
    \centering
    \begin{subfigure}{0.5\linewidth}
      \centering
      \includegraphics[width=\linewidth]{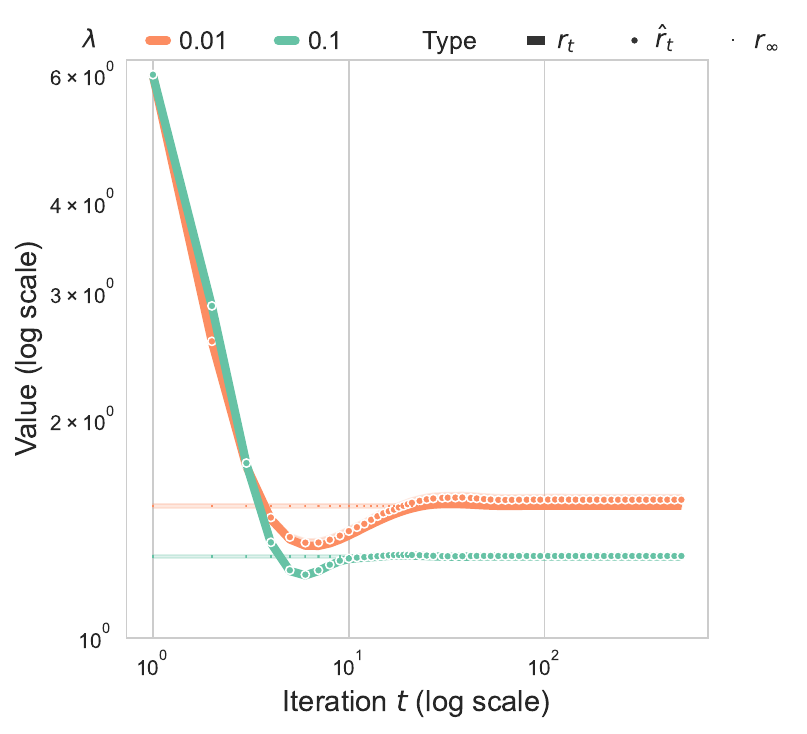}
      \caption{Risk curves versus iteration number.}
      \label{fig:Lasso-FISTA-risk-1200}
    \end{subfigure}
    \hfill
    \begin{subfigure}{0.49\linewidth}
      \centering
      \includegraphics[width=\linewidth]{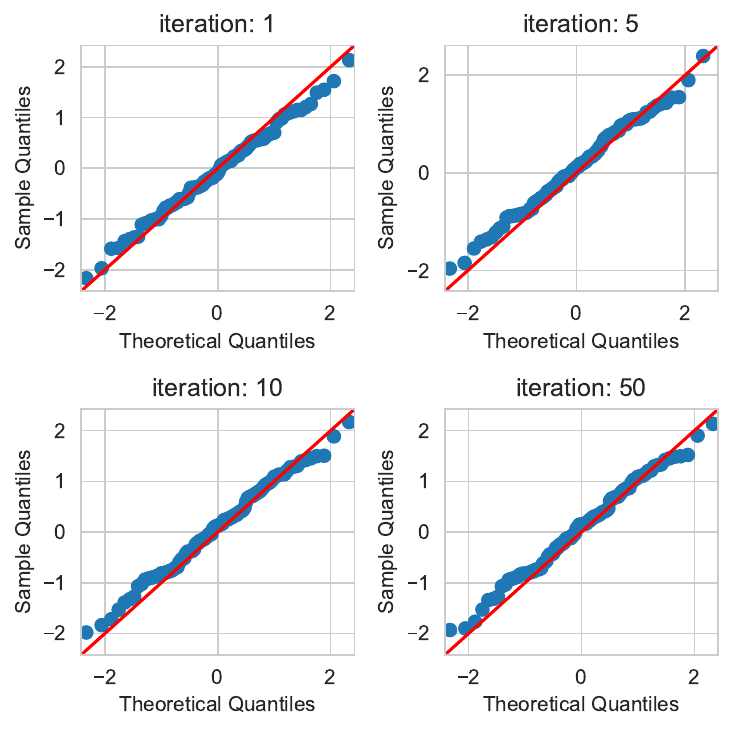}
      \caption{Q-Q plots of z-score for $\lambda = 0.1$.}
      \label{fig:Lasso-FISTA-zscore-1200}
    \end{subfigure}
    \caption{Risk curves and qq-plots of z-score of FISTA for $(n,p)=(1200, 1200)$.}
    \label{fig:Lasso-FISTA-1200}
  \end{figure}

\subsubsection{LQA}
We present the simulation results for LQA in \Cref{fig:MCP-LQA-1200}.
\begin{figure}
    \centering
    \begin{subfigure}{0.5\linewidth}
      \centering
      \includegraphics[width=\linewidth]{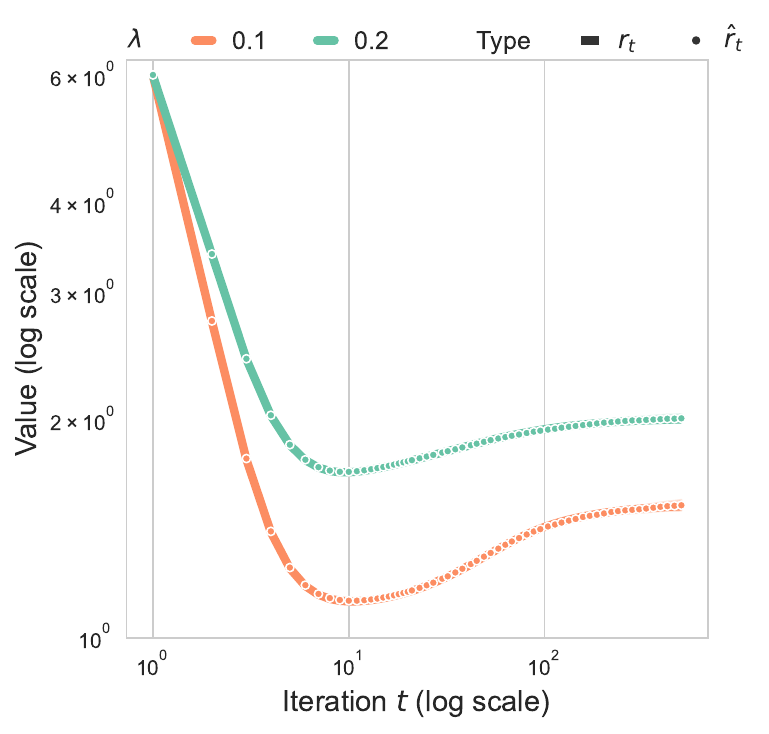}
      \caption{Risk curves versus iteration number.}
      \label{fig:MCP-LQA-risk-1200}
    \end{subfigure}
    \hfill
    \begin{subfigure}{0.49\linewidth}
      \centering
      \includegraphics[width=\linewidth]{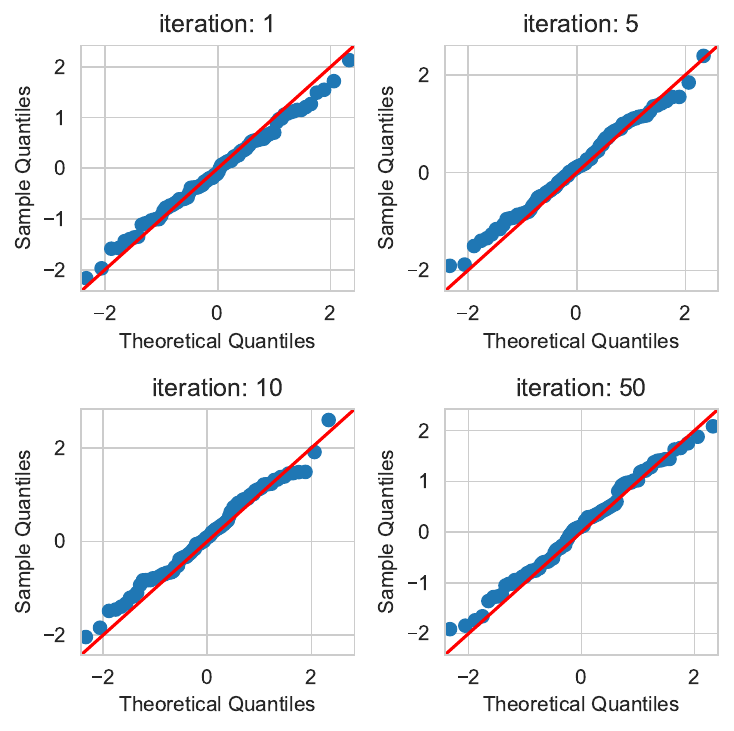}
      \caption{Q-Q plots of z-score for $\lambda = 0.1$.}
      \label{fig:MCP-LQA-zscore-1200}
    \end{subfigure}
    \caption{Risk curves and qq-plots of z-score of LQA for $(n,p)=(1200, 1200)$.}
    \label{fig:MCP-LQA-1200}
  \end{figure}

\section{Preliminary}
\subsection{Notation}
Notation and definitions that will be used in the rest of the paper are given here. 
Regular variables like $a, b, ...$ refer to scalars, bold lowercase letters such as $\ba, \bb, ...$ represent vectors, and bold uppercase letters like $\bA, \bB, ...$ indicate matrices.
Let $[n] = \{1, 2,\ldots, n \}$ for all $n\in\N$. 
The vectors $\be_i\in \R^n,\be_j\in\R^p, \be_t\in\R^T$ denote the canonical basis vector of the corresponding index. 
For a real vector $\ba \in \R^p$,  
$\norm*{\ba}$ 
denotes its Euclidean norm. 
For any matrix $\bA$, $\bA^\dagger$ is its Moore–Penrose inverse;  
$\fnorm*{\bA}$ 
,$\opnorm*{\bA}$, 
$\norm*{\bA}_*$ 
denote its Frobenius, operator and nuclear norm, respectively.  
Let $\bA\otimes \bB$ be the Kronecker product of $\bA$ and $\bB$.
For $\bA$ symmetric, $\phi_{\min}(\bA)$ and $\phi_{\max}(\bA)$ denote its smallest and largest eigenvalues, respectively. 
Let ${\bf 1}_n$ denote the all-ones vector in $\R^n$, and $\bI_n$ denote the identity matrix of size $n$. 
For a mapping $\R^{p} \to \R^{n}: \bx \mapsto \bff(\bx)$, 
we denote its Jacobian by 
$\pdv{\bff(\bx)}{\bx}\in \R^{n\times p}$, 
\ie, $(\pdv{\bff(\bx)}{\bx})_{i,j} \defas \pdv{\bff_i(\bx)}{\bx_j}$ for all $i\in[n], j\in[p]$.
For a random sequence $\xi_n$, we write $\xi_n = O_P(a_n)$ if $\xi_n/a_n$ is stochastically bounded. 
For two scalars $a,b$, $a\vee b$ denotes the maximum of $a$ and $b$.

Let $\I(\Omega)$ denote the indicator function of event $\Omega$. It takes the value 1 if the event $\Omega$ occurs and 0 otherwise. 
Let $C$ represent an absolute constant. Additionally, we use $C(\zeta, \gamma)$ to denote a positive constant that only depends on $\tau$ and $\gamma$. Similarly, we extend this notation to $C(\zeta, \gamma, \kappa, \ldots)$, representing positive constants dependent only on $\tau$, $\gamma$, $\kappa$, and so forth. The exact value of these constants may vary from place to place.
We write $a\lesssim b$ if $a\le C b$ for some absolute constant $C$.

Let $\mathsf{N}(\mu, \sigma^2)$ denote the univariate Gaussian distribution with mean $\mu$ and variance $\sigma^2$, and $\mathsf{N}_k(\bmu, \bSigma)$ denote the k-dimensional Gaussian distribution with mean $\bmu$ and covariance matrix $\bSigma$.

In this paper, we adopt the convention use of the expectation and conditional expectation. 
For a random variable $X$, the expressions $\E X$, $\E (X)$, and $\E[X]$ all refer to the expectation of $X$. Similarly, $\E X^2$, $\E (X^2)$, and $\E [X^2]$ mean that we first square the random variable $X$ and then compute its expectation.
In contrast, expressions like $\E(X)^{2}$ and $\E[X]^{2}$ denote a different operation; here, we first calculate the expectation of $X$ and then square the result.
This convention extends consistently to other power operations and conditional expectations throughout the paper.

\subsection{Review of Kronecker product}

In this section, we review the definitions and properties of the Kronecker product.
Let $\bA$ be an $m \times n$ matrix, and $\bB$ be a $p \times q$ matrix. The Kronecker product of $\bA$ and $\bB$, denoted by $\bA \otimes \bB$, is defined as:
\begin{equation*}
    \bA \otimes \bB = \begin{bmatrix}
        a_{11}\bB & \cdots & a_{1n}\bB \\
        \vdots & \ddots & \vdots \\
        a_{m1}\bB & \cdots & a_{mn}\bB
    \end{bmatrix},
\end{equation*}
where $a_{ij}$ represents the entry in the $i$-th row and $j$-th column of matrix $\bA$.
Below we list a few properties of the Kronecker product that will be useful in our proofs.
\begin{enumerate}
    \item The mixed-product property: If $\bA, \bB, \bC$ and $\bD$ are matrices of such size that the matrix products $\bA\bC$ and $\bB\bD$ make sense, then
    \begin{equation}\label{eq:kronecker-mix-product}
        (\bA \otimes \bB)(\bC \otimes \bD) = (\bA\bC) \otimes (\bB\bD)\in \R^{mp \times nq}.
    \end{equation}
    \item Inverse property: 
    \begin{align}\label{eq:kronecker-inverse}
        (\bA \otimes \bB)^{-1} = \bA^{-1} \otimes \bB^{-1} \text{ and }
        (\bA \otimes \bB)^{\dagger} = \bA^{\dagger} \otimes \bB^{\dagger},
    \end{align}
    where $\bA^{\dagger}$ means the Moore-Penrose inverse of $\bA$.
    \item Trace property: 
    \begin{align}\label{eq:kronecker-trace}
        \trace(\bA \otimes \bB) = \trace(\bA) \trace(\bB).
    \end{align}
    \item Norm property: 
    \begin{align}\label{eq:kronecker-norm}
        \opnorm{\bA \otimes \bB} = \opnorm{\bA} \opnorm{\bB} \text{ and }
        \fnorm{\bA \otimes \bB} = \fnorm{\bA} \fnorm{\bB}.
    \end{align}
    \item Relationship with vectorization operator: If the matrix product $\bA\bB\bC$ makes sense, then 
    \begin{align}\label{eq:kronecker-vectorization}
        \vec(\bA \bB \bC) = (\bC^\top \otimes \bA) \vec(\bB),
    \end{align}
    where $\vec(\cdot)$ is the vectorization operator such that $\vec(\bB)$ is the vector obtained by stacking vertically the columns of $\bB$ on top of one another.
\end{enumerate}

\section{Derivative formula}
In this section, we provide the derivative formulas for the iteration function $\bg_t$ in \eqref{eq:general-form}. 
These derivative formulas will be important ingredient in the proof of the main results in the next sections.
Recall $\hbb^1, \hbb^2, \ldots, \hbb^T$ represent the first $T$ iterates of an algorithm using the iteration function $\bg_t$
in \eqref{eq:iterates_all_previous},
and $\bv^t = \tfrac{1}{n} \bX^\top(\by - \bX\hbb^t)$ as in \eqref{v_t}.
The error matrix and residual matrix for the first $T$ iterates $\hbb^1, \hbb^2, \ldots, \hbb^T$ are defined as
\begin{equation}\label{eq:F-H}
    \bH = [\hbb^1 - \bb^*, \ldots, \hbb^T - \bb^*] \in \R^{p \times T}, \quad 
    \bF = [\by - \bX \hbb^1, \ldots, \by - \bX \hbb^T] \in \R^{n \times T}.
\end{equation}
We also define $\hbB = [\hbb^1, \ldots, \hbb^T]$ is a matrix of size $p \times T$, 
$\bB^* = [\bb^*, \ldots, \bb^*]$, 
$\bE = [\bep, ..., \bep]$
and $\bY = [\by, \ldots, \by]$ are matrices formed by repeating the vectors $\bb^*$, $\bep$, and $\by$ column-wise $T$ times.
Therefore, we have $\bH = \hbB - \bB^*$, and $\bF = \bY - \bX \hbB$. 
Recall also the matrices $\bD_{t,s},\bJ_{t,s}$ in \eqref{J_ts_D_ts_all_previous}
and the large matrices $\mathcal D,\mathcal J\in\R^{pT\times pT}$ in \eqref{eq:all_previous_calD_calJ}.
\begin{lemma}[Derivative of iterates]\label{lem:dot-b-4}
    For $(\hbb^1, ..., \hbb^T)$ in \eqref{eq:iterates_all_previous},
    for almost every $(\bX,\bep)$,
    \begin{align}
        \pdv{\hbb^t}{x_{ij}}
        &= n^{-1} (\be_t^\top \otimes \bI_p) \calM^{-1}\calD
        \Bigl[\Bigl((\bF^\top \be_i) \otimes \be_j\Bigr) - \Bigl((\bH^\top \be_j)  \otimes (\bX^\top\be_i)\Bigr)
        \Bigr],\label{eq:pdv-bt-x}\\
        \pdv{\hbb^t}{\eps_l} &= n^{-1}
        (\be_t^\top \otimes \bI_p) \calM^{-1}\calD (\bd1_T \otimes \bX^\top\be_l), \label{eq:pdv-bt-ep}
    \end{align}
    where $\be_i,\be_l\in\R^n,\be_j\in\R^p,\be_t\in\R^T$ and
    $\calM = \bI_{pT} + \calD (\bI_T \otimes \tfrac{\bX^\top \bX}{n}) - \calJ$. 
\end{lemma}
The derivative results in \Cref{lem:dot-b-4} directly imply the following
by the product rule.

\begin{corollary}[Derivative of residuals]\label{lem:dot-F-4}
    Let $F_{lt}$ be the $(l,t)$-th entry of $\bF$, \ie, $F_{lt} = \be_l^\top \bF \be_t$. 
    Under the conditions of \Cref{lem:dot-b-4}, for each $i,l\in[n]$, $j\in[p]$, $t\in[T]$,
    we have
    \begin{equation}
        \label{eq:D2_decompoistion_Dij_Delta_ij}
        \frac{\partial F_{lt}}{\partial x_{ij}} = D_{ij}^{lt} + \Delta_{ij}^{lt} 
        \quad {and} \quad
        \pdv{\bF \be_t}{\eps_i }
        = \be_i - n^{-1} (\be_t^\top \otimes \bX) \calM^{-1}\calD
        (\bd1_T \otimes \bX^\top \be_i),
    \end{equation}
    where the expressions of $D_{ij}^{lt}$ and  $\Delta_{ij}^{lt}$ are given by
    \begin{align}
    D_{ij}^{lt} 
    &= -(\be_t^\top\otimes \be_l^\top) 
    [\bI_{nT} - n^{-1}(\bI_T \otimes \bX) \calM^{-1} \calD (\bI_T \otimes \bX^\top )] ((\bH^\top\be_j)   \otimes \be_i), \label{eq:D-ij-lt}
    \\
    \Delta_{ij}^{lt} 
    &= - n^{-1} (\be_t^\top \otimes \be_l^\top)(\bI_T\otimes \bX)
	\calM^{-1} \calD \bigl(\bF^\top \otimes \bI_p\bigr)(\be_i \otimes \be_j) \label{eq:Delta-ij-lt}.
    \end{align}
    Here $\calD$ is defined in \eqref{eq:calD-J} and $\calM = \bI_{pT} + \calD (\bI_T \otimes \tfrac{\bX^\top \bX}{n}) - \calJ$ is defined in \Cref{lem:dot-b-4}.
    It immediately follows that
    \begin{align}\label{eq:sum-D}
        \sum_{i=1}^n D_{ij}^{it} 
        = -\be_t^\top (n \bI_T - \hbA) \bH^\top \be_j 
        \quad\text{and}\quad
        \sum_{i=1}^n \sum_{t=1}^T D_{ij}^{it} \be_t^\top
        = - \be_j^\top \bH(n \bI_T - \hbA^\top),
        \\
            \label{eq:lem:div-F-eps}
            \trace\Bigl(\pdv{\bX \hbb^t}{\bep}\Bigr)=
            \sum_{l=1}^n\pdv{\be_l^\top\bX \hbb^t}{\eps_l}
            =
            \be_t^\top \hbA {\bf1}_T
        \quad \text{ and } \quad
            \trace\Bigl(\pdv{\bF \be_t}{\bep}\Bigr) = \be_t^\top (n\bI_T - \hbA) {\bf1}_T.
    \end{align}
\end{corollary}

\begin{proof}[Proof of \Cref{lem:dot-b-4}]
By composition of locally Lipschitz functions, the map
$(\bX,\bep)\mapsto \hbb^t$ is locally Lipschitz and thus differentiable
almost everywhere. Let $(\dot\bX,\dot\bep)$ be a perturbation direction.
We now compute the directional derivative of $\hbb^t$ with respect to
this direction by taking, for any function $F(\bep,\bX)$ the limit
$u^{-1}(F(\bep+u\dot\bep, \bX+u\dot\bX)-F(\bep,\bX))$ as $u\to 0$
and call the limit $\dot F$. For $\hbb^t$, we denote the corresponding
directional derivative by $\dot\bb^t$, and for $\bv^t$ by $\dot\bv^t$.
By definition of $\hbb^t$ in \eqref{eq:iterates_all_previous},
we have almost surely by the chain rule
(if necessary with modification from \cite[Corollary 3.2]{ambrosio1990general} as explained after
\eqref{chainrule_lipschitz_example}),
\begin{align}\label{eq:dot-b-1}
    \dot\bb^t
    =
    \sum_{s=1}^{t-1}
    \bJ_{t, s} \dot\bb^s + \bD_{t, s} \dot\bv^s.
\end{align}
Since $\bv^ t= n^{-1}\bX^\top\bF\be_t=n^{-1}\bX^\top(\bep - \bX\bH\be_t)$ 
and $\dot\bH\be_t = \dot\bb^t$, we have by the product rule
\begin{equation}
    \dot \bv^t
    = -\tfrac{1}{n} \bX^\top \bX \dot \bb^t + 
    \underbrace{\tfrac{1}{n} \big[\dot\bX^\top \bF \be_t + \bX^\top (\dot\bep - \dot\bX \bH \be_t) \big]}_{\ba^t}
    \label{def_a_t}
\end{equation}
where $\ba^t\in\R^p$.
Substituting the expressions for $\dot\bv^{s}$ into \eqref{eq:dot-b-1}, we have
\begin{align}
    \label{eq:linear-system}
    \dot\bb^{t} 
    = 
    \sum_{s=1}^{t-1}
    (\bJ_{t, s} - \bD_{t, s} \tfrac{\bX^\top \bX}{n}) \dot\bb^{s} 
    + \bD_{t, s} \ba^{s}. 
\end{align}
For the initial condition, we have $\dot \bb^1 = \b0$ since $\hbb^1$ is a constant independent of $\bX$. 
Similarly $\bD_{2,0} = \bJ_{2,0} =\b0$ since $\hbb^0$ and $\bv^0$ are set as constant vector $\b0$. 
Therefore
we can write \eqref{eq:linear-system}
as a linear system of size $pT$:
\begin{align*}
    \renewcommand{\arraystretch}{2.5}
    &
    \underbrace{
    \begin{bmatrix}
        \bI_p &\empty &\empty &\empty &\empty\\
        \bD_{2,1} \tfrac{\bX^\top \bX}{n} - \bJ_{2,1} & \bI_p &\empty &\empty&\empty\\
        \bD_{3,1} \tfrac{\bX^\top \bX}{n} - \bJ_{3,1} & \bD_{3,2} \tfrac{\bX^\top \bX}{n}- \bJ_{3,2} & \bI_p &\empty&\empty\\
        \vdots & \ddots & \ddots & \ddots &\empty\\
        \bD_{T,1}\tfrac{\bX^\top \bX}{n} - \bJ_{T, 1} & \cdots & \bD_{T,T-2} \tfrac{\bX^\top \bX}{n} - \bJ_{T, T-2} & \bD_{T, T-1} \tfrac{\bX^\top \bX}{n} -\bJ_{T, T-1}  &\bI_p
    \end{bmatrix}}_{\calM}
    \underbrace{
    \begin{bmatrix}
        \dot \bb^1\\
        \dot \bb^{2}\\
        \dot \bb^{3}\\
        \vdots\\
        \dot \bb^T\\
    \end{bmatrix}}_{\vec(\dot \bB)} \\
    &=
    \underbrace{\begin{bmatrix}
        \b0 & \b0 & \b0 & \b0 & \b0\\
        \bD_{2,1} & \b0 & \b0 & \b0 & \b0\\
        \bD_{3,1} & \bD_{3,2} & \b0 & \b0 & \b0\\
        \vdots & \ddots & \ddots & \ddots &\empty\\
        \bD_{T,1} & \cdots & \bD_{T, T-2} & \bD_{T, T-1} &\b0\\
    \end{bmatrix}}_{\calD}
    \underbrace{
    \begin{bmatrix}
        \ba^1\\
        \ba^2\\
        \ba^3\\
        \vdots\\
        \ba^T\\
    \end{bmatrix}}_{\ba}.\notag
\end{align*}
In the above equation, 
the matrix $\calM$ is the same as the one defined in \Cref{lem:dot-b-4},
\ie, $\calM = \bI_{pT} + \calD (\bI_T \otimes \tfrac{\bX^\top \bX}{n}) - \calJ$ 
for $\calD,\calJ$ in \eqref{eq:all_previous_calD_calJ}.
Solving the linear system \eqref{eq:linear-system} for $\dot \bb^t$, 
we obtain
    \begin{align*}
        \dot \bb^t
        = \vec(\dot \bB \be_t)
        =& (\be_t^\top \otimes \bI_p) \vec(\dot \bB) &&\mbox{by \eqref{eq:kronecker-vectorization}}\\
        =& (\be_t^\top \otimes \bI_p) \calM^{-1}\calD \ba &&\mbox{by \eqref{eq:linear-system}}
        .
    \end{align*}
    By definition of $\ba\in\R^{pT}$ and $\ba^t$ in \eqref{def_a_t},
    this completes the proof of 
    \eqref{eq:pdv-bt-x} by taking
    $(\dot\bep,\dot\bX) = (\b0_n,\be_i\be_j^\top)$ and of
    \eqref{eq:pdv-bt-ep} by taking $(\dot\bep,\dot\bX) = (\be_l,\b0_{n\times p})$.
\end{proof}

\section{Change of variables}
\label{sec:change-of-var}
For the linear model 
$\by = \bX\bb^* + \bep$, its design matrix $\bX$ may not be isotropic.
However,
we can always consider the following change of variables:

\begin{align}\label{eq:change-of-var}
    \bX  \rightsquigarrow	 \bX\bSigma^{-1/2}:=\bG,
    \quad 
    \bb^*  \rightsquigarrow	 \bSigma^{1/2}\bb^*:= \btheta^*.
\end{align}
This way, the original linear model can be rewritten as 
\begin{align}\label{eq:linear-model-G}
    \by = \bG\btheta^* + \bep,
\end{align}
where $\bG$ is the design matrix with \iid entries from $\mathsf{N}(0, 1)$, and $\btheta^*$ is the new unknown coefficient vector.
For the linear model \eqref{eq:linear-model-G} with isotropic design matrix $\bG$, we use $\hbtheta^1, ..., \hbtheta^T$ to denote the first $T$ iterates of an iterative algorithm detailed in next paragraph, where $\hbtheta^t$ are constructed such that $\hbtheta^t = \bSigma^{1/2}\hbb^t$.  
Similar to the definition of $\bF, \bH$ in \eqref{eq:F-H}, for the linear model \eqref{eq:linear-model-G}, we define
\begin{align*}
    \bF^* = [\by - \bG \hbtheta^1, ..., \by - \bG \hbtheta^T], \quad \bH^* = [\hbtheta^1 - \btheta^*, ..., 
    \hbtheta^T - \btheta^*].
\end{align*}
Denoting with a superscript$^*$ the quantities after the change of variable,
we have
$\bF^* = \bF$ and $\bH^* = \bSigma^{1/2}\bH$. 

We now describe the iteration to generate the iterates $\hbtheta^1, ..., \hbtheta^T$ that guarantees $\hbtheta^t = \bSigma^{1/2}\hbb^t$. 
We start with $\hbtheta^1 = \bSigma^{1/2}\hbb^1$, and generate $\hbtheta^t$ for $t\ge 2$ by the recursion:
$\hbtheta^t = \tbg_t(\hbtheta^{t-1}, \tbv^{t-1}, \hbtheta^{t-2}, \tbv^{t-2})$, where 
$\tbv^t = \frac{1}{n} \bG^\top(\by - \bG \hbtheta^t)$, and 
the iteration function $\tbg_t$ is defined as
\begin{equation}
\label{eq:tbg}
\begin{aligned}
    &\tbg_t\bigl( 
        \hbtheta^{t-1},\dots,\hbtheta^{1},~~
        \tbv^{t-1},\dots \tbv^{1}
\bigr) 
    \\= \bSigma^{1/2} &\bg_{t}\Bigl(
    \bSigma^{-1/2}\hbtheta^{t-1}, \dots 
    \bSigma^{-1/2}\hbtheta^{1},~~
    \bSigma^{1/2}\tbv^{t-1},
    \dots
    \bSigma^{1/2}\tbv^{1}
\Bigr).
\end{aligned}
\end{equation}
Since $\bg_t$ is $\zeta$-Lipschitz continuous from \Cref{assu:Lipschitz}, we have $\tbg_t$ is $\zeta\kappa$-Lipschitz continuous using that $\opnorm{\bSigma}\le \opnorm{\bSigma}\opnorm{\bSigma^{-1}}\le \kappa$ from \Cref{assu:design}. 
By construction, $\hbtheta^s=\bSigma^{1/2}\hbb^s$ for all $s\le t-1$
implies $\hbtheta^t=\bSigma^{1/2}\hbb^t$, so by induction
the relation holds for all $t\ge 1$.
For the iteration function $\tbg_t$ in \eqref{eq:tbg},
similarly to \eqref{J_ts_D_ts_all_previous},
we define the derivative matrices with respect to each argument as
\begin{equation}
    \bJ_{t, s}^*= \pdv{\tbg_t}{\btheta^{s}}\Bigl(\hbtheta^{t-1}, \dots,\hbtheta^{1},
        \tbv^{t-1}, \dots \tbv^{1}
    \Bigr),
    \qquad 
    \bD_{t, s}^* =
    \pdv{\tbg_t}{\tbv^{s}}\Bigl(\hbtheta^{t-1}, \dots,\hbtheta^{1},
        \tbv^{t-1}, \dots \tbv^{1}
    \Bigr)
    \label{eq:relation_jacobians_star_without}
    \end{equation}
    for each $s\le t-1$.
Furthermore, we have by chain rule, 
\begin{equation}\label{eq:J-D-star}
\begin{aligned}
    \bJ_{t, s}^* 
    &= \pdv{\tbg_{t}}{\hbtheta^{s}}
    = \bSigma^{1/2} \pdv{\bg_{t}}{\hbb^{s}} \bSigma^{-1/2} 
    = \bSigma^{1/2} \bJ_{t, s} \bSigma^{-1/2},\\
    \bD_{t, t-1}^* 
    &= \pdv{\tbg_{t}}{\tbv^{s}}
    = \bSigma^{1/2} \pdv{\bg_{t}}{\bv^{s}} \bSigma^{1/2} = \bSigma^{1/2} \bD_{t, s} \bSigma^{1/2}.
\end{aligned}
\end{equation}
We may apply \Cref{lem:dot-b-4} to
$\tbg_t$, $\btheta_t$, $\bG$ to obtain
with $\mathcal M^*=\bI_{pT} + \mathcal D^*
(\bI_T \otimes \bG^\top\bG/n)-\mathcal J^*$
\begin{align*}
    \pdv{\hbtheta^t}{g_{ik}} 
    &= n^{-1} (\be_t \otimes \bI_p) (\calM^*)^{-1} \calD^* 
    \sum_{j=1}^p
    \Bigl[\Bigl((\bF^*)^\top \be_i\Bigr) \otimes \be_k - \Bigl((\bH^*)^\top \be_k\Bigr)  \otimes (\bG^\top\be_i)\Bigr]
\end{align*}
where 
    $\calD^* = \sum_{t=2}^T 
    \sum_{s=1}^{t-1}
(\be_t \be_s^\top) \otimes
\bD_{t,s}^*$ and 
$\calJ^* = \sum_{t=2}^T 
\sum_{s=1}^{t-1}
(\be_t \be_s^\top) \otimes
\bJ_{t,s}^*$
as in \eqref{eq:all_previous_calD_calJ}.
Note that the relationships 
\begin{equation}
\calM^* = (\bI_T \otimes \bSigma^{1/2}) \calM (\bI_T \otimes \bSigma^{-1/2}),
\qquad \calD^* = (\bI_T \otimes \bSigma^{1/2}) \calD (\bI_T \otimes \bSigma^{1/2})
\label{relation_M_M^*_D_D^*}
\end{equation}
hold due to \eqref{eq:J-D-star}.
Now we verify that $\hbA$ is unchanged by the change of variable,
that is, $\hbA=\hbA^*$. 
By definition of $\hbA^*$ (as in \eqref{eq:hat-A}),
\begin{align*}
    \hbA^*
    \stackrel{\text{\tiny def}}{=} 
    \frac1n \sum_{i=1}^n (\bI_T \otimes \be_i^\top \bG)(\calM^*)^{-1} \calD^* (\bI_T \otimes \bG^\top \be_i)
    = \frac1n \sum_{i=1}^n (\bI_T \otimes (\be_i^\top \bX))\calM^{-1} \calD (\bI_T \otimes (\bX^\top \be_i))
\end{align*}
using \eqref{relation_M_M^*_D_D^*} for the second equality.
The rightmost quantity is exactly $\hbA$, so
$\hbA=\hbA^*$
stays the same after the change of variable. 

In summary, we present the relevant quantities for both the original model \eqref{eq:linear-model} and the new model \eqref{eq:linear-model-G} in \Cref{table:change-of-var}.
\begin{table}[h]
    \centering
    \begin{tabular}{|c|c|c|c|}
    \hline
    Model & $\by = \bX \bb^* + \bep$ & $\by = \bG \btheta^* + \bep$ & Relationship \\ \hline
    Design matrix & $\bX$ & $\bG$ & $\bX = \bG \bSigma^{1/2}$\\ \hline
    Covariance matrix & $\bSigma$ & $\bSigma^* = \bI_p$ & $\bSigma = \bSigma \bSigma^*$\\ \hline
    Coefficient vector & $\bb^*$ & $\btheta^*$ & $\bb^* = \bSigma^{-1/2}\btheta^*$\\ \hline
    $t$-th iterate & $\hbb^t$ &$\hbtheta^t$
    & $\hbb^t = \bSigma^{-1/2}\hbtheta^t$\\ \hline
    Error matrix & $\bH = \sum_{t} (\hbb^t - \bb^*)\be_t^\top$ 
    & $\bH^* = \sum_{t} (\hbtheta^t - \btheta^*) \be_t^\top $ & $\bH = \bSigma^{-1/2}\bH^*$\\ \hline
    Residual matrix & $\bF = \sum_{t}(\by - \bX \hbb^t)\be_t^\top$ & $\bF^* = \sum_{t}(\by - \bG \hbtheta^t)\be_t^\top$ & $\bF = \bF^*$\\ \hline
    Memory matrix & $\hbA$ & $\hbA^*$ &$\hbA = \hbA^*$\\ \hline
    \end{tabular}
    \caption{Change of variables.}
    \label{table:change-of-var}
\end{table}
The benefit of this change of variable is that, the transformed model \eqref{eq:linear-model-G} has isotropic design matrix $\bG$, and the following quantities stay the same:
$$\bF = \bF^*, \quad 
\hbA = \hbA^*, \quad
\bSigma^{1/2}\bH = \bH^*.
$$
Therefore, \Cref{thm:rt,thm:generalization-error} hold for general $\bSigma$ if they hold for $\bSigma = \bI_p$. 
The only change is the Lipschitz constant of 
the function $\bg_t$ changes to the Lipschitz constant of $\tbg_t$, with an extra factor of $\kappa$ by \Cref{assu:design} as we argued below \eqref{eq:tbg}.
Throughout the subsequent proofs of \Cref{thm:rt,thm:generalization-error}, we will assume $\bSigma = \bI_p$ without loss of generality. The proof for general $\bSigma$ follows by changing the Lipschitz constant from $\zeta$ to $\kappa\zeta$.

\section{Useful intermediate results}
\subsection{Deterministic operator norm bound}
\begin{lemma}\label{lem:opnorm-D-J}
    Under \Cref{assu:Lipschitz},
    for $\calJ,\calD$ in \eqref{eq:all_previous_calD_calJ} we have
    \begin{equation}\label{eq:opnorm-D-J}
        \max\{\opnorm{\calD}, \opnorm{\calJ} \} \le T\zeta.
    \end{equation}
\end{lemma}
\begin{proof}[Proof of \Cref{lem:opnorm-D-J}]
    By \Cref{assu:Lipschitz} the function $\bg_t$ is $\zeta$-Lipschitz, thus
    the block row of size $p \times pT$ corresponding to $\bg_t$ in $\calD$ and $\calJ$ is bounded in operator norm by $\zeta$. Since it has $T$ such row-blocks of size $p\times pT$, we obtain
    the result by the triangle inequality.
\end{proof}

\begin{lemma}\label{lem:opnorm-M-N}
    With $\calM$ as in \Cref{lem:dot-b-4} and 
    $\bN = n^{-1} (\bI_T \otimes \bX) \calM^{-1} \calD (\bI_T \otimes \bX^\top)$, 
    \begin{equation}\label{eq:opnorm-M-N}
        \|\calM^{-1}\|_{\rm op} 
        \le C(\zeta, T) (1 + \opnorm{\bX}^2/n)^{T-1}
        \text{ and }
        \|\bN\|_{\rm op} 
        \le C(\zeta, T) (1 + \opnorm{\bX}^2/n)^{T}
    \end{equation}
    where $C(\zeta, T)$ is a constant depending on $\zeta$ and $T$ only. 
\end{lemma}

\begin{proof}[Proof of \Cref{lem:opnorm-M-N}]
By definition, 
$\calM = \bI_{pT} - [\calJ - \calD (\bI_T \otimes \tfrac{\bX^\top \bX}{n})] $
and $(\bI_{pT}-\mathcal M)^T=\b0_{pT\times pT}$ since
$\bI_{pT}-\mathcal M$ is lower triangular with zero diagonal blocks of size
$p\times p$.
Using the matrix identity $(\bI - \bA)^{-1} = \sum_{k=0}^{\infty} \bA^k=\sum_{k=1}^{T-1}$ for any matrix $\bA$ with $\bA^T=0$, we have 
$\calM^{-1} 
    = \sum_{k=0}^{T-1} [\calJ - \calD (\bI_T \otimes \tfrac{\bX^\top \bX}{n})]^k
$.
Furthermore for each $k=0,...,T_1$,
\begin{align*}
    \|\calJ - \calD (\bI_T \otimes \tfrac{\bX^\top \bX}{n}) \|^{k}_{\rm op}
    &\le (\opnorm{\calJ} + \|\calD\|_{\rm op} \opnorm{\bX}^2/n )^{k} &&\mbox{by \eqref{eq:kronecker-norm}}\\
    &\le [(1 + \opnorm{\bX}^2/n) \zeta T]^k &&\mbox{by \eqref{eq:opnorm-D-J}}.
\end{align*}
This provides $\|\mathcal M\|_{\rm op}\le T (1 + \zeta T)^{T-1} (1 + \opnorm{\bX}^2/n)^{T-1}$.
The bound on $\opnorm{\bN}$ follows from
$
    \|\bN\|_{\rm op} 
    \le \|\calM^{-1}\|_{\rm op} \|\calD\|_{\rm op} \opnorm{\bX}^2/n
$
and the previous bound.
\end{proof}

\begin{lemma}\label{lem:opnorm-Ahat}
    Under \Cref{assu:Lipschitz}, we have 
    \begin{align}
        \|\bI_T - \hbA/n\|_{\rm op} 
        &\le C(\zeta, T) (1 + \opnorm{\bX}^2/n)^{T} 
        ,
        \label{eq:I-hbA/n_opnorm}
        \\
        \|(\bI_T - \hbA/n)^{-1}\|_{\rm op} 
        &\le C(\zeta, T) (1 + \opnorm{\bX}^2/n)^{T^2}.
        \end{align}
\end{lemma}

\begin{proof}[Proof of \Cref{lem:opnorm-Ahat}]
    \label{proof:opnorm-Ahat}
    Since $\hbA = \sum_{i=1}^n (\bI_T\otimes \be_i^\top)\bN(\bI_T\otimes \be_i)$
    we have for any $\bu,\bu\in \R^T$ that
    $|\bu^\top\hbA\bv| =|\sum_{i=1}^n (\bu\otimes \be_i)^\top \bN (\bv\otimes \be_i)| \le \|\bN\|_{\rm op} \sum_{i=1}^n 
    \|\bu\otimes \be_i\|
    \|\bv\otimes \be_i\|
    \le n\|\bN\|_{\rm op}$ using 
    \eqref{eq:kronecker-mix-product}
    and \eqref{eq:kronecker-norm}.
    The bound on the operator norm of $\bN$ from \Cref{lem:opnorm-M-N}
    gives \eqref{eq:I-hbA/n_opnorm}.

    Since $\hbA$ is a lower triangular matrix, we know $\hbA$ is a nilpotent matrix with $\hbA^k = 0$ if $k\ge T$. 
    Consequently
    $(\bI_T - \hbA/n)^{-1} 
    = \sum_{k=0}^{\infty} (\hbA/n)^k
    = \sum_{k=0}^{T-1} (\hbA/n)^k$, 
    hence 
    \begin{align*}
        \|(\bI_T - \frac{\hbA}{n})^{-1}\|_{\rm op} 
    &\le \sum_{k=0}^{T-1} \|\frac{\hbA}{n}\|_{\rm op}^k
    \le C(\zeta, T) \sum_{k=0}^{T-1} \Bigl(1 + \frac{\opnorm{\bX}^2}{n}\Bigr)^{Tk}
    \le C(\zeta, T) (1 + \opnorm{\bX}^2/n)^{T^2}.
    \end{align*}
    This finishes the proof.
    \end{proof}

\begin{lemma}\label{lem:opnorm-Chat}
    Under \Cref{assu:Lipschitz}, 
    let $\hbC = \sum_{j=1}^p (\bI_T \otimes \be_j^\top) \calM^{-1} \calD (\bI_T \otimes \be_j)$, we have 
    \begin{align}
        \|\bI_T + \hbC/n\|_{\rm op} 
        &\le C(\zeta, T, \gamma) (1 + \opnorm{\bX}^2/n)^{T} 
        ,
        \label{eq:I-hbC/n_opnorm}
        \\
        \|(\bI_T + \hbC/n)^{-1}\|_{\rm op} 
        &\le C(\zeta, T, \gamma) (1 + \opnorm{\bX}^2/n)^{T^2}. \label{eq:I-hbC/n_inv_opnorm}
        \end{align}
\end{lemma}
\begin{proof}[Proof of \Cref{lem:opnorm-Chat}]
    By triangular inequality, we have 
    \begin{align*}
        \opnorm{\hbC}/n
        &\le n^{-1} \sum_{j=1}^p \opnorm{(\bI_T \otimes \be_j^\top) \calM^{-1} \calD (\bI_T \otimes \be_j)}\\
        &\le p/n \opnorm{\calM^{-1}} \opnorm{\calD} \\
        &\le C(\zeta, T, \gamma) (1 + \opnorm{\bX}^2/n)^{T-1} 
        &&\mbox{by \Cref{lem:opnorm-M-N,lem:opnorm-D-J}}.
    \end{align*}
    This directly implies \eqref{eq:I-hbC/n_opnorm} by the triangle inequality.
    Since $\hbC$ is lower triangular by its definition, the same argument in the proof of \Cref{lem:opnorm-Ahat} gives \eqref{eq:I-hbC/n_inv_opnorm}. 
\end{proof}

\begin{lemma}\label{lem:opnorm-dF/de}
    Under \Cref{assu:Lipschitz}, we have 
    \begin{align}
        \label{eq:opnorm-dF/de}
        \big\|\pdv{\bF \be_t}{\bep}\big\|_{\rm op}
        =
        \big\|\bI_n - (\be_t\otimes \bI_n)^\top\bN(\bd1_T\otimes \bI_n)\|_{\rm op}
        &\le C(\zeta, T) (1 + \opnorm{\bX}^2/n)^{T},
        \\
        \big\|\pdv{\bH \be_t}{\bep}\big\|_{\rm op}
        =
        \big\|\frac1n (\be_t\otimes \bI_p)^\top\calM^{-1}\calD(\bd1_T\otimes \bX^\top)
        \big\|_{\rm op}
        &\le \frac{C(\zeta, T)}{\sqrt n} (1 + \opnorm{\bX}^2/n)^{T}.
        \label{eq:opnorm-dH/de}
    \end{align}
\end{lemma}
\begin{proof}[Proof of \Cref{lem:opnorm-dF/de}]
    By \Cref{lem:dot-F-4}, we have 
    $$\pdv{\bF \be_t}{\bep} 
    = \bI_n - n^{-1} (\be_t^\top \otimes \bX) \calM^{-1}\calD
    (\bd1_T \otimes \bX^\top)
    =
    \bI_n - (\be_t\otimes \bI_n)^\top\bN (\bd1_T \otimes \bI_n)
    $$
    so the bound for $\bF$ follows from \eqref{eq:opnorm-M-N}
    and \eqref{eq:kronecker-norm}.
    The bound for $\bH$ follows from \eqref{eq:pdv-bt-ep} and the same
argument with the operator norm bounds
\eqref{eq:opnorm-D-J} and \eqref{eq:opnorm-M-N} for $\calM$.
\end{proof}

\begin{lemma}\label{lem:opnorm-dF/dZ}
    Assume \Cref{assu:Lipschitz} holds. Let $\pdv{\vec(\bF)}{\vec(\bX)}\in \R^{nT\times np}$ be the Jacobian of the mapping $\R^{np}\to \R^n: \vec(\bX) \mapsto \vec(\bF)$. Then we have 
    \begin{align}
        \big\|\pdv{\vec(\bF)}{\vec(\bX)}\big\|_{\rm op} 
        &\le C(\zeta, T) (1 + \opnorm{\bX}^2/n)^T 
        (\opnorm{\bH} + \opnorm{\bF}/\sqrt{n}),
        \label{eq:op-norm-bound-F-vec}
        \\
        \big\|\pdv{\vec(\bF)}{\vec(\bX)}\big\|_{\rm F} 
        &\le C(\zeta, T) (1 + \opnorm{\bX}^2/n)^T 
        (\fnorm{\bH\otimes \bI_n} + \fnorm{\bF\otimes \bI_p}/\sqrt{n})
        \label{eq:frob-norm-bound-F-vec}
        \\&= C(\zeta, T) (1 + \opnorm{\bX}^2/n)^T 
        (\fnorm{\bH} \sqrt n + \fnorm{\bF}\sqrt{ p/n }).
        \nonumber
    \end{align}
\end{lemma}

\begin{proof}[Proof of \Cref{lem:opnorm-dF/dZ}]
    With
    $\bF=\sum_{lt}\be_l F_{lt}\be_t^\top$
    we have
    $\vec(\bF) = \sum_{lt}(\be_t\otimes \be_l) F_{lt}$
    by 
    \eqref{eq:kronecker-vectorization} and
    so that
    \begin{align*}
        \pdv{\vec(\bF)}{\vec(\bX)} 
        &= \sum_{l=1}^n\sum_{t=1}^T \sum_{i=1}^n \sum_{j=1}^p (\be_t\otimes \be_l) \pdv{F_{lt}}{x_{ij}}  (\be_j^\top \otimes \be_i^\top)
        = \sum_{lt,ij} 
        (\be_t\otimes \be_l)
        (D_{ij}^{lt} + \Delta_{ij}^{lt})
       (\be_i^\top \otimes \be_j^\top)
    \end{align*}
    where $D_{ij}^{lt}$ and $\Delta_{ij}^{lt}$ are defined in \eqref{eq:D-ij-lt} and \eqref{eq:Delta-ij-lt}.
    Since permuting the ordering
    of canonical basis vectors
    does not change the operator norm
    and the Frobenius norm,
    we find by \eqref{eq:D-ij-lt} 
    and the definition of $\bN$
    in \Cref{lem:opnorm-M-N} that
    \begin{equation}
        \Big\|
        \pdv{\vec(\bF)}{\vec(\bX)} 
        \Big\|_?
        \le
        \Big\|
        \bN (\bH^\top \otimes \bI_n)
        \Big\|_?
        +
        \Big\|
        n^{-1}
        (\bI_T\otimes \bX)
	\calM^{-1} \calD \bigl(\bF^\top \otimes \bI_p\bigr)
        \Big\|_?
    \end{equation}
    where $\|\cdot\|_?$ is either the operator norm or the Frobenius norm.
    The property \eqref{eq:kronecker-norm} of the Kronecker product
    and the operator norm bounds \eqref{eq:opnorm-M-N}, \eqref{eq:opnorm-D-J}
    provide the two inequalities. For the last equality,
    we use $\|\bH\otimes \bI_n\|_{\rm F}=\|\bH\|_{\rm F} \sqrt n$
    and similarly for $\bF^\top\otimes \bI_p$ by \eqref{eq:kronecker-norm}.
\end{proof}

\begin{lemma}\label{lem:opnorm-dH/dZ}
    Assume \Cref{assu:Lipschitz} holds. Let $\pdv{\vec(\bH)}{\vec(\bX)}\in \R^{n\times np}$ be the Jacobian of the mapping $\R^{np}\to \R^n: \vec(\bX) \mapsto \vec(\bH)$. Then we have 
    \begin{align}
        \big\|\pdv{\vec(\bH)}{\vec(\bX)}\big\|_{\rm op} 
        &\le C(\zeta, T) (1 + \opnorm{\bX}^2/n)^T 
        \Bigl(\opnorm{\bH} + \opnorm{\bF}/\sqrt{n}\Bigr)
        \frac{1}{\sqrt n},
        \label{eq:op-norm-bound-H-vec}
        \\
        \big\|\pdv{\vec(\bH)}{\vec(\bX)}\big\|_{\rm F} 
        &\le C(\zeta, T) (1 + \opnorm{\bX}^2/n)^T 
        \Bigl(\fnorm{\bF^\top\otimes \bI_p} + \sqrt n\fnorm{\bH\otimes \bI_n}\Bigr)\frac 1 n
        \label{eq:frob-norm-bound-H-vec}
        \\&= C(\zeta, T) (1 + \opnorm{\bX}^2/n)^T 
        \bigl(\fnorm{\bF n^{-1/2}} \sqrt\gamma + \fnorm{\bH}\bigr).
        \nonumber
    \end{align}
\end{lemma}

\begin{proof}[Proof of \Cref{lem:opnorm-dH/dZ}]
    By the same argument as in
    \Cref{lem:opnorm-dF/dZ}, given \eqref{eq:pdv-bt-x},
    \begin{equation*}
        \Big\|
        \pdv{\vec(\bH)}{\vec(\bX)} 
        \Big\|_?
        \le
        \Big\|
        n^{-1}\calM^{-1}\calD (\bF^\top \otimes \bI_p)
        \Big\|_?
        +
        \Big\|
        n^{-1}\calM^{-1}\calD(\bI_T\otimes \bX^\top) (\bH^\top \otimes \bI_n)
        \Big\|_?
    \end{equation*}
    where $\|\cdot\|_?$ is either the operator norm or the Frobenius norm.
    The property \eqref{eq:kronecker-norm} of the Kronecker product
    and the operator norm bounds \eqref{eq:opnorm-M-N}, \eqref{eq:opnorm-D-J}
    provide the two inequalities. For the last equality,
    we use $\|\bF^\top\otimes \bI_p\|_{\rm F}=\|\bF\|_{\rm F} \sqrt p$
    and similarly for $\bH^\top\otimes \bI_n$ by \eqref{eq:kronecker-norm}.
\end{proof}

\subsection{Moment bounds}
\begin{lemma}[Moment bound of $\bX$]\label{lem:moment-bound-X}
    Under \Cref{assu:design,assu:regime}, 
    the following inequality holds for any positive, finite integer $k$: 
    $$\E [\|\bX\bSigma^{-1/2}/\sqrt{n}\|_{\rm op}^k] \le C(\gamma, k).$$
\end{lemma}

\begin{proof}[Proof of \Cref{lem:moment-bound-X}]
    By \cite[Theorem II.13]{DavidsonS01}, there exists a random variable
    $z\sim \mathsf{N}(0,1)$ s.t. 
    $\|\bX\bSigma^{-1/2}\|_{\rm op} \le \sqrt{n} + \sqrt{p} + z$ almost surely. 
    Thus, 
    $$\E [\|\bX\bSigma^{-1/2}/\sqrt{n}\|_{\rm op}^k] \le \E [(1 + \sqrt{p/n} + z)^k] \le C(\gamma, k).$$
    The last inequality follows from $p/n\le\gamma$ and 
    the fact that $\E [z^k]$ is bounded for any finite $k$ if $z\sim \mathsf{N}(0,1)$. 
\end{proof} 

\begin{lemma}[Moment bound of $\bep$]\label{lem:moment-bound-bep}
    Under \Cref{assu:noise}, for any positive finite integer $k$, 
    there exists a constant $C(k)$ depending only on $k$ such that
    $\E[(\norm{\bep}^{2}/n)^k] \le C(k) (\sigma^2)^k.$
\end{lemma}
    This is a known bound on the finite moment of the $\chi^2_n$
    distribution. Alternatively, since $\bep/\sigma$ has a same
    distribution as any column of $\bX\bSigma^{-1/2}$,
    we have $\E[(\|\bep\|^2/\sigma^2)^k]
    \le \E[\|\bX\bSigma^{-1/2}\|_{\rm op}^{2k}]$
    so \Cref{lem:moment-bound-bep} follows from
    \Cref{lem:moment-bound-X} with $\gamma = 1$.

\begin{lemma}[Moment bounds of $\bH$ and $\bF$]
    \label{lem:moment-bound-F-H}
        If \Cref{assu:design,assu:noise,assu:regime,assu:Lipschitz} are fulfilled
        then
        \begin{align*}
            \E [\|\bSigma^{1/2}\bH\|_{\rm F} ^2]
            \vee \E[\|\bSigma^{1/2}\bH\|_{\rm F} ^4]^{1/2} 
            \vee \E[\|\bSigma^{1/2}\bH\|_{\rm F} ^8]^{1/4}
            &\le C(\zeta, T, \gamma, \kappa) \var(y_1),
            \\
            \E [\|\bF\|_{\rm F} ^2/n] 
            \vee \E[\|\bF\|_{\rm F} ^4/n^2]^{1/2} 
            \vee \E[\|\bF\|_{\rm F} ^8/n^4]^{1/4} 
            &\le C(\zeta, T, \gamma, \kappa) \var(y_1).
        \end{align*}
        where $y_1$ is the first entry of the response vector and $\var(y_1)
        = \norm{\bSigma^{1/2}\bb^*}^2 + \sigma^2$. 
        Consequently, by compactness, we can extract a subsequence of regression problems such that 
            \begin{equation}
            \E[\bF^\top \bF/n] \to \bK, \text{ and }
        \E[\bH^\top \bSigma \bH + \bS] \to \bar \bK, 
        \label{eq:K_bar_K}
        \end{equation}
        where $\bK$ and $\bar \bK$ are two positive semi-definite deterministic matrices.
    \end{lemma}

    \begin{proof}[Proof of \Cref{lem:moment-bound-F-H}]
    Without loss of generality, we assume $\bSigma = \bI_p$. 
    Otherwise, if $\bSigma \ne \bI_p$,  we apply the change of variable \eqref{eq:change-of-var}, the desired result follows.  

    Using the fact that $\|\bH\|_{\rm F} ^2 = \sum_{t=1}^T \|\bH \be_t\|^2$ and $\|\bF\|_{\rm F} ^2 = \sum_{t=1}^T \|\bF \be_t\|^2$, it suffices to bound the moments of $\|\bH \be_t\|$ and $\|\bF \be_t\|$ for each $t\in [T]$.

    Let $a_s = \max\{\|[\hbb^s\mid \hbb^{s-1}\mid \dots \mid \hbb^1]\|_{\rm F} ,
    n^{-1/2}\|[\by-\bX\hbb^s\mid \by-\bX\hbb^{s-1}\mid \by-\bX\hbb^1]\|_{\rm F}\}$.
    By definition of $\hbb^t$ in \eqref{eq:iterates_all_previous},
    using $\bg_t(\b0)=\b0$ and \Cref{assu:Lipschitz} we have
    \begin{align*}
        \norm{\hbb^t - \b0} 
        =
        \norm{\hbb^t - \bg_t(\b0)} 
        &\le
        \zeta \|\bigl[
            \hbb^{t-1} \mid \hbb^{t-2} \mid \dots \mid \hbb^1
            \mid \bv^{t-1} \mid \dots \mid\bv^1\bigr]
            \|_{\rm F}
        \\&\le \zeta
        \|\bigl[
            \hbb^{t-1} \mid \hbb^{t-2} \mid \dots \mid \hbb^1
            \bigr]
            \|_{\rm F}
            +\zeta
        \|\bigl[
            \bv^{t-1} \mid \dots \mid \bv^1\bigr]\|_{\rm F}
        \\&\le \zeta(a_{t-1} + \|\bX n^{-1/2}\|_{\rm op} a_{t-1})
    \end{align*}
    and $\|\by-\bX\hbb^t\|\le \|\by\| + \|\bX n^{-1/2}\|_{\rm op}\|\hbb^t\|$.
    Since $\by = \by - \bX\hbb^1$ since $\hbb^1=\b0$, we also have
    $\|\by\|/\sqrt n = a_1 \le a_{t-1}$
    and $\|\by-\bX\hbb^t\| \le \sqrt n a_{t-1}C(\zeta)(1+\|\bX n^{-1/2}\|_{\rm op}^2)$.
    This proves
    \begin{equation}
        a_t\le C(\zeta)(1+\|\bX n^{-1/2}\|_{\rm op}^2) a_{t-1},
        \quad
        \text{ and }
        \quad
        a_T\le C(\zeta,T)(1+\|\bX n^{-1/2}\|_{\rm op}^2)^T a_1
    \end{equation}
    by induction, where $a_1 = \|\by\|/\sqrt n$.
    By the triangle inequality,
    $\|\bH\|_{\rm F} \le \sqrt{T} \|\bb^*\| + a_T$, so we have established
    \begin{equation*}
        \var(y_1)^{-1/2}
        \Bigl(
        \|\bH\|_{\rm F}
        + \|\bF\|_{\rm F}/\sqrt n
        \Bigr)
        \le C(\zeta,T)
        (1+\|\bX n^{-1/2}\|_{\rm op}^2)^T
        \Bigl(\frac{\|\by\|}{\sqrt n\var(y_1)^{1/2}} + 1\Bigr).
    \end{equation*}
    The moments of order 2, 4 and 8 (and any other finite moment)
    of the right-hand side are bounded by $C(\zeta,T,\gamma)$
    by \Cref{lem:moment-bound-X,lem:moment-bound-bep}.
\end{proof}

\subsection{Frobenius norm bounds on Jacobians}

\begin{lemma}[Frobenius norm bound of $\bF$ w.r.t. $\bX$]
\label{lem:Fnorm-dF-dZ}
Under \Cref{assu:Lipschitz,assu:regime},
\begin{equation*}
    \frac1n
    \big\|\pdv{\vec(\bF)}{\vec(\bX)}\big\|_{\rm F}^2
    =
    \frac1n\sum_{i=1}^n \sum_{j=1}^p \fnorm*{\pdv{\bF}{x_{ij}}}^2 
    \le C(\zeta, T) \Bigl(1 + \frac1n\opnorm{\bX}^2\Bigr)^{2T}  
    \Bigl(\|\bH\|_{\rm F} ^2 
    +  \gamma\|\bF\|_{\rm F}^2/n \Bigr),
\end{equation*}
Furthermore, if \Cref{assu:noise,assu:design} hold with $\bSigma=\bI_p$ then
\begin{equation*}
    \E \Bigl[\frac1n\sum_{i=1}^n\sum_{j=1}^p \fnorm*{\pdv{\bF}{x_{ij}}}^2 \Bigr]
    \le
    \E \Bigl[\Bigl(\frac1n\sum_{i=1}^n\sum_{j=1}^p  \fnorm*{\pdv{\bF}{x_{ij}}}^2 \Bigr)^2\Bigr]^{1/2}
    \le C(\zeta, T, \gamma) \var(y_1).
\end{equation*}
\end{lemma}
\begin{proof}[Proof of \Cref{lem:Fnorm-dF-dZ}]
    The first line is proved in \eqref{eq:frob-norm-bound-F-vec}.
    For the second line, we use the Cauchy-Schwarz inequality
    with the moment bounds from \Cref{lem:moment-bound-F-H,lem:moment-bound-X}.
\end{proof}

\begin{lemma}[Frobenius norm bound of $\bH$ w.r.t. $\bX$]
\label{lem:Fnorm-H}
Under \Cref{assu:Lipschitz,assu:regime},
\begin{align*}
    \sum_{i=1}^n\sum_{j=1}^p \fnorm*{\pdv{\bH}{x_{ij}}} ^2 
    \le \gamma C(\zeta, T) (1 + \opnorm{\bX}^2/n)^{2T} (\fnorm{\bF}^2/n + \fnorm{\bH}^2).
\end{align*}
In addition, if \Cref{assu:noise,assu:design} hold with $\bSigma=\bI_p$ then
\begin{align*}
    \E \Bigl[\sum_{i=1}^n\sum_{j=1}^p \fnorm*{\pdv{\bH}{x_{ij}}} ^2 \Bigr]
    \le
    \E \Bigl[\Bigl(\sum_{i=1}^n\sum_{j=1}^p \fnorm*{\pdv{\bH}{x_{ij}}}^2\Bigr)^2 \Bigr]^{1/2}
    \le~& C(\zeta, T, \gamma) \var(y_1).
\end{align*}
\end{lemma}

\begin{proof}[Proof of \Cref{lem:Fnorm-H}]
    The same argument as for
    \eqref{eq:frob-norm-bound-F-vec}
    would provide the desired bound.
    We provide an alternative argument to showcase another 
    means to control such quantities.
    By the expression of $\pdv{\hbb^t}{x_{ij}}$ in \eqref{eq:pdv-bt-x}, we have 
    $\pdv{\be_k^\top \bH \be_t}{x_{ij}} =
         \pdv{\be_k^\top (\hbb^t - \bb^*)}{x_{ij}} = \pdv{\be_k^\top \hbb^t}{x_{ij}}$ so that
    \begin{align*}
        \pdv{\be_k^\top \bH \be_t}{x_{ij}} &= n^{-1} \be_k^\top (\be_t^\top \otimes \bI_p) \calM^{-1}\calD
        [((\bF^\top \be_i) \otimes \be_j) - ((\bH^\top \be_j)  \otimes (\bX^\top\be_i))]&&\mbox{by \eqref{eq:pdv-bt-x}}\\
                                                                                        &= n^{-1}  (\be_t^\top \otimes \be_k^\top) \calM^{-1}\calD
        [(\bF^\top \otimes \bI_p) (\be_i \otimes \be_j) - (\bH^\top \otimes \bX^\top)
        (\be_j  \otimes \be_i)]   &&\mbox{by \eqref{eq:kronecker-mix-product}}.
    \end{align*}
    Therefore, using $(a+b)^2\le 2a^2+2b^2$,
    \begin{align*}
        &\sum_{i=1}^n\sum_{j=1}^p \fnorm*{\pdv{\bH}{x_{ij}}} 
        ^2 
        = \sum_{ij,kt}
        \Bigl(\pdv{\be_k^\top \bH \be_t}{x_{ij}}\Bigr)^2 \\
        &\le 2 n^{-2}\fnorm{\calM^{-1} \calD (\bF^\top \otimes \bI_p)}^2 
        + 2 n^{-2}\fnorm{\calM^{-1} \calD (\bH^\top \otimes \bX^\top)}^2 \\
        &\le 2 n^{-2}\opnorm{\calM^{-1} \calD}^2 \fnorm{\bF^\top \otimes \bI_p}^2 
        + 2 n^{-2}\opnorm{\calM^{-1} \calD}^2 \fnorm{\bH^\top \otimes \bX^\top}^2 \\
        &=  2 p n^{-2}\opnorm{\calM^{-1} \calD}^2 \fnorm{\bF}^2 
        + 2 n^{-2}\opnorm{\calM^{-1} \calD}^2 \fnorm{\bH}^2 \fnorm{\bX}^2 &&\mbox{by \eqref{eq:kronecker-norm}}\\
        & 
        \le  2 p n^{-2}\opnorm{\calM^{-1}\calD}^2 (\fnorm{\bF}^2 + \fnorm{\bH}^2 \opnorm{\bX}^2) &&\mbox{by $\fnorm{\bX} \le \sqrt{p} \opnorm{\bX}$}\\
                                                                                                 &\le  2 p n^{-2}\opnorm{\calM^{-1}}^2 \opnorm{\calD}^2 (\fnorm{\bF}^2 + \fnorm{\bH}^2 \opnorm{\bX}^2) &&\mbox{submultiplicativity of $\|\cdot\|_{\rm op}$}\\
        &
        \le  \gamma C(\zeta, T) (1 + \opnorm{\bX}^2/n)^{2T} (\fnorm{\bF}^2/n + \fnorm{\bH}^2) &&\mbox{by \eqref{eq:opnorm-D-J} and \eqref{eq:opnorm-M-N}}.
    \end{align*}
    For the upper bound on the moments, we use the Cauchy-Schwarz inequality
    with the moment bounds from \Cref{lem:moment-bound-F-H,lem:moment-bound-X}
    as for the proof
    of \Cref{lem:Fnorm-dF-dZ}.
\end{proof}

\begin{lemma}\label{lem:var-F-H}
    Under \Cref{assu:design,assu:regime,assu:Lipschitz,assu:noise}, we have 
    \begin{align}
        \label{eq:variance_to_0_F}
        \E[\|\bF^\top\bF/n - \E[\bF^\top\bF/n]\|_{\rm F}^2]
        &\le C(\gamma,\zeta,T,\kappa) \var(y_1)^2/ n ,
        \\
        \E[\|\bH^\top\bSigma\bH   - \E[\bH^\top\bSigma\bH  ]\|_{\rm F}^2]
        &\le C(\gamma,\zeta,T,\kappa) \var(y_1)^2/ n.
        \label{eq:variance_to_0_H}
    \end{align}
    Consequently, if $\var(y_1)$ is bounded from above by a constant, by Markov's inequality,
    \begin{equation}
        \bF^\top \bF/n - \E (\bF^\top \bF/n) = O_P(n^{-1/2})\text{ and } \bH^\top \bSigma \bH -\E (\bH^\top \bSigma \bH) = O_P(n^{-1/2}).
    \end{equation}
\end{lemma}
\begin{proof}[Proof of \Cref{lem:var-F-H}]
    By the change of variable argument in \Cref{lem:moment-bound-F-H}, it suffices to prove the results under $\bSigma = \bI_p.$ 
    We view $\bF$ as a function of $(\bX,\bep)$. 
    By the Gaussian Poincar\'e inequality applied to $\be_t^\top\bF^\top\bF\be_s$ for each $t,s\in [T]$, we find
    \begin{equation}
        \var\Bigl((\bF \be_t)^\top \bF \be_s\Bigr) 
    \le
    \E \sum_{i=1}^n \sigma^2 \Big(\pdv{( \be_t^\top  \bF^\top \bF \be_s)}{\ep_i}\Big)^2 + 
    \E \sum_{i=1}^n \sum_{j=1}^p \Big(\pdv{(\be_t ^\top \bF^\top \bF \be_s)}{x_{ij}}\Big)^2
    \label{eq:poincare_FTF_first}
    \end{equation}
    Let $\partial$ denote either $\partial/\partial x_{ij}$ or $\partial/\partial\eps_i$.
    Using the product rule $\partial(\be_t^\top\bF^\top\bF\be_s))
    =\be_t^\top (\partial \bF)^\top\bF \be_s + \be_t^\top\bF^\top (\partial \bF)\be_s$ as well as
    $(a+b)^2\le 2a^2 + 2b^2$ 
    and summing
    over all $s\in[T]$ and all $t\in[T]$,
    \begin{align}
    \sum_{s=1}^T\sum_{t=1}^T
        \var\Bigl((\bF \be_t)^\top \bF \be_s\Bigr) 
        &\le
        4\E
        \sum_{s=1}^T\sum_{t=1}^T\sum_{i=1}^n
        \Bigl(
            \sigma^2
            \Bigl(\be_s^\top\bF^\top\frac{\partial \bF\be_t}{\partial\eps_i}\Bigr)^2
        +\sum_{j=1}^p
            \Bigl(\be_s^\top\bF^\top\frac{\partial \bF\be_t}{\partial x_{ij}}\Bigr)^2
        \Bigr)
        \nonumber
      \\&= 4\E\Bigl[
        \sum_{i=1}^n
        \Bigl(
        \sigma^2
      \|\bF^\top \frac{\partial \bF}{\partial\eps_i}\|_{\rm F}^2
      +\sum_{j=1}^p \|\bF^\top \frac{\partial \bF}{\partial x_{ij}}\|_{\rm F}^2
      \Bigr) 
      \Bigr]
        \label{eq:Poincare_F}
    \end{align}
    where we used $\sum_{s=1}^T\be_s\be_s^\top=\bI_T$ and similarly for the sum
    over $t\in[T]$.
    We rewrite the above using the vectorization operator:
    $\|\bF^\top\tfrac{\partial}{\partial x_{ij}} \bF\|_{\rm F}^2 
    = \|(\bI_T\otimes \bF^\top)\vec(\tfrac{\partial}{\partial x_{ij}}\bF)\|^2$,
    which is also the squared norm of the $(i,j)$-th column of
    $(\bI_T\otimes \bF^\top) \frac{\partial \vec \bF}{\partial \vec \bX}$,
    so that 
    \begin{equation}
        \sum_{i=1}^n\sum_{j=1}^p
        \|\bF^\top \frac{\partial \bF}{\partial x_{ij}}\|_{\rm F}^2
        =
        \Big
        \|(\bI_T\otimes \bF^\top)\frac{\partial \vec(\bF)}{\partial \vec(\bX)}
        \Big\|_{\rm F}^2
        \le T \|\bF\|_{\rm F}^2
        \Big
        \|\frac{\partial \vec(\bF)}{\partial \vec(\bX)}
        \Big\|_{\rm op}^2
        \label{eq:Poincare_F_post_x_ij}
    \end{equation}
    using $\|\bM\bM'\|_{\rm F}\le \|\bM\|_{\rm F}\|\bM'\|_{\rm op}$ for the inequality.
    By the same argument,
    \begin{equation}
        \sum_{i=1}^n
        \|\bF^\top \frac{\partial \bF}{\partial \eps_i}\|_{\rm F}^2
        =
        \Big
        \|(\bI_T\otimes \bF^\top)\frac{\partial \vec(\bF)}{\partial \bep}
        \Big\|_{\rm F}^2
        \le T \|\bF\|_{\rm F}^2
        \Big
        \|\frac{\partial \vec(\bF)}{\partial \bep}
        \Big\|_{\rm op}^2.
        \label{eq:Poincare_F_post_eps_i}
    \end{equation}
    By \eqref{eq:op-norm-bound-F-vec}
    and \Cref{eq:opnorm-dF/de}, both previous displays
    are bounded from above by 
    $$C(\zeta,T,\gamma)\|\bF\|_{\rm F}^2\max\{\sigma^2, \|\bH\|_{\rm F}^2 + \|\bF\|_{\rm F}^2/n \}(1+\|\bX\|_{\rm op}^2/n)^T.$$
    Using
    the Cauchy-Schwarz inequality to leverage
    the moment bounds \eqref{lem:moment-bound-X} and \eqref{lem:moment-bound-F-H},
    we find that \eqref{eq:Poincare_F} is bounded from above
    by $C(\zeta,T,\gamma) n$ and the proof of \eqref{eq:variance_to_0_F}
    is complete.

    By exactly the same argument,
    \eqref{eq:poincare_FTF_first},
    \eqref{eq:Poincare_F},
    \eqref{eq:Poincare_F_post_eps_i} and
    \eqref{eq:Poincare_F_post_x_ij}
    hold
    with $\bF$
    replaced by $\bH$. We use the upper bounds
    \eqref{eq:op-norm-bound-H-vec} and \eqref{eq:opnorm-dH/de}
    to control the operator norm of
    $\frac{\partial \vec(\bF)}{\partial \vec(\bX)}$
    and
    $\frac{\partial \vec(\bF)}{\partial \bep}$,
    and the moment
    bounds \eqref{lem:moment-bound-X} and \eqref{lem:moment-bound-F-H}
    to obtain \eqref{eq:variance_to_0_H}.
    \end{proof}

\section{Proof of Theorem~\ref{thm:generalization-error}}
\label{sec:proof-thm-generalization-error}
Throughout this proof, we assume $\bSigma = \bI_p$, 
and the proof for general $\bSigma$ follows the same line of argument by changing $\zeta$ to $\kappa\zeta$, thanks to the change of variables in \Cref{sec:change-of-var}. 

Before stating the proof, we recall a few defintions from \eqref{eq:F-H}: 
\begin{equation*}
    \bH = [\hbb^1 - \bb^*, \ldots, \hbb^T - \bb^*] \in \R^{p \times T}, \quad 
    \bF = [\by - \bX \hbb^1, \ldots, \by - \bX \hbb^T] \in \R^{n \times T}.
\end{equation*}
We first derive the upper bound of $\var(r_{tt'})$. 
By definition of $r_{tt'}$ and $\bH$, we know
$r_{tt'}$ is the $(t,t')$ entry of $\bH^\top \bH$. Thus,
\begin{align*}
    &\var(r_{tt'})\\
    = & \E [(r_{tt'} - \E[r_{tt'}])^2]\\
    = & \E \big[\big([\bH^\top \bH]_{t,t'} - \E[[\bH^\top \bH]_{t,t'}]\big)^2\big]\\
    \le& \E\big[\fnorm{\bH^\top \bH - \E[\bH^\top \bH]}^2\big] \\
    \le& n^{-1} C(\gamma, \zeta, T, \kappa) \var(y_1)^2 &&\mbox{by \Cref{lem:var-F-H}}.
\end{align*}
It thus remains to show $\E[|r_{tt'} - \hat r_{tt'}|] \le n^{-1/2} C(\gamma, \zeta, T, \kappa) \var(y_1)$.
Define
$\bS = \sigma^2 \bd1_T \bd1_T^\top\in \R^{T\times T}$, then we have 
\begin{align*}
    \E[|r_{tt'} - \hat r_{tt'}|] 
    =& \E\Big[[\bH^\top \bH + \bS]_{t,t'} - [(\bI_T - \hbA)^{-1} \bF^\top\bF/n (\bI_T - \hbA^\top)^{-1}]_{t,t'}\Big]\\
    \le& \E\Big[\fnorm[\big]{\bH^\top \bH + \bS - (\bI_T - \hbA)^{-1} \bF^\top\bF/n (\bI_T - \hbA^\top)^{-1}}\Big].
\end{align*}
So it suffices to show 
\begin{equation}
    \label{eq:key-equality}
    \E 
[\fnorm{\bH^\top \bH + \bS - (\bI_T - \hbA)^{-1} \bF^\top\bF/n (\bI_T - \hbA^\top)^{-1}}] \le n^{-1/2} C(\zeta, T, \gamma) \var(y_1).
\end{equation}
To this end, we define
\begin{equation}\label{hbC}
    \hbC = \sum_{j=1}^p (\bI_T \otimes \be_j^\top) \calM^{-1} \calD (\bI_T \otimes \be_j)
    \in \R^{T\times T},
\end{equation}
where $\calM = \bI_{pT} + \calD (\bI_T \otimes \tfrac{\bX^\top \bX}{n}) - \calJ$ as in \Cref{lem:dot-b-4}. 
We also define the matrices $\bQ_1,\bQ_2\in\R^{T\times T}$ that are bounded 
in \Cref{prop:A1,prop:A2} below: 
\begin{align*}
    \bQ_1 
    &= n^{-1/2}\big[\bF^\top \bF (\bI_T + \hbC/n)^\top - (n\bI_T - \hbA) (\bH^\top \bH + \bS)\big], \\
    \bQ_2 
    &= n^{-1/2} \big[n(\bH^\top \bH + \bS) - (\bI_T + \hbC/n) \bF^\top \bF (\bI_T + \hbC/n)^\top\big].
\end{align*}

\begin{proposition}[Proof is given in \Cref{proof:A1}]\label{prop:A1}
    Let \Cref{assu:Lipschitz,assu:noise,assu:design,assu:regime} be fulfilled and $\bSigma = \bI_p$, then we have 
    $\E[\|\bQ_1\|_{\rm F} ^2] \le C(\zeta, T, \gamma) \var(y_1)^2.$
    As a consequence, by Jensen's inequality,
    $\E[\|\bQ_1\|_{\rm F}] \le C(\zeta, T, \gamma) \var(y_1).$
\end{proposition}
\begin{proposition}[Proof is given in \Cref{proof:A2}] \label{prop:A2}
	Let \Cref{assu:Lipschitz,assu:noise,assu:design,assu:regime} be fulfilled and $\bSigma = \bI_p$, then 
 $\E [\fnorm*{\bQ_2}] \le C(\zeta, T, \gamma) \var(y_1).$
\end{proposition}

Now we are ready to prove \Cref{eq:key-equality} using the above two propositions.
For brevity, let
$\bV = [-\bH^\top, \sigma \bd1_T]^\top \in \R^{(p+1)\times T}$
and the lower triangular matrices $\bL= \bI_T - \hbA/n$
and $\bT = \bI_T + \hbC/n$.
With this notation,
\begin{align}
    \bQ_1 &= \sqrt n [
    (\bF^\top\bF/n) \bT^\top
    - \bL\bV^\top\bV
    ],
    \label{Q1}
        \\\bQ_2&=
        \sqrt n
        [
        \bV^\top\bV
        - \bT(\bF^\top\bF/n) \bT^\top
        ].
        \label{Q2}
    \end{align}
By expanding the expressions of $\bQ_1$ and $\bQ_2$ in \eqref{Q1}-\eqref{Q2}, we have 
by simple algebra
\begin{align}
    &n^{-1/2}
    \bigl[
        \bQ_1^\top (\bL^\top)^{-1}
        +
        (\bL^{-1} \bQ_1 + \bQ_2)(\bT^\top)^{-1}(\bL^\top)^{-1}
    \bigr]
  \nonumber
  \\
  \nonumber
    &= \begin{cases}
      {\color{blue}\bT ({\bF^\top\bF}/{n})(\bL^\top)^{-1}}
  -\bV^\top\bV
  \\ + \bL^{-1}({\bF^\top\bF}/{n})(\bL^\top)^{-1}
  - {\color{red}\bV^\top\bV(\bT^\top)^{-1}(\bL^\top)^{-1}}
  \\ + {\color{red}\bV^\top\bV(\bT^\top)^{-1}(\bL^\top)^{-1}}
  - {\color{blue}\bT(\bF^\top\bF/n)(\bL^\top)^{-1}}
\end{cases}
  \\&= \bL^{-1} (\bF^\top\bF/n)(\bL^\top)^{-1} - \bV^\top\bV
  \nonumber
  \\&=
        (\bI_T - \tfrac1n\hbA)^{-1} (\bF^\top \bF/n) (\bI_T - \tfrac1n\hbA^\top)^{-1}- (\bH^\top \bH + \bS) 
               \label{eq:to_bound_FTF_LHTHLT}
\end{align}
as all terms except two cancel out. Consequently by the triangle inequality,
\begin{equation*}
\sqrt n\fnorm{\bL^{-1} (\bF^\top\bF/n)(\bL^\top)^{-1} - \bV^\top\bV}
\le \max\{1,\|\bL^{-1}\|_{\rm op}^3,\|\bT\|_{\rm op}^3\}(2\|\bQ_1\|_{\rm F}+\|\bQ_2\|_{\rm F})
.
\end{equation*}
Let 
$\Omega = 
\{\bX\in \R^{n\times p}: 
\opnorm{\bX}/\sqrt{n} \le 2 + \sqrt{\gamma}\}$. 
Under \Cref{assu:design} and here $\bSigma=\bI_p$,
\cite[Theorem II.13]{DavidsonS01} implies that $\P(\Omega) \ge 1 - e^{-n}$. 
Then in the event $\Omega$, we have by \Cref{lem:opnorm-Ahat,lem:opnorm-Chat}, 
\begin{align}
\max\{\|\bL\|_{\rm op}, \|\bL^{-1}\|_{\rm op},\|\bT\|_{\rm op}, \|\bT^{-1}\|_{\rm op} \}
    \le C(\zeta, T, \gamma),
\end{align}
so that by the previous display and \Cref{prop:A1,prop:A2},
\begin{align}
     \E[
    \mathbb I(\Omega)
    \sqrt n
    \|\bL^{-1} (\bF^\top\bF/n)(\bL^\top)^{-1} - \bV^\top\bV \|_{\rm F} 
    ]
    &\le C(T,\zeta,\gamma) \E[\|\bQ_1\|_{\rm F} + \|\bQ_2\|_{\rm F}]
    \nonumber
    \\&\le C(T,\zeta,\gamma) \var(y_1).
    \label{eq_conclusion_Omega}
\end{align}
On the other hand, the same expectation with $\Omega^c$ is exponentially
small due to
\begin{align}
    \E[\I(\Omega^c)\|\eqref{eq:to_bound_FTF_LHTHLT}\|_{\rm F}]
    &\le 
    \P(\Omega^c)^{1/2}
    \E[\{(1+\|\bL\|_{\rm op}^2)(\|\bF\|_{\rm F}^2/n + \|\bV\|^2_{\rm F})\}^2]^{1/2}
    &&\text{(C. Schwarz)}
    \nonumber
  \\&\le\P(\Omega^c)^{1/2}
    \E[(1+\|\bL\|_{\rm op}^2)^4]^{1/4}
    \E[(\|\bF\|_{\rm F}^2/n + \|\bV\|^2_{\rm F})^4]^{1/4}
    &&\text{(C. Schwarz)}
    \nonumber
\\&\le
    e^{-n/2}
    C(\zeta,T,\gamma)
    \var(y_1)
    \label{eq_conclusion_Omega_c}
\end{align}
thanks to
\Cref{lem:opnorm-Ahat,lem:opnorm-Chat,lem:moment-bound-X} 
and $\P(\Omega)\le e^{-n}$
for the last line.
This completes the proof
of \Cref{thm:generalization-error} for $\bSigma = \bI_p$. 

    For $\bSigma \ne \bI_p$, we apply the change of variables argument presented in \Cref{sec:change-of-var} to achieve the desired result. In this context, the constant $C$ is dependent on $\zeta, T, \gamma, \kappa$. 
    This concludes the proof of \Cref{thm:generalization-error}.

In passing, let us mention that we also have by definitions of $\bQ_1$ and $\bQ_2$
that
\begin{equation}
    [\bL^{-1}\bQ_1(\bT^\top)^{-1} + \bQ_2(\bT^\top)^{-1}](\bL^{-1}-\bT)^\top
    =\sqrt n (\bL^{-1}-\bT)(\bF^\top\bF/n)(\bL^{-1}-\bT)^\top
\end{equation}
holds. By the same argument as in \eqref{eq_conclusion_Omega}-\eqref{eq_conclusion_Omega_c}, that the right-hand side of the previous display is bounded
as 
\begin{equation}
    \sqrt n 
    \E[
    \|(\bL^{-1}-\bT)\bF^\top n^{-1/2}\|_{\rm F}^2]
    \le C(T,\zeta,\gamma)\var(y_1).
    \label{eq:bound_alternative_weights}
\end{equation}

\subsection{Proofs of \Cref{prop:A1}}
\label{proof:A1}
We frist write 
$$
\bF^\top \bF 
= \bF^\top \Bigl[\bX,  \frac{\bep}{\sigma}\Bigr]
\bigl[-\bH^\top, \sigma \bd1_T \bigr]^\top.
$$
Applying \cite[Lemma E.10] 
{tan2022noise} 
to $\bU = \bF\in \R^{n\times T}$, 
$\bZ = [\bX, \bep/\sigma] \in \R^{n\times (p+1)}$, and $\bV = \bigl[-\bH^\top, \sigma \bd1_T \bigr]^\top\in \R^{(p+1)\times T}$
gives 
\begin{align}
    &\E \Big[\Big\|\bU^\top \bZ \bV - 
    \sum_{j=1}^{p+1}\sum_{i=1}^n
    \frac{\partial}{\partial z_{ij} }\Bigl(\bU^\top \be_i \be_j^\top \bV \Bigr)\Big\|_{\rm F}^2\Big] \label{UZV-LHS}
    \\ \le~ & \E [\fnorm*{\bU}^2 \fnorm*{\bV}^2]+ \E \sum_{i=1}^n \sum_{j=1}^{p+1}\Big[
    2\fnorm*{\bV}^2\fnorm*{ \frac{\partial \bU}{\partial z_{ij}} }^2
    + 2\fnorm*{\bU}^2\fnorm*{ \frac{\partial \bV}{\partial z_{ij}} }^2\Big].\label{UZV-RHS}
\end{align}
To prove \Cref{prop:A1}, we need the following two lemmas. 
\begin{lemma}[Proof is given on \Cpageref{proof:A1-step-i}]
    \label{A1-step-i}
    Let the assumptions of \Cref{prop:A1} be fulfilled. 
    For $\bU = \bF$, 
    $\bZ = [\bX, \bep/\sigma]$, and $\bV = \bigl[-\bH^\top, \sigma \bd1_T \bigr]^\top$, 
    we have 
    \begin{align*}
        \bU^\top \bZ \bV - 
        \sum_{j=1}^{p+1}\sum_{i=1}^n
        \frac{\partial}{\partial z_{ij} }\Bigl(\bU^\top \be_i \be_j^\top \bV \Bigr)
        = \sqrt{n}\bQ_1 + \Rem_1,
    \end{align*}
    where $\Rem_1$ is a $T\times T$ matrix satisfying 
    $\E[\fnorm{\Rem_1}^2] 
    \le nC(\zeta, T, \gamma) \var(y_1)^2.$
\end{lemma}

\begin{lemma}[Proof is given on \Cpageref{proof:A1-step-ii}]
    \label{A1-step-ii}
    Let the assumptions of \Cref{prop:A1} be fulfilled, then
    \begin{equation*}
        \eqref{UZV-RHS} \le nC(\zeta, T, \gamma) \var(y_1)^2.
    \end{equation*}
\end{lemma}
Now we prove \Cref{prop:A1} using the above two lemmas.
According to \Cref{A1-step-i}, we have
$\sqrt{n} \bQ_1 = \bU^\top \bZ \bV - 
\sum_{j=1}^{p+1}\sum_{i=1}^n
\frac{\partial}{\partial z_{ij} }(\bU^\top \be_i \be_j^\top \bV) - \Rem_1$. 
By triangular inequality, we obtain 
\begin{align*}
    &n \E [\fnorm{\bQ_1}^2] \\
    \le & 2 \E \Big[\fnorm{\bU^\top \bZ \bV - 
    \sum_{j=1}^{p+1}\sum_{i=1}^n
    \frac{\partial}{\partial z_{ij} }(\bU^\top \be_i \be_j^\top \bV)}^2\Big] + 2 \E [\fnorm{\Rem_1}^2]\\
    \le& 2 \Bigl(\eqref{UZV-RHS} + \E[\fnorm{\Rem_1}^2]\Bigr)\\
    \le &
    n C(\zeta, T, \gamma) \var(y_1)^2 &&\mbox{by \Cref{A1-step-i,A1-step-ii}.}
\end{align*}
This completes the proof of \Cref{prop:A1}.

\subsection{Proofs of \Cref{prop:A2}}
\label{proof:A2}
We first state a useful lemma, which is an extension of Lemma E.12 in \cite{tan2022noise} to allow $\fnorm{\bU}\ge 1$ and $\fnorm{\bV}\ge 1$. 
\begin{lemma}[Proof is given on \Cpageref{proof:lem:Chi2type}] \label{lem:Chi2type}
    Let $\bU,\bV: \R^{n\times p} \to \R^{n\times T}$
    be two locally Lipschitz 
    functions of $\bZ$ with \iid $\mathsf{N}(0,1)$ entries.
    Provided the following expectations are finite, we have 
    \begin{align*}
        &\E\Bigl[
        \fnorm[\Big]{
            p \bU^\top \bV - \sum_{j=1}^p 
            \Bigl(\sum_{i=1}^n \partial_{ij} \bU^\top \be_i - \bU^\top \bZ \be_j\Bigr)
            \Bigl(\sum_{i=1}^n \partial_{ij} \be_i^\top \bV  - \be_j^\top \bZ^\top \bV\Bigr)
        }
        \Bigr] 
        \\
        \le~&
        (1 + 2\sqrt{p}) 
        \bigl( 
            \E [\fnorm{\bU}^4]^{1/2}
            + \E [\fnorm{\bV}^4]^{1/2}
            + \E [\norm{\bU}^4_{\partial}]^{1/2}
            + \E [\norm{\bV}^4_{\partial}]^{1/2}
            \bigr),
    \end{align*}
    where $ \partial_{ij} \bU = \partial \bU /\partial z_{ij}$ 
    and $\|\bU\|_{\partial} = (\sum_{i=1}^n\sum_{j=1}^p\|\partial_{ij} \bU\|_{\rm F}^2)^{1/2}$.
\end{lemma}
To apply \Cref{lem:Chi2type}, we consider the following mapping:
$$
\R^{(p+1)\times n}\to \R^{(p+1) \times T}: 
\bZ^\top \mapsto \bV,$$
where $\bZ = [\bX, \frac{\bep}{\sigma}]$ and $\bV = [-\bH^\top, \sigma \bd1_T]^\top$.
Applying \Cref{lem:Chi2type} to $\bU=\bV = [-\bH^\top, \sigma \bd1_T]^\top$ and the Gaussian matrix $\bZ^\top$, we have 
\begin{equation}\label{eq:A2-Stein}
\begin{aligned}
    &\E\Bigl[
    \fnorm[\Big]{
        n \bV^\top \bV - \sum_{i=1}^n 
        \Bigl(\sum_{j=1}^{p+1}  \pdv{\bV^\top}{z_{ij}} \be_j - \bV^\top \bZ^\top \be_i\Bigr)
        \Bigl(\sum_{j=1}^{p+1}  \pdv{\be_j^\top \bV}{z_{ij}}  - \be_i^\top \bZ \bV \Bigr)
    }
    \Bigr] 
    \\
    \le~&
    2(1 + 2\sqrt{n}) 
    \bigl( 
        \E [\fnorm{\bV}^4]^{1/2}
        + \E [\norm{\bV}^4_{\partial}]^{1/2}
        \bigr),
\end{aligned}
\end{equation}
where $\norm{\bV}_{\partial}
:= (\sum_{i=1}^{n}\sum_{j=1}^{p+1}
\fnorm{\pdv{\bV}{z_{ij}}}^2)^{1/2}
$. 
The desired bound of $\E[\fnorm{\bQ_2}]$ then follows from the subsequent two lemmas. 
\begin{lemma}[Proof is given on \Cpageref{proof:A2-Rem}]
    \label{lem:A2-Rem}
    Let the assumptions of \Cref{prop:A2} be fulfilled.
    For $\bZ = [\bX, \frac{\bep}{\sigma}]$ and $\bV = [-\bH^\top, \sigma \bd1_T]^\top$, we have 
    \begin{align}\label{A2-LHS}
        n \bV^\top \bV - \sum_{i=1}^n 
        \Bigl(\sum_{j=1}^{p+1}  \pdv{\bV^\top}{z_{ij}} \be_j - \bV^\top \bZ^\top \be_i\Bigr)
        \Bigl(\sum_{j=1}^{p+1}  \pdv{\be_j^\top \bV}{z_{ij}}  - \be_i^\top \bZ \bV \Bigr) = \sqrt{n}\bQ_2 - \Rem_2, 
    \end{align}
    where $\Rem_2$ is a $T\times T$ matrix satisfying 
    $\E[\fnorm{\Rem_2}] \le \sqrt{n} C(\zeta, T, \gamma) \var(y_1)$. 
\end{lemma}
\begin{lemma}[Proof is given on \Cpageref{proof:A2-RHS}]
    \label{lem:A2-RHS}
    Under the same conditions of \Cref{prop:A2}, for $\bV = [-\bH^\top, \sigma \bd1_T]^\top$, we have 
    \begin{align*}
        2(1 + 2\sqrt{n}) 
    \bigl( 
        \E [\fnorm{\bV}^4]^{1/2}
        + \E [\norm{\bV}^4_{\partial}]^{1/2}
        \bigr)
        &\le \sqrt{n} C(\zeta, T, \gamma) \var(y_1). 
    \end{align*}
\end{lemma}
Now we are ready to prove \Cref{prop:A2}.
By the triangle inequality, we have
\begin{align*}
     &\E [\fnorm{\sqrt{n}\bQ_2}]\\
    \le& \E[\fnorm{\Rem_2}] + 2(1 + 2\sqrt{n}) 
    \bigl( 
        \E [\fnorm{\bV}^4]^{1/2}
        + \E [\norm{\bV}^4_{\partial}]^{1/2}
        \bigr)&&\mbox{by \eqref{eq:A2-Stein} and \eqref{A2-LHS}}\\
    \le& \sqrt{n} C(\zeta, T, \gamma) \var(y_1) &&\mbox{by \Cref{lem:A2-Rem,lem:A2-RHS}}.
\end{align*}
This finishes the proof of \Cref{prop:A2}.

\subsection{Proofs of supporting lemmas}

\begin{proof}[Proof of \Cref{A1-step-i}]
    \label{proof:A1-step-i}
        First, by definitions of $\bU, \bZ, \bV$ in \Cref{A1-step-i}, we have 
        $$\bU^\top \bZ \bV = \bF^\top \bF.$$
    
        Next, by product rule and spliting the summation over $j$ into two parts: $j\in [p]$ and $j=p+1$, we have 
        \begin{align}
            &\sum_{j=1}^{p+1}\sum_{i=1}^n
            \frac{\partial}{\partial z_{ij} }\Bigl(\bU^\top \be_i \be_j^\top \bV \Bigr)\notag\\
            =& \sum_{j=1}^{p+1}\sum_{i=1}^n
            \Bigl(\frac{\partial \bU^\top}{\partial z_{ij} } \be_i \be_j^\top \bV  + \bU^\top \be_i \be_j^\top \frac{\partial \bV}{\partial z_{ij}} \Bigr)\notag\\
            =& \sum_{j=1}^{p}\sum_{i=1}^n
            \Bigl(\frac{\partial \bF^\top}{\partial x_{ij} } \be_i \be_j^\top (-\bH) \Bigr) + \sum_{i=1}^n \pdv{\bF^\top}{\eps_i/\sigma} \be_i \sigma \bd1_T^\top
            - \sum_{i=1}^n \sum_{j=1}^p \bF^\top \be_i \be_j^\top \frac{\partial \bH}{\partial x_{ij}} \notag\\
            =& -\sum_{j=1}^{p}\sum_{i=1}^n
            \Bigl(\frac{\partial \bF^\top}{\partial x_{ij} } \be_i \be_j^\top \bH \Bigr) + \sigma^2 \sum_{i=1}^n \pdv{\bF^\top}{\eps_i} \be_i \bd1_T^\top
            - \sum_{i=1}^n \sum_{j=1}^p \bF^\top \be_i \be_j^\top \frac{\partial \bH}{\partial x_{ij}}. \label{abc}
        \end{align}
        For the first term in \eqref{abc}, we have its transpose is 
        \begin{align*}
            \quad & -\bH^\top \sum_{j=1}^{p}\sum_{i=1}^n \be_j \be_i^\top 
            \frac{\partial \bF}{\partial x_{ij} }
            =
            \bH^\top \bH (n\bI_T - \hbA)^\top - \Rem_{1,1} 
        \end{align*}
        by \eqref{eq:A1-1}.
        For the second term in \eqref{abc}, we have its transpose is  
        \begin{align*}
            &\sigma^2 \bd1_T \sum_{i=1}^n \be_i^\top \pdv{\bF}{\eps_i} \\ 
            =& \sigma^2 \bd1_T \sum_{i=1}^n \sum_{t=1}^T \be_i^\top \pdv{\bF\be_t}{\eps_i} \be_t^\top\\
            =& \sigma^2 \bd1_T  \sum_{t=1}^T \trace\Big(\pdv{\bF\be_t}{\bep}\Big) \be_t^\top\\
            =& \sigma^2 \bd1_T  \sum_{t=1}^T \be_t^\top (n\bI_T - \hbA) \bd1_T \be_t^\top &&\mbox{by \eqref{eq:lem:div-F-eps}}\\
            =& \sigma^2 \bd1_T  \bd1_T^\top (n\bI_T - \hbA)^\top \\
            =& \bS (n\bI_T - \hbA)^\top. 
        \end{align*}
        For the third term in \eqref{abc}, we have by \eqref{eq:A1-2}, 
        \begin{align*}
            \sum_{i=1}^n \sum_{j=1}^p \bF^\top \be_i \be_j^\top \frac{\partial \bH}{\partial x_{ij}}
            = \bF^\top \bF \hbC^\top/n + \Rem_{1,2}. 
        \end{align*}
        Combining the three terms in \eqref{abc}, we have 
        \begin{align*}
            \eqref{abc} = (n\bI_T - \hbA)(\bH^\top \bH + \bS)
            - \bF^\top \bF \hbC^\top/n
            - \Rem_{1,1}^\top - \Rem_{1,2}.
        \end{align*}
        It follows that 
        \begin{align*}
            &\bU^\top \bZ \bV - 
            \sum_{j=1}^{p+1}\sum_{i=1}^n
            \frac{\partial}{\partial z_{ij} }\Bigl(\bU^\top \be_i \be_j^\top \bV \Bigr)\\
            =& \bF^\top \bF (\bI_T + \hbC/n)^\top - (n\bI_T - \hbA)(\bH^\top \bH + \bS) + (\Rem_{1,1}^\top + \Rem_{1,2})\\
            =& \sqrt{n}\bQ_1 + \Rem_1,
        \end{align*}
        where $\Rem_1 = (\Rem_{1,1}^\top + \Rem_{1,2})$, and $\Rem_{1,1}$ and $\Rem_{1,2}$ are defined in \Cref{lem:A1-Rems}.
    
        It remains to bound $\E[\fnorm{\Rem_1}^2]$.
        By the triangle inequality, 
        \begin{align*}
            \E[\fnorm{\Rem_1}^2] 
            \le& 2 \E[\fnorm{\Rem_{1,1}}^2] + 2\E[\fnorm{\Rem_{1,2}}^2]\\
            \le& nC(\zeta, T, \gamma) \var(y_1)^2 &&\mbox{by \eqref{eq:moment-bound-Rems}.}
        \end{align*}
        This concludes the proof of \Cref{A1-step-i}.
    \end{proof}
    
    \begin{proof}[Proof of \Cref{A1-step-ii}]
        \label{proof:A1-step-ii}
        Let us recall \eqref{UZV-RHS} here for convenience:
        \begin{align*}
            \eqref{UZV-RHS}
            =
            \E \Bigl[\fnorm*{\bU}^2 \fnorm*{\bV}^2\Bigr]+ 
            2\E \Big[\fnorm*{\bV}^2 \sum_{i=1}^n \sum_{j=1}^{p+1}
            \fnorm*{ \frac{\partial \bU}{\partial z_{ij}} }^2\Big]
            + 2 
            \E \Big[\fnorm*{\bU}^2 \sum_{i=1}^n \sum_{j=1}^{p+1}
            \fnorm*{ \frac{\partial \bV}{\partial z_{ij}} }^2\Big].
        \end{align*}
        For the first term in \eqref{UZV-RHS}, since 
        $\fnorm{\bV}^2 = \fnorm{\bH}^2 + T \sigma^2$, 
        we have 
        \begin{align*}
            \E [\fnorm*{\bU}^2 \fnorm*{\bV}^2]
            =& \E [\fnorm{\bF}^2 (\fnorm{\bH}^2 + T\sigma^2)]\\
            \le& \E [\fnorm{\bF}^4]^{1/2} \E [\fnorm{\bH}^4 + T^2\sigma^4]^{1/2} &&\mbox{by C. Schwarz}\\
            \le& n C(\zeta, T, \gamma) \var(y_1)^2 &&\mbox{by \Cref{lem:moment-bound-F-H}.}
        \end{align*}
        For the second term in \eqref{UZV-RHS}, we have
        \begin{align*}
            &\E \Big[\fnorm*{\bV}^2 \sum_{i=1}^n \sum_{j=1}^{p+1}
            \fnorm*{\frac{\partial \bU}{\partial z_{ij}} }^2\Big]\\
            =& \E \Big[(\fnorm{\bH}^2 + T\sigma^2)  \Big(\sum_{i=1}^n \sum_{j=1}^{p}
            \fnorm*{ \frac{\partial \bF}{\partial x_{ij}} }^2 + \sigma^2 \sum_{i=1}^n \fnorm*{\pdv{\bF}{\eps_i}}^2\Big) \Big]\\
            =& \E \Big[(\fnorm{\bH}^2 + T\sigma^2)  
            \Big(\fnorm*{\pdv{\vec(\bF)}{\vec(\bX)}}^2 + \sigma^2 \fnorm*{\pdv{\vec(\bF)}{\bep}}^2\Big) \Big]\\
            \le& \E \Big[(\fnorm{\bH}^2 + T\sigma^2)  
            \Big(\fnorm*{\pdv{\vec(\bF)}{\vec(\bX)}}^2 + \sigma^2 \fnorm*{\pdv{\vec(\bF)}{\bep}}^2\Big) \Big]\\
            \le & C(\zeta, T) \E [(\fnorm{\bH}^2 + T\sigma^2) (1 + \opnorm{\bX}^2/n)^{2T}  
            (\|\bH\|_{\rm F}^2+  \gamma\|\bF\|_{\rm F}^2/n)] &&\mbox{by \Cref{lem:Fnorm-dF-dZ}}\\ 
            & + n \sigma^2 C(\zeta, T) \E [(\fnorm{\bH}^2 + T\sigma^2) (1 + \opnorm{\bX}^2/n)^{2T}] &&\mbox{by \eqref{eq:opnorm-dF/de}}\\
        \le& n C(\zeta, T, \gamma) \var(y_1)^2 &&\mbox{by \Cref{lem:moment-bound-F-H,lem:moment-bound-X}.}
        \end{align*}
        
        For the third term in \eqref{UZV-RHS}, since 
        $\bV = 
        \begin{bmatrix}
            -\bH\\
            \sigma \bd1_T^\top
        \end{bmatrix}
        $,
        we have
        \begin{align*}
            &\E \Big[\fnorm*{\bU}^2 \sum_{i=1}^n \sum_{j=1}^{p+1}
            \fnorm*{\frac{\partial \bV}{\partial z_{ij}} }^2\Big]\\
            =& \E \Big[\fnorm{\bF}^2 \sum_{i=1}^n \sum_{j=1}^{p+1}
            \fnorm*{\frac{\partial \bH}{\partial z_{ij}} }^2 \Big] &&\mbox{by $\pdv{\bd1_T}{z_{ij}} = 0$}\\
            =& \E \Big[\fnorm{\bF}^2 \sum_{i=1}^n \sum_{j=1}^{p}
            \fnorm*{\frac{\partial \bH}{\partial x_{ij}} }^2 + \sigma^2 \fnorm{\bF}^2 \sum_{i=1}^n \fnorm*{\pdv{\bH}{\eps_i}}^2\Big] \\
            =& \E \Big[\fnorm{\bF}^2 
            \fnorm*{\pdv{\vec(\bH)}{\vec(\bX)}}^2 + \sigma^2 \fnorm{\bF}^2 \fnorm*{\pdv{\vec(\bH)}{\bep}}^2\Big] \\
            \le& C(\zeta, T) \E \Big[\fnorm{\bF}^2  (1 + \opnorm{\bX}^2/n)^{2T} 
            \bigl(\fnorm{\bF n^{-1/2}} \sqrt\gamma + \fnorm{\bH}\bigr)^2\Big] &&\mbox{by \eqref{eq:frob-norm-bound-H-vec}}\\
            &+ C(\zeta, T) \E \Big[\sigma^2 \fnorm{\bF}^2  (1 + \opnorm{\bX}^2/n)^{2T} \Big]
            &&\mbox{by \eqref{eq:opnorm-dH/de}}\\
            \le& n C(\zeta, T, \gamma) \var(y_1)^2 &&\mbox{by \Cref{lem:moment-bound-F-H,lem:moment-bound-X}.}
        \end{align*}
        Combining the above three bounds gives the desired bound for \eqref{UZV-RHS}.
        This concludes the proof.
    \end{proof}
    
    \begin{lemma}[Proof is given on \Cpageref{proof:A1-Rems}]
        \label{lem:A1-Rems}
        Under the same conditions of \Cref{prop:A1}, we have
            \begin{align}
                \bH^\top \sum_{i=1}^n \sum_{j=1}^p \be_j\be_i^\top\frac{\partial \bF}{\partial x_{ij}} 
                &= - \bH^\top \bH (n\bI_T -\hbA^\top) + \Rem_{1,1} \label{eq:A1-1},\\
                \sum_{i=1}^n \sum_{j=1}^p \bF^\top \be_i \be_j^\top \frac{\partial \bH}{\partial x_{ij}}
                &= \bF^\top \bF \hbC^\top/n + \Rem_{1,2},\label{eq:A1-2}
            \end{align}
            where 
            \begin{align}
                \Rem_{1,1} 
                &= \bH^\top \sum_{i=1}^n \sum_{j=1}^p\be_j \sum_t \Delta_{ij}^{it} \be_t^\top,\label{eq:Rem11}\\
                \Rem_{1,2} 
                &= - n^{-1}\sum_{i=1}^n \sum_{j=1}^p \sum_{t=1}^T \bF^\top \be_i \be_j^\top  (\be_t^\top \otimes \bI_p) \calM^{-1}\calD
                \bigl((\bH^\top \be_j)  \otimes (\bX^\top\be_i)\bigr)
                \be_t^\top. \label{eq:Rem12}
            \end{align}
            In addition, we have 
            \begin{equation}\label{eq:moment-bound-Rems}
                \begin{aligned}
                    \E[\fnorm{\Rem_{1,1}}^2] 
                    &\le n C(\zeta, T, \gamma) \var(y_1)^2,\\
                    \E[\fnorm{\Rem_{1,2}}^2] 
                    &\le n C(\zeta, T, \gamma) \var(y_1)^2.
                \end{aligned}
            \end{equation}
            
    \end{lemma}

    \begin{proof}[Proof of \Cref{lem:Chi2type}]
        \label{proof:lem:Chi2type}
        For each $j\in[p]$, let $\E_j [\cdot] $ denote the conditional expectation $\E[\cdot |\{\bZ \be_k, k\ne j\}]$. The left-hand side of the desired inequality can be rewritten as 
        \begin{align*}
            &\E\Big[
            \fnorm[\big]{ 
                p \bU^\top \bV  - \sum_{j=1}^p 
                (\E_j[\bU]^\top \bZ - \bL^\top) \be_j
                \be_j^\top (\bZ^\top \E_j[\bV] - \hat \bL)
            }
            \Big] 
        \end{align*}
        with $\bL\in\R^{p\times T}$ defined by
        $\bL^\top \be_j = \E_j[\bU]^\top \bZ \be_j - \bU^\top \bZ \be_j + \sum_{i=1}^n \partial_{ij} \bU^\top \be_i$
        and $\hat \bL$ defined similarly with $\bU$ replaced by $\bV $.
        We develop the terms in the sum over $j$ as follows:
        \begin{align}
            \nonumber
            &p \bU^\top \bV  - \sum_j
            (\E_j[\bU]^\top \bZ - \bL^\top) \be_j
            \be_j^\top (\bZ^\top \E_j[\bV] - \hat \bL)
            \\=\quad&
            \sum_j
            \Bigl(\bU^\top \bV  -  \E_j[\bU]^\top\E_j[\bV]\Bigr)
            \label{eq:first-term-swap-trick}
            \\&  + \sum_j
            \Bigl(
            \E_j[\bU]^\top\E_j[\bV]
            -
            \E_j[\bU]^\top \bZ \be_j \be_j^\top \bZ^\top\E_j[\bV]
            \Bigr)
            \label{eq:difficult-term}
            \\ & - \bL^\top \hat \bL \label{eq:Xi-hat-Xi-term}
            \\ & + \sum_j \Bigl(\E_j[\bU]^\top \bZ \be_j \be_j^\top \hat \bL\Bigr) 
            + 
            \Bigl(\bL^\top \be_j \be_j^\top \bZ^\top \E_j[\bV]\Bigr).
            \label{eq:last-two-terms}
        \end{align}
        The following proof is dedicated to bounding the expectation of Frobenius norm of each term in \eqref{eq:first-term-swap-trick}-\eqref{eq:last-two-terms}.
        We now bound \eqref{eq:Xi-hat-Xi-term},  \eqref{eq:last-two-terms}, \eqref{eq:first-term-swap-trick}, and \eqref{eq:difficult-term} in order. 
    
        \textbf{For \eqref{eq:Xi-hat-Xi-term}}.
        We have by the Cauchy-Schwarz inequality
        $\E\bigl[ \|\bL^\top \hat \bL\|_{\rm F}\bigr]
        \le \E\bigl[\|\bL\|_{\rm F}^2\bigr]^{\frac 12}\E\bigl[\|\hat\bL\|_{\rm F}^2\bigr]^{\frac 12}$.
        For a fixed $j\in[p]$ and $t\in [T]$, 
        \begin{align*}
            \E[(\be_j^\top \bL \be_t)^2]
            &\le
            \sum_{i=1}^n
            \E[(\be_i^\top(\E_j[\bU] - \bU)\be_t )^2]
            +
            \E
            \sum_{i=1}^n
            \sum_{l=1}^n
            \Bigl(
            \frac{\be_i^\top \partial \bU \be_t }{\partial z_{lj}}
            \Bigr)^2\\
            &\le
            2
            \E
            \sum_{i=1}^n
            \sum_{l=1}^n
            \Bigl(
            \frac{\be_i^\top \partial \bU \be_t }{\partial z_{lj}}
            \Bigr)^2,
        \end{align*}
        where the two inequalities are due to the second-order stein inequality in \cite{bellec2021second}, and Gaussian-Poincar\'e inequality in \cite[Theorem 3.20]{boucheron2013concentration}, respectively. 
        Summing over $j\in[p]$ and $t\in[T]$ we obtain
        $\E[\|\bL\|_{\rm F}^2] \le 2 \E \sum_{lj} \| \partial_{lj} \bU \|_{\rm F}^2
        = 2 \E [\|\bU\|_\partial^2]$.
        Combined with the same bound for $\hat \bL$, we obtain
        $$\E[\|\eqref{eq:Xi-hat-Xi-term}\|_{\rm F}] 
        \le
        2 \E[\|\bU\|^2_\partial]^{1/2} \E[\|\bV \|_\partial^2]^{1/2}.$$
        
        \textbf{For \eqref{eq:last-two-terms}}.
        We focus on the first term in \eqref{eq:last-two-terms}; the similar bound applies to the second term by exchanging the role of $\bU$ and $\bV$. For the first term, we have 
        \begin{align*}
            \E\Bigl[\|\sum_j \E_j[\bU]^\top \bZ \be_j \be_j^\top \hat \bL\|_{\rm F}\Bigr]
            &\le
            \sum_j \E\Bigl[\|\E_j[\bU]^\top \bZ \be_j\|_2 \|\be_j^\top \hat \bL\|_2\Bigr]
            \\ &\le \E[
            \sum_j \| \E_j[\bU]^\top \bZ \be_j\|_2^2
            ]^{\frac 12}
            \E[ \sum_j \|\be_j^\top \hat \bL\|_2^2]^{\frac 12}
            \\&\le (p \E[\fnorm{\bU}^2])^{\frac 12}
            \E[ \|\hat \bL\|_{\rm F}^2]^{\frac 12}
            \\&\le (p \E[\fnorm{\bU}^2])^{\frac 12}
             \sqrt{2} \E[\|\bV\|_{\partial}^2]^{1/2},
        \end{align*}
        where we used that $\| \ba \bb^\top \|_{\rm F} = \|\ba\|_2\|\bb\|_2$ for two vectors $\ba,\bb$,
        the Cauchy-Schwarz inequality, $\E[\|\bA \bz_j\|_2^2|\bA] = \|\bA\|_{\rm F}^2$
        if matrix $\bA$ is independent of $\bz_j\sim \mathsf{N}(0,\bI_n)$ (set $\bz_j=\bZ \be_j$),
        and Jensen's inequality.
        By symmetry of the two terms in \eqref{eq:last-two-terms}, we have 
        \begin{align*}
            &\E\|\eqref{eq:last-two-terms}\|_{\rm F}  
            \le \sqrt{2p}\E[\fnorm{\bU}^2]^{\frac 12}
             \E[\|\bV\|_{\partial}^2]^{1/2} +  \sqrt{2p}\E[\fnorm{\bV}^2]^{\frac 12}
             \E[\|\bU\|_{\partial}^2]^{1/2}.
        \end{align*}
        
        \textbf{For \eqref{eq:first-term-swap-trick}}.
        We decompose \eqref{eq:first-term-swap-trick} as
        $\sum_j \bU^\top(\bV  - \E_j[\bV] )
        +
        \sum_j(\bU - \E_j[\bU])^\top\E_j[\bV] 
        $. 
        We focus on the left term; similar bound will apply to the second term by exchanging the roles of $\bV $ and $\bU$.
        For the first term, we have by the submultiplicativity of the Frobenius norm
        and the Cauchy-Schwarz inequality
        \begin{align*}
            \E[\|\bU^\top\sum_{j}(\bV  - \E_j[\bV])\|_{\rm F}]
            & \le \E[\sum_j\fnorm{\bU} \| \bV  - \E_j[\bV] \|_{\rm F}]
            \\&\le \E[p \fnorm{\bU}^2]^{\frac 12} 
            \E[\sum_j\| \bV  - \E_j[\bV] \|_{\rm F}^2]^{\frac 12}.
        \end{align*}
        By the Gaussian Poincar\'e inequality applied $p$ times,
        $\E[\sum_j\| \bV  - \E_j[\bV] \|_{\rm F}^2] \le \E[\|\bV  \|_{\partial}^2]$,
        so that the previous display is bounded from above by 
        $\sqrt{p} \E[\fnorm{\bU}^2]^{\frac 12}\E[\norm{\bV}^2_{\partial}]^{\frac 12}$
        .
        Using the same argument, we have 
        \begin{align*}
            \E[\fnorm{\sum_j(\bU - \E_j[\bU])^\top\E_j[\bV]} ]
            \le \sqrt{p} \E[\fnorm{\bV}^2]^{\frac 12}\E[\norm{\bU}^2_{\partial}]^{\frac 12}.
        \end{align*}
        Hence, 
        \begin{align*}
            \E[\|\eqref{eq:first-term-swap-trick}\|_{\rm F}] 
            \le \sqrt p \E[\|\bU\|_{\rm F} ^2]^{1/2}
            \E[\norm{\bV}^2_{\partial}]^{\frac 12}
            + \sqrt p \E[\|\bV\|_{\rm F} ^2]^{1/2} \E[\norm{\bU}^2_{\partial}]^{\frac 12}.
        \end{align*}
    
        \textbf{For \eqref{eq:difficult-term}}.
        We first use 
        $\E[\| \eqref{eq:difficult-term} \|_{\rm F}]
        \le
        \E[\| \eqref{eq:difficult-term} \|_{\rm F}^2]^{\frac 12}
        $ by Jensen's inequality and now proceed to bound
        $\E\| \eqref{eq:difficult-term} \|_{\rm F}^2$. We have
        $$
        \| \eqref{eq:difficult-term} \|_{\rm F}^2
        =
        \|\sum_j \E_j[\bU]^\top \E_j[\bV] - \E_j[\bU]^\top 
        \bZ
        \be_j \be_j^\top \bZ^\top \E_j[\bV]\|_{\rm F}^2
        =
        \sum_{j,k} \trace[\bM_j^\top \bM_k],
        $$
        where $\bM_j = \E_j[\bU]^\top \E_j[\bV] - \E_j[\bU]^\top \bZ \be_j \be_j^\top \bZ^\top \E_j[\bV]$. 
        For the summation $\sum_{j,k}$, we split it into two cases $\sum_{j=k}$ and $\sum_{j\ne k}$. 
        We first bound $\E [\sum_j \|\bM_j\|_{\rm F}^2]$.
        Since the variance of $\ba^\top\bb - \ba^\top \bg \bg^\top \bb$ for standard normal $\bg\sim \mathsf{N}(0,I_p)$
        is $2\|(\ba \bb^\top + \bb \ba^\top)/2\|_{\rm F}^2\le 2 \|\ba\|_2^2 \|\bb\|_2^2$, 
        applying this variance bound on each pair of coordinates $(t,t')\in[T]\times [T]$ 
        gives $\sum_j \|\bM_j\|_{\rm F}^2 \le 
        \sum_j 2 \|\E_j[\bU]\|_{\rm F}^2\|\E_j[\bV]\|_{\rm F}^2
        $.
        Hence, using the Cauchy-Schwarz inequality and Jensen's inequality, we have 
        \begin{align*}
            \E [\sum_j \|\bM_j\|_{\rm F}^2] 
            \le 2
            \sum_j \E [\|\E_j[\bU]\|_{\rm F}^2\|\E_j[\bV]\|_{\rm F}^2]
            \le 2 p \E [\|\bU\|_{\rm F}^4]^{1/2} \E [\|\bV\|_{\rm F}^4]^{1/2} .
        \end{align*}
        We now bound $\sum_{j\ne k} \trace[\bM_j^\top\bM_k]$.
        Setting $\bz_j = \bZ \be_j \sim \mathsf{N}(0,\bI_n)$ for every $j\in[p]$,
        we will use many times
        the identity
        \begin{equation}
            \label{stein}
            \E[(\bz_j^\top f(\bZ)  - \sum_i \partial_{ij} f(\bZ)^\top \be_i )g(\bZ)
            =
            \E[
            \sum_{i}
            f(\bZ)^\top \be_i \partial_{ij} g(\bZ)
            ]
        \end{equation}
        which follows from Stein's formula
        for $f:\R^{n\times p}\to \R^n$ and $g:\R^{n\times p}\to \R$.
        With $f^{tt'}(\bZ) =(\bz_j^\top \E_j[\bU] \be_{t'}) \E_j[\bV] \be_t$
        and $g^{tt'}(\bZ) = \be_{t'}^\top \bM_k \be_t$,
        we find
        \begin{align*}
            &\E\trace[\bM_j^\top \bM_k]
            =\E\trace [\bM_j^\top \sum_{t'} \be_{t'} \be_{t'}^\top \bM_k]
            =\E[\sum_{tt'} \be_t^\top  \bM_j^\top \be_{t'} \be_{t'}^\top \bM_k \be_t]
            \\&=\E
            \sum_{tt'}
            \Bigl(\bz_j^\top f^{tt'}(\bZ) - \sum_i \be_i^\top \partial_{ij} f^{tt'}(\bZ)\Bigr)
            g^{tt'}(\bZ)
            \\&
            =\E
            \sum_{tt'}\sum_i
            \be_i^\top f^{tt'}(\bZ) \partial_{ij} g^{tt'}(\bZ).
        \end{align*}
        where $g_{tt'}(\bZ) = 
        (
        \be_t^\top \E_k \bV  ^\top \E_k \bV  \be_t' 
        -
        \be_t^\top \E_k   \bU^\top \bz_k \bz_k^\top \E_k \bV  \be_t'
        )$ and
        $$\partial_{ij} g_{tt'} =
        \partial_{ij}
        \be_{t'}^\top \bM_k \be_t
        =
        \be_{t'}^\top \partial_{ij}[\E_k \bU ^\top \E_k \bV ] \be_t 
        -
        \bz_k^\top\partial_{ij}[\E_k \bU \be_{t'} \be_{t}^\top \E_k  \bU^\top] \bz_k.
        $$
        Now define   $\tilde f^{tt'}(\bZ) =
        \partial_{ij}[\E_k \bU \be_{t'} \be_{t}^\top \E_k   \bU^\top] \bz_k$
        and $\tilde g^{tt'}(\bZ) = \sum_i \be_i^\top f^{tt'}(\bZ)$.
        By definition of $\tilde f^{tt'}(\bZ)$,
        the previous display is equal to
        $\bz_k^\top\tilde f^{tt'}(\bZ) - \sum_l \partial_{lk} \be_l^\top\tilde f^{tt'}(\bZ)$.
        We apply \eqref{stein} again with respect to $\bz_k$, so that
        \begin{align*}
            \E \trace[\bM_j^\top \bM_k]
            &=
            \sum_{il, tt'}
            \be_i^\top \partial_{lk}[f^{tt'}(\bZ)]
            \be_l^\top \tilde f^{tt'}(\bZ)
            \\&=
            \sum_{il, tt'}
            \Bigl(\be_i^\top \partial_{lk}\Bigl[\E_j[\bV] \be_t \be_{t'}^\top \E_j[\bU]^\top\Bigr]\bz_j\Bigr)
            \Bigl(\be_l^\top \partial_{ij}\Bigl[\E_k[\bU]      \be_{t'} \be_t^\top \E_k[\bV ]^\top\Bigr] \bz_k\Bigr).
        \end{align*}
        To remove the indices $t,t'$, we rewrite the above using $\sum_t \be_t \be_t^\top=I_T$
        and $\sum_{t'} \be_{t'} \be_{t'}^\top = I_T$ so that it equals
        $$
        \E\sum_{il}\trace
        \Bigl\{
        \partial_{lk}\Bigl[
        \E_j[\bU]^\top \bz_j \be_i^\top \E_j[\bV] 
        \Bigr]
        \partial_{ij}\Bigl[
        \E_k[\bV ]^\top \bz_k \be_l^\top \E_k[\bU] 
        \Bigr]
        \Bigr\}.
        $$
        Summing over $j,k$,
        using $\trace [\bA^\top \bB ] \le \|\bA\|_{\rm F}\|\bB\|_{\rm F})$
        and the Cauchy-Schwarz inequality,
        the above is bounded from above by
        \begin{align}
            \label{eq:chi-a}
            &
            \Bigl\{
            \E\sum_{jk, il}
            \Big\|
            \partial_{lk}
            \Bigl[
            \E_j[\bU]^\top
            \bz_j
            \be_i^\top
            \E_j
            [\bV ]
            \Bigr]
            \Big\|_{\rm F}^2
            \Bigl\}^{\frac 12}
            \Bigl\{
            \E
            \sum_{jk, il}
            \Big\|
            \partial_{ij}\Bigl[
            \E_k[\bV ]^\top \bz_k
            \be_l^\top
            \E_k[\bU]  
            \Bigr]
            \Big\|_{\rm F}^2
            \Bigr\}^{\frac 12}.
        \end{align}
        At this point the two factors are symmetric,
        with $(\bV , \bU)$ in the left factor
        replaced with $(\bU,\bV )$ on the right factor.
        We focus on the left factor; similar bound will
        apply to the right one by exchanging the roles of $\bV $ and $\bU$.
        If $\bz_j$ is independent of matrices $\bA^{(q)}$, then
        $\E_j[\|\sum_{q=1}^n (\be_q^\top\bz_j) \bA^{(q)}\|_{\rm F}^2]
        =
        \sum_{q=1}^n
        \|\bA^{(q)}\|_{\rm F}^2
        $. So that
        with $A^{(q)} =\partial_{lk}[\E_j[\bU]^\top \be_q \be_i^\top \E_j[\bU]]$, the
        first factor in the above display is equal to
        \begin{align*}
            &\quad~\Bigl\{
            \E\sum_{jk, ilq}
            \Big\|
            \partial_{lk}
            \Bigl(
            \E_j[\bU]^\top
            \be_q
            \be_i^\top
            \E_j[\bV]
            \Bigr)
            \Big\|_{\rm F}^2
            \Bigl\}^{\frac 12}
            \\&\stackrel{\text{(i)}}{=}
            \Bigl\{
            \E\sum_{jk, ilq}
            \Big\|
            \partial_{lk}
            \Bigl(
            \E_j[\bU]^\top
            \Bigr)
            \be_q
            \be_i^\top
            \E_j
            [\bV ]
            +
            \E_j 
            [\bU] ^\top
            \be_q
            \be_i^\top
            \partial_{lk}
            \Bigl(
            \E_j
            [\bV ]
            \Bigr)
            \Big\|_{\rm F}^2
            \Bigl\}^{\frac 12}
            \\&\stackrel{\text{(ii)}}{\le}
            \Bigl\{
            \E\sum_{jk, ilq}
            \Big\|
            \partial_{lk}
            \Bigl(
            \E_j 
            [\bU]^\top
            \Bigr)
            \be_q
            \be_i^\top
            \E_j
            [\bV ]
            \Big\|_{\rm F}^2
            \Bigl\}^{\frac 12}
            +
            \Bigl\{
            \E\sum_{jk, ilq}
            \Big\|
            \E_j[\bU]^\top
            \be_q
            \be_i^\top
            \partial_{lk}
            \Bigl(
            \E_j
            [\bV ]
            \Bigr)
            \Big\|_{\rm F}^2
            \Bigl\}^{\frac 12}
            \\&\stackrel{\text{(iii)}}{=}
            \Bigl\{
            \E\sum_{jk, ilq}
            \Big\|
            \E_j [ \partial_{lk} \bU]^\top
            \be_q
            \Big\|_2^2
            \Big\|
            \be_i^\top
            \E_j[\bV]
            \Big\|_2^2
            \Bigl\}^{\frac 12}
            +
            \Bigl\{
            \E\sum_{jk, ilq}
            \Big\|
            \E_j[\bU]^\top
            \be_q
            \Big\|_2^2
            \Big\|
            \be_i^\top
            \E_j
            [ \partial_{lk} \bV ]
            \Big\|_2^2
            \Bigl\}^{\frac 12}
            \\&\stackrel{\text{(iv)}}{=}
            \Bigl\{
            \E\sum_{jk, l}
            \Big\|
            \E_j [ \partial_{lk} \bU]^\top
            \Big\|_{\rm F}^2
            \Big\|
            \E_j[\bV]
            \Big\|_{\rm F}^2
            \Bigl\}^{\frac 12}
            +
            \Bigl\{
            \E\sum_{jk, l}
            \Big\|
            \E_j[\bU]^\top
            \Big\|_{\rm F}^2
            \Big\|
            \E_j
            [ \partial_{lk} \bV ]
            \Big\|_{\rm F}^2
            \Bigl\}^{\frac 12},
        \end{align*}
        where (i) is the chain rule,
        (ii) the triangle inequality,
        (iii) holds
        provided that the order of the derivation $\partial_{lk}$
        and the expectation sign $\E_j$ can be switched (thanks to dominant convergence theorem) and using
        $\|\ba \bb^\top\|_{\rm F}^2 = \|\ba\|_2^2 \|\bb\|_2^2$ for vectors $\ba, \bb$,
        and (iv) holds using
        $\sum_i \|\bA \be_i\|_2^2 = \|\bA\|_{\rm F}^2 = \sum_q \|\bA \be_q\|_2^2$ for a matrix $\bA$ with $n$ columns.
        Finally, by Jensen's inequality, the above display is bounded by
        \begin{equation*}
            \Bigl\{
            \E\sum_{k, l}
            \Big\|
            \partial_{lk}
            \bU
            \Big\|_{\rm F}^2
            \sum_j
            \Big\|
            \E_j
            [\bV ]
            \Big\|_{\rm F}^2
            \Bigl\}^{\frac 12}
            +
            \Bigl\{
            \E\sum_{k, l}
            \Big\|
            \partial_{lk}
            \bV 
            \Big\|_{\rm F}^2
            \sum_j
            \Big\|
            \E_j
            [\bU]
            \Big\|_{\rm F}^2
            \Bigl\}^{\frac 12}
            .
        \end{equation*}
    
        By the Cauchy-Schwarz inequality, the first term in the above display is bounded by
        \begin{align*}
            &
            \Bigl\{\sum_j \sqrt{\E \Bigl[\bigl(\sum_{k, l}\| \partial_{lk} \bU \|_{\rm F}^2\bigr)^2 \Bigr]} 
            \sqrt{\E \| \E_j [\bV ] \|_{\rm F}^4}\Bigr\}^{1/2}\\
            \le~& \Bigl\{p \sqrt{\E \Bigl[\bigl(\sum_{k, l}\| \partial_{lk} \bU \|_{\rm F}^2\bigr)^2 \Bigr]} 
            \sqrt{\E \| \bV\|_{\rm F}^4 }\Bigr\}^{1/2}\\
            =~& \sqrt{p} 
            \E [\|\bU \|_{\partial}^4] ^{1/4} \E [\|\bV\|_{\rm F}^4]^{1/4},
        \end{align*}
        where we used 
        $\|\E_j [\bV]\|_{\rm F} ^4 \le \E_j [\|\bV\|_{\rm F} ^4]$ by Jensen's inequality. 
        Hence, 
        \begin{align*}
            \eqref{eq:chi-a}
            &\le \big(
                \sqrt{p} 
            \E [\|\bU \|_{\partial}^4] ^{1/4} \E [\|\bV\|_{\rm F}^4]^{1/4}
             +  \E [\|\bV \|_{\partial}^4] ^{1/4} \E [\|\bU\|_{\rm F}^4]^{1/4}
            \big)^2
            \\
            &\le 2p 
            \E [\|\bU \|_{\partial}^4] ^{1/2} \E [\|\bV\|_{\rm F}^4]^{1/2}
             +  2p \E [\|\bV \|_{\partial}^4] ^{1/2} \E [\|\bU\|_{\rm F}^4]^{1/2},
        \end{align*}
        where we uses $(a+b)^2 \le 2a^2 + 2b^2$. 
    
        Combining the above bounds of
        $\sum_{j=k} \trace[\bM_j^\top \bM_k]$ and $\sum_{j\ne k} \trace[\bM_j^\top \bM_k]$, we have 
        \begin{align*}
            &\E[
        \|\eqref{eq:difficult-term}
        \|_{\rm F}
        ]\\
        \le &\E[
        \|\eqref{eq:difficult-term}
        \|_{\rm F}^2]^{\frac 12}
        \\
        \le &\big(2p \E [\|\bU\|_{\rm F} ^4]^{1/2} \E [\|\bV\|_{\rm F} ^4]^{1/2} + 2p 
        \E [\|\bU \|_{\partial}^4] ^{1/2} \E [\|\bV\|_{\rm F}^4]^{1/2}
         +  2p \E [\|\bV \|_{\partial}^4] ^{1/2} \E [\|\bU\|_{\rm F}^4]^{1/2}\big)^{1/2} \\
        \le & \sqrt{2p} \big( \E [\|\bU\|_{\rm F} ^4]^{1/4} \E [\|\bV\|_{\rm F} ^4]^{1/4} 
        +
        \E [\|\bU \|_{\partial}^4] ^{1/4} \E [\|\bV\|_{\rm F}^4]^{1/4}
         + \E [\|\bV \|_{\partial}^4] ^{1/4} \E [\|\bU\|_{\rm F}^4]^{1/4} \big),
        \end{align*}
        where the last inequality uses 
        $(a + b + c)^{1/2} \le a^{1/2} + b^{1/2} + c^{1/2}$ for non-negative $a$, $b$, and $c$.
    
        Now, combining the bounds on the terms \eqref{eq:first-term-swap-trick}-\eqref{eq:difficult-term}-\eqref{eq:Xi-hat-Xi-term}-\eqref{eq:last-two-terms}
        with the triangle inequality, the expectation in the lemma statement can be bounded from above by
        \begin{align*}
            &\E\Bigl[
            \fnorm[\Big]{
                p \bU^\top \bV - \sum_{j=1}^p 
                \Bigl(\sum_{i=1}^n \partial_{ij} \bU^\top \be_i - \bU^\top \bZ \be_j\Bigr)
                \Bigl(\sum_{i=1}^n \partial_{ij} \be_i^\top \bV  - \be_j^\top \bZ^\top \bV\Bigr)
            }
            \Bigr] \\
            \le &2 \E[\|\bU\|_\partial^2]^{1/2}\E[\|\bV\|_\partial^2]^{1/2}\\
            &+ (1 + \sqrt{2})\sqrt{p}
            \bigl(
                \E[\fnorm{\bU}^2]^{\frac 12}
                \E[\|\bV\|_\partial^2]^{1/2}+  \E[\fnorm{\bV}^2]^{\frac 12}
                \E[\|\bU\|_\partial^2]^{1/2}
            \bigr) \\
     &+ \sqrt{2p} \big( \E [\|\bU\|_{\rm F} ^4]^{1/4} \E [\|\bV\|_{\rm F} ^4]^{1/4} 
     +
     \E [\|\bU \|_{\partial}^4] ^{1/4} \E [\|\bV\|_{\rm F}^4]^{1/4}
      + \E [\|\bV \|_{\partial}^4] ^{1/4} \E [\|\bU\|_{\rm F}^4]^{1/4} \big)\\
      \le & \E[\|\bU\|_\partial^2] + \E[\|\bV\|_\partial^2]\\
      &+ (1 + \sqrt{2})\sqrt{p}/2
      \bigl(
          \E[\fnorm{\bU}^2] + 
          \E[\|\bV\|_\partial^2]+  \E[\fnorm{\bV}^2] +
          \E[\|\bU\|_\partial^2]
      \bigr) \\
    &+ \sqrt{2p}/2 \big( \E [\|\bU\|_{\rm F} ^4]^{1/2} +\E [\|\bV\|_{\rm F} ^4]^{1/2} 
    +
    \E [\|\bU \|_{\partial}^4] ^{1/2} +\E [\|\bV\|_{\rm F}^4]^{1/2}
    + \E [\|\bV \|_{\partial}^4] ^{1/2}+ \E [\|\bU\|_{\rm F}^4]^{1/2} \big)\\
    \le& (1 + 2\sqrt{p}) \bigl(\E[\fnorm{\bU}^4]^{1/2} + 
    \E[\|\bV\|_\partial^4]^{1/2}+  \E[\fnorm{\bV}^4]^{1/2} +
    \E[\|\bU\|_\partial^4]^{1/2}\bigr),
    \end{align*}
    where the second inequality uses $ab \le (a^2+b^2)/2$ and the last inequality uses 
    $$
    \E[\fnorm{\bU}^2]^{1/2} \le \E [\fnorm{\bU}^4]^{1/4} \text{ and }
    \E [\norm{\bU}^2_\partial]^{1/2} \le \E [\norm{\bU}^4_\partial]^{1/4}
    $$
    from Jensen's inequality. 
    This completes the proof of \Cref{lem:Chi2type}. 
    
\end{proof}

\begin{proof}[Proof of \Cref{lem:A2-Rem}]
\label{proof:A2-Rem}
For notation simplicity, we define
\begin{align*}
    \brho_i
    &= \sum_{j=1}^{p+1}  \pdv{\bV^\top}{z_{ij}} \be_j - \bV^\top \bZ^\top \be_i.
\end{align*}
Then the left-hand side of \eqref{A2-LHS} can be written as
$n \bV^\top \bV - \sum_{i=1}^n \brho_i \brho_i^\top$. 
Since 
$\bV = [-\bH^\top, \sigma\bd1_T]^\top$, we have $\bV^\top \bV = \bH^\top \bH + \bS$. 
We need the following lemma for this proof. 
\begin{lemma}[Proof is given on \Cpageref{proof:Wi}]
    \label{lem:Wi}
    We have 
    \begin{equation}
        \brho_i^\top = \sum_{j=1}^{p+1} \pdv{\be_j^\top \bV}{z_{ij}}  - \be_i^\top \bZ \bV 
    = -\be_i^\top \bF (\bI_T + \hbC/n)^\top + \be_i^\top \bXi,
    \end{equation}
    where $\bXi = n^{-1} 
    \sum_{j=1}^p 
    \bX
    \big((\be_j^\top \bH) \otimes \bI_p\big) (\calM^{-1}\calD)^\top (\bI_T \otimes \be_j)$ and it satisfies 
    $$\opnorm{\bXi}
    \le C(\zeta, T, \gamma) \fnorm{\bH} (1 + \opnorm{\bX}^2/n)^{T}.$$ 
\end{lemma}
By \Cref{lem:Wi},
\begin{align*}
    &n\bV^\top \bV - \sum_{i=1}^n \brho_i \brho_i^\top\\
    =& n(\bH^\top \bH + \bS) - 
    \big((\bI_T + \hbC/n) \bF^\top + \bXi^\top\big)
    \big(\bF (\bI_T + \hbC/n)^\top + \bXi \big)\\
    = & n(\bH^\top \bH + \bS) - (\bI_T + \hbC/n) \bF^\top \bF(\bI_T + \hbC/n)^\top\\
     & - \underbrace{\bXi^\top \bF (\bI_T + \hbC/n)^\top - (\bI_T + \hbC/n) \bF^\top \bXi - \bXi^\top \bXi}_{\Rem_2}\\
    =& \sqrt{n}\bQ_2 - \Rem_2.
\end{align*}
The desired bound for $\E[\fnorm{\Rem_2}]$ then follows from the Cauchy-Schwarz inequality, \Cref{lem:moment-bound-F-H,lem:moment-bound-X,lem:opnorm-Chat} and the operator norm bound of $\bXi$ from \Cref{lem:Wi}. 
This completes the proof of \Cref{lem:A2-Rem}.
\end{proof}

\begin{proof}[Proof of \Cref{lem:A2-RHS}]
    \label{proof:A2-RHS}
    We need to bound $\E[\fnorm{\bV}^4]^{1/2}$ and 
    $\E[\norm{\bV}_{\partial}^4]^{1/2}$.
    For the first term, 
    since $\fnorm{\bV}^2 = \fnorm{\bH}^2 + T\sigma^2$,
    we have by \Cref{lem:moment-bound-F-H}, 
    \begin{align*}
        \E [\fnorm{\bV}^4]^{1/2} 
        = \E [(\fnorm{\bH}^2 + T\sigma^2)^2]^{1/2}
        \le 2 \E [\fnorm{\bH}^4]^{1/2} + 2T\sigma^2
        \le C(\zeta, T, \gamma) \var(y_1).
    \end{align*}
    For the second term, we have
    \begin{align*}
    \norm{\bV}_{\partial}^2
    = \sum_{i=1}^{n}\sum_{j=1}^{p+1}
    \fnorm{\pdv{\bV}{z_{ij}}}^2
    = \sum_{i=1}^{n}\sum_{j=1}^{p}\fnorm{\pdv{\bH}{x_{ij}}}^2
    = \fnorm{\pdv{\vec(\bH)}{\vec(\bX)}}.
    \end{align*}
    Applying \eqref{eq:frob-norm-bound-H-vec}, and the moment bound of $\bF, \bH, \bX$ in \Cref{lem:moment-bound-F-H,lem:moment-bound-X}, we have
    \begin{align*}
        \E [\norm{\bV}_{\partial}^4]^{1/2} 
        \le C(\zeta, T, \gamma) \var(y_1).
    \end{align*}
    Combining the above two bounds, we have the desired inequality. 
    This completes the proof of \Cref{lem:A2-RHS}.
\end{proof}

\begin{proof}[Proof of \Cref{lem:A1-Rems}]
    \label{proof:A1-Rems}
    By \Cref{lem:dot-F-4}, we have 
    $\frac{\partial F_{lt}}{\partial x_{ij}} = D_{ij}^{lt} + \Delta_{ij}^{lt}$, where 
   \begin{align*}
       D_{ij}^{lt} 
       &= -(\be_t^\top\otimes \be_l^\top) (\bI_{nT} - \bN) ((\bH^\top \be_j)  \otimes \be_i),\\
       \Delta_{ij}^{lt} 
       &= -n^{-1} (\be_t^\top \otimes \be_l^\top)(\bI_T\otimes \bX)
       \calM^{-1} \calD \bigl(\bF^\top \otimes \bI_p\bigr)(\be_i \otimes \be_j).
   \end{align*}
   
   We first prove \eqref{eq:A1-1}.
   Since $\be_i^\top \frac{\partial \bF}{\partial x_{ij} }
   = \sum_t \pdv{F_{it}}{x_{ij}} \be_t^\top
   =\sum_t (D_{ij}^{it} + \Delta_{ij}^{it}) \be_t^\top$, we have
   \begin{align*}
       &\bH^\top \sum_{i=1}^n \sum_{j=1}^p\be_j\be_i^\top\frac{\partial \bF}{\partial x_{ij}}\\
       =~& \bH^\top \sum_{ij}\be_j \sum_{t=1}^T (D_{ij}^{it} + \Delta_{ij}^{it}) \be_t^\top\\
       =~& \bH^\top \sum_{ij}
       \be_j \sum_{t} D_{ij}^{it} \be_t^\top 
       + \underbrace{\bH^\top \sum_{ij} \be_j \sum_{t} \Delta_{ij}^{it} \be_t^\top}_{\Rem_{1,1}}.
   \end{align*}
   The first term in the last line can be simplified as below,
   \begin{align*}
       &\bH^\top \sum_{j=1}^p\sum_{i=1}^n
       \be_j \sum_{t=1}^T D_{ij}^{it} \be_t^\top\\
       =~&-\bH^\top \sum_{j=1}^p\sum_{i=1}^n
       \be_j \sum_{t=1}^T(\be_j^\top\bH \otimes \be_i^\top)(\bI_{nT} - \bN^\top)(\be_t\otimes \be_i)\be_t^\top &&\mbox{by \eqref{eq:D-ij-lt}}\\
       =~& -
       \bH^\top  \bH
       \Bigl[
       \sum_{i=1}^n
       (
       \bI_T \otimes \be_i^\top
       )
       (\bI_{nT} - \bN^\top)
       (\bI_T \otimes \be_i)
       \Bigr]\\
       =~& -
       \bH^\top \bH(n\bI_T - \hbA^\top) &&\mbox{by \eqref{eq:hat-A}.}
   \end{align*}

   Next, we prove \eqref{eq:A1-2}. 
   \begin{align*}
       &\sum_{i=1}^n \sum_{j=1}^p \bF^\top \be_i \be_j^\top \frac{\partial \bH}{\partial x_{ij}}\\
       =& \sum_{i=1}^n \sum_{j=1}^p \sum_{t=1}^T \bF^\top \be_i \be_j^\top \frac{\partial \hbb^t}{\partial x_{ij}}\be_t^\top \\
       =& n^{-1} \sum_{i=1}^n \sum_{j=1}^p \sum_{t=1}^T \bF^\top \be_i \be_j^\top (\be_t^\top \otimes \bI_p) \calM^{-1}\calD
       \Bigl((\bF^\top \be_i) \otimes \be_j\Bigr) \be_t^\top + \Rem_{1,2}&&\mbox{by \eqref{eq:pdv-bt-x}}\\
       =& n^{-1} \sum_{i=1}^n \sum_{j=1}^p \sum_{t=1}^T \bF^\top \be_i \Big[ (\be_t^\top \otimes \be_j^\top) \calM^{-1}\calD
       \Bigl((\bF^\top \be_i) \otimes \be_j\Bigr)\Big]^\top \be_t^\top + \Rem_{1,2}\\
       =& n^{-1} \sum_{i=1}^n \sum_{j=1}^p \sum_{t=1}^T \bF^\top \be_i
       \Bigl((\be_i^\top \bF) \otimes \be_j^\top \Bigr) (\calM^{-1}\calD)^\top (\be_t \otimes \be_j)  \be_t^\top + \Rem_{1,2}&&\mbox{by \eqref{eq:kronecker-mix-product}}\\
       =& n^{-1} \bF^\top \bF \sum_{j=1}^p 
       (\bI_T \otimes \be_j^\top) (\calM^{-1}\calD)^\top (\bI_T \otimes \be_j) + \Rem_{1,2} &&\mbox{by \eqref{eq:kronecker-mix-product}}\\
       =& \bF^\top \bF \hbC^\top/n + \Rem_{1,2} && \mbox{by \eqref{hbC}}. 
   \end{align*}

   It remains to bound $\E[\fnorm{\Rem_{1,1}}^2]$ and 
   $\E[\fnorm{\Rem_{1,2}}^2]$. 
      By expression of $\Rem_{1,1}$ in \eqref{eq:Rem11}, we have 
      \begin{align*}
        &\quad \Rem_{1,1} \\
        &= \bH^\top \sum_{i=1}^n \sum_{j=1}^p\be_j \sum_{t=1}^T \Delta_{ij}^{it} \be_t^\top,\\
        &= - n^{-1} \bH^\top \sum_{i=1}^n \sum_{j=1}^p \sum_{t=1}^T \be_j (\be_t^\top \otimes \be_i^\top)(\bI_T\otimes \bX)
        \calM^{-1} \calD \bigl(\bF^\top \otimes \bI_p\bigr)(\be_i \otimes \be_j) \be_t^\top &&\mbox{by \eqref{eq:Delta-ij-lt}}\\
        &= - n^{-1} \bH^\top \sum_{i=1}^n \sum_{j=1}^p \sum_{t=1}^T \be_j \big[(\be_t^\top \otimes \be_i^\top)(\bI_T\otimes \bX)
        \calM^{-1} \calD \bigl(\bF^\top \otimes \bI_p\bigr)(\be_i \otimes \be_j)\big]^\top \be_t^\top\\
        &= - n^{-1} \bH^\top \sum_{i=1}^n \sum_{j=1}^p \sum_{t=1}^T \be_j 
        (\be_i^\top \otimes \be_j^\top)
        (\bF \otimes \bI_p)
        (\calM^{-1} \calD)^\top 
        (\bI_T\otimes \bX^\top)
        (\be_t \otimes \be_i)\be_t^\top\\
        &= - n^{-1} \bH^\top \sum_{i=1}^n 
        (\be_i^\top \otimes \bI_p)
        (\bF \otimes \bI_p)
        (\calM^{-1} \calD)^\top 
        (\bI_T\otimes \bX^\top)
        (\bI_T \otimes \be_i) &&\mbox{by \eqref{eq:kronecker-mix-product}}\\ 
        &= - n^{-1} \bH^\top \sum_{i=1}^n 
        ((\be_i^\top\bF) \otimes \bI_p)
        (\calM^{-1} \calD)^\top 
        (\bI_T\otimes (\bX^\top\be_i)) &&\mbox{by \eqref{eq:kronecker-mix-product}.}
      \end{align*}
      Using the fact that $\fnorm{\bM}\le \sqrt{T} \opnorm{\bM}$ for $\bM\in\R^{T\times T}$ and the
        triangular inequality, we have
        \begin{align*}
            &\fnorm*{\Rem_{1,1}} \\
            \le& \sqrt{T} n^{-1} \opnorm{\bH} \opnorm{\calM^{-1} \calD}\sum_{i=1}^n 
            \opnorm{(\be_i^\top\bF) \otimes \bI_p}
            \opnorm{\bI_T\otimes (\bX^\top\be_i)}\\
            \le&  \sqrt{T} n^{-1} \opnorm{\bH} \opnorm{\calM^{-1} \calD}\big[\sum_{i=1}^n 
            \opnorm{(\be_i^\top\bF) \otimes \bI_p}^2\big]^{1/2}
            \big[\sum_{i=1}^n \opnorm{\bI_T\otimes (\bX^\top\be_i)}^2\big]^{1/2}\\
            \le&  \sqrt{T} n^{-1}\fnorm{\bH}\fnorm{\bF} \fnorm{\bX} \opnorm{\calM^{-1}} \opnorm{\calD}\\
            \le&  \sqrt{T} n^{-1/2}\fnorm{\bH}\fnorm{\bF} \opnorm{\bX} \opnorm{\calM^{-1}} \opnorm{\calD}&&\mbox{by $\fnorm{\bX} \le \sqrt{n}\opnorm{\bX}$}\\
            \le&  C(\zeta, T, \gamma) \fnorm{\bF}\fnorm{\bH} (1 + \opnorm{\bX}^2/n)^T &&\mbox{by \Cref{lem:opnorm-M-N,lem:opnorm-D-J}}.
        \end{align*}
        It follows by the Cauchy-Schwarz inequality and \Cref{lem:moment-bound-F-H,lem:moment-bound-X} that 
    \begin{align*}
        \E [\fnorm{\Rem_{1,1}}^2]
        \le n C(\zeta, T, \gamma) \var(y_1)^2.
    \end{align*}
    We next bound $\E[\fnorm{\Rem_{1,2}}^2]$. 
    By the expression of $\Rem_{1,2}$ in \eqref{eq:Rem12}, we have 
      \begin{align*}
        -\Rem_{1,2} 
        &= n^{-1}\sum_{i=1}^n \sum_{j=1}^p \sum_{t=1}^T \bF^\top \be_i \be_j^\top  (\be_t^\top \otimes \bI_p) \calM^{-1}\calD
        \bigl((\bH^\top \be_j)  \otimes (\bX^\top\be_i)\bigr)
        \be_t^\top\\
        &= n^{-1}\sum_{i=1}^n \sum_{j=1}^p \sum_{t=1}^T \bF^\top \be_i \Big[ (\be_t^\top \otimes \be_j^\top ) \calM^{-1}\calD
        \bigl((\bH^\top \be_j)  \otimes (\bX^\top\be_i)\bigr)\Big]^\top
        \be_t^\top\\
        &= n^{-1}\sum_{i=1}^n \sum_{j=1}^p \sum_{t=1}^T \bF^\top \be_i 
        \big((\be_j^\top \bH)  \otimes (\be_i^\top \bX)\big) (\calM^{-1}\calD)^\top 
        (\be_t \otimes \be_j) \be_t^\top\\
        &= n^{-1} \sum_{j=1}^p 
        \big((\be_j^\top \bH)  \otimes (\bF^\top \bX)\big) (\calM^{-1}\calD)^\top 
        (\bI_T \otimes \be_j).
      \end{align*}
      Using the same argument that we used to bound $\E[\fnorm{\Rem_{1,1}}^2]$, we have
      \begin{align*}
        \E [\fnorm{\Rem_{1,2}}^2]
        \le n C(\zeta, T, \gamma) \var(y_1)^2. 
      \end{align*}

    This completes the proof of \Cref{lem:A1-Rems}.
\end{proof}

\begin{proof}[Proof of \Cref{lem:Wi}]
    \label{proof:Wi}
    By the definition of $\bV$, we have 
\begin{align*}
    &\sum_{j=1}^{p+1} \pdv{\be_j^\top \bV}{z_{ij}} \\
    =& -\sum_{j=1}^{p}\pdv{\be_j^\top \bH}{x_{ij}} \\
    =& -\sum_{j=1}^{p} \sum_{t=1}^T \pdv{\be_j^\top \hbb^t}{x_{ij}} \be_t^\top\\ 
    =& -n^{-1} \sum_{j=1}^{p} \sum_{t=1}^T (\be_t^\top \otimes \be_j^\top) \calM^{-1}\calD
    \bigl((\bF^\top \be_i) \otimes \be_j\bigr) \be_t^\top + \bDelta_i &&\mbox{by \eqref{eq:pdv-bt-x}}\\
    =& -n^{-1} \sum_{j=1}^{p} \sum_{t=1}^T [\be_t(\be_t^\top \otimes \be_j^\top) \calM^{-1}\calD
    \bigl((\bF^\top \be_i) \otimes \be_j\bigr)]^\top + \bDelta_i\\ 
    =& -n^{-1} \sum_{j=1}^{p}  [(\bI_T \otimes \be_j^\top) \calM^{-1}\calD
    \bigl(\bI_T \otimes \be_j\bigr) \bF^\top \be_i]^\top + \bDelta_i\\
    =& -\be_i^\top \bF \hbC^\top/n + \bDelta_i, 
\end{align*}
where $\bDelta_i 
= n^{-1} \sum_{j=1}^{p} \sum_{t=1}^T \be_j^\top(\be_t^\top \otimes \bI_p) \calM^{-1}\calD
\bigl((\bH^\top \be_j) \otimes (\bX^\top \be_i)\bigr) \be_t^\top.$

We claim that 
\begin{equation}\label{eq:Xi}
    \bDelta_i = \be_i^\top \bXi.
\end{equation}
Therefore, using $\bZ \bV = \bF$, we have 
\begin{align*}
    \brho_i^\top 
    &= \sum_{j=1}^{p+1}  \pdv{\be_j^\top \bV}{z_{ij}}  - \be_i^\top \bZ \bV\\
    &= -\be_i^\top \bF \hbC^\top/n + \bDelta_i - \be_i^\top \bF\\
    &= -\be_i^\top \bF (\bI_T + \hbC/n)^\top  + \be_i^\top \bXi. 
\end{align*}

Now we prove the claim \eqref{eq:Xi}. By definition, 
\begin{align*}
    \bDelta_i 
    &= n^{-1} \sum_{j=1}^{p} \sum_{t=1}^T \be_j^\top(\be_t^\top \otimes \bI_p) \calM^{-1}\calD
    \bigl((\bH^\top \be_j) \otimes (\bX^\top \be_i)\bigr) \be_t^\top\\
    &= n^{-1} \sum_{j=1}^{p} \sum_{t=1}^T (\be_t^\top \otimes \be_j^\top) \calM^{-1}\calD
    \bigl((\bH^\top \be_j) \otimes (\bX^\top \be_i)\bigr) \be_t^\top\\
    &= n^{-1} \sum_{j=1}^{p} \sum_{t=1}^T \big[\be_t (\be_t^\top \otimes \be_j^\top) \calM^{-1}\calD
    \bigl((\bH^\top \be_j) \otimes (\bX^\top \be_i)\bigr)\big]^\top\\
    &= n^{-1} \sum_{j=1}^{p}
    \big[(\bI_T \otimes \be_j^\top) \calM^{-1}\calD
    \bigl((\bH^\top \be_j) \otimes (\bX^\top \be_i)\bigr)\big]^\top\\
    &= n^{-1} \sum_{j=1}^p \big((\be_j^\top \bH) \otimes (\be_i^\top \bX)\big) (\calM^{-1}\calD)^\top (\bI_T \otimes \be_j)\\
    &= n^{-1} \sum_{j=1}^p 
    \be_i^\top \bX
    \big((\be_j^\top \bH) \otimes \bI_p\big) (\calM^{-1}\calD)^\top (\bI_T \otimes \be_j)\\
    &= \be_i^\top \bXi. 
\end{align*}

It remains to bound $\opnorm{\bXi}$. 
By definition of $\bXi$, we have 
\begin{align*}
    &\opnorm{\bXi}\\
    =& n^{-1} \opnorm{\sum_{j=1}^p \bX
    \big((\be_j^\top \bH) \otimes \bI_p\big) (\calM^{-1}\calD)^\top (\bI_T \otimes \be_j)}\\
    \le& n^{-1} \sum_{j=1}^p \opnorm{ \bX
    \big((\be_j^\top \bH) \otimes \bI_p\big) (\calM^{-1}\calD)^\top (\bI_T \otimes \be_j)}\\
    \le& n^{-1} \opnorm{\bX}
    \opnorm{\calM^{-1}}
    \opnorm{\calD}
    \sum_{j=1}^p \opnorm{ 
    (\be_j^\top \bH) \otimes \bI_p} \opnorm{\bI_T \otimes \be_j}\\
    \le& n^{-1} \opnorm{\bX}
    \opnorm{\calM^{-1}}
    \opnorm{\calD}
    [\sum_{j=1}^p \opnorm{ 
    (\be_j^\top \bH) \otimes \bI_p}^2]^{1/2}[\sum_{j}\opnorm{\bI_T \otimes \be_j}^2]^{1/2}\\
    \le& n^{-1} \opnorm{\bX}
    \opnorm{\calM^{-1}}
    \opnorm{\calD}
    \fnorm{\bH} \sqrt{Tp}\\
    \le& C(\zeta, T, \gamma) \fnorm{\bH} (1 + \opnorm{\bX}^2/n)^{T} && \mbox{by \Cref{lem:opnorm-M-N,lem:opnorm-D-J}.}
\end{align*}
This completes the proof of \Cref{lem:Wi}.
\end{proof}

\section{Proof of Theorem~\ref{thm:rt}}
\label{sec:proof-thm-rt}
\begin{proof}[Proof of \Cref{thm:rt}]
    By definition, we know 
    $\hat r_{t} = \hat r_{t,t}$ and 
    $r_{t} = r_{t,t}$. 
    Thus the result of this theorem follows directly from the result of 
    \Cref{thm:generalization-error}. 
\end{proof}

\section{Proof of Corollary~\ref{cor:early}}
\label{proof:cor:early}
\begin{proof}[Proof of \Cref{cor:early}]
    By the definition of $r_t$ in \eqref{eq:rt}, it suffices to show that 
    \begin{align*}
        \P\Bigl(
        r_{\hat t} - \min_{s\in[T]} r_s
        \ge
        \frac{\var(y_1)}{n^{1/2-c}}
        \Bigr) \le 
        \frac{C(\zeta,\gamma,T,\kappa)}{n^c}.
    \end{align*}
To see this, we have 
\begin{align*}
    &\P\Bigl(
        r_{\hat t} - \min_{s\in[T]} r_s \ge
        \frac{\var(y_1)}{n^{1/2-c}}
        \Bigr)\\
    \le& \frac{n^{1/2-c}}{\var(y_1)} \E\bigl[r_{\hat t} - \min_{s\in[T]} r_s\bigr] &&\mbox{Markov's inequality}\\
    \le& \frac{n^{1/2-c}}{\var(y_1)} \Bigl(\E[|r_{\hat t} - \hat r_{\hat t}|]
    + \E [|\hat r_{\hat t} - \min_{s\in[T]} r_s|]\Bigr) &&\mbox{triangle inequality}\\
    =& \frac{n^{1/2-c}}{\var(y_1)} \Bigl(\E[|r_{\hat t} - \hat r_{\hat t}|]
    + \E [|\min_{s\in[T]} \hat r_s - \min_{s\in[T]} r_s|]\Bigr) &&\mbox{by definition of $\hat t$}\\
    \le & \frac{n^{1/2-c}}{\var(y_1)} \Bigl(\E[|r_{\hat t} - \hat r_{\hat t}|]
    + \E [\max_{s\in[T]}|\hat r_s -r_s|]\Bigr)&&\mbox{by $|\min_s a_s - \min_s b_s| \le \max_s |a_s - b_s|$}\\
    \le & \frac{n^{1/2-c}}{\var(y_1)} \Bigl(\E[|r_{\hat t} - \hat r_{\hat t}|]
    + \max_{s\in[T]}\E [|\hat r_s -r_s|]\Bigr)&&\mbox{Jensen's inequality}\\
    \le& \frac{C(\zeta,\gamma,T,\kappa)}{n^c} &&\mbox{by \Cref{thm:rt}}.
\end{align*}
This concludes the proof. 
\end{proof}

\section{Proof of Theorem~\ref{thm:debias}}
\label{sec:proof-thm-debias}

Note that for this proof, we will consider general $\bSigma$ directly. 
We first present two lemmas that will be useful. 

\begin{lemma}\label{lem:moment-bound-Rem}
    Under the same condition of \Cref{thm:debias}. 
    For $\Rem_{j} = \sum_{i=1}^n  \sum_{t=1}^T \Delta_{ij}^{it} \be_t^\top \frac{1}{\|\bSigma^{-1/2} \be_j\|}$, we have 
    \begin{align*}
        \E \big[\sum_{j=1}^p \|\Rem_{j}\|^2\big] \le n C(\zeta, T, \kappa, \gamma) \var(y_1).
    \end{align*}
\end{lemma}
\begin{proof}[Proof of \Cref{lem:moment-bound-Rem}]
    By definition, $\Rem_j$ is a row vector in $\R^T$, and its $t$-th entry is 
    \begin{align*}
        \Rem_{jt} 
        &= \sum_{i=1}^n \Delta_{ij}^{it} \frac{1}{\|\bSigma^{-1/2} \be_j\|}\\
        &= - n^{-1} \sum_{i=1}^n (\be_t^\top\otimes (\be_i^\top\bX))
        \calM^{-1} \calD ((\bF^\top \be_i) \otimes \be_j ) \frac{1}{\|\bSigma^{-1/2} \be_j\|} 
        &&\mbox{by \eqref{eq:Delta-ij-lt}}\\
        &= - n^{-1} \sum_{i=1}^n \be_i^\top(\be_t^\top\otimes \bX)
        \calM^{-1} \calD (\bF^\top  \otimes \be_j ) \be_i \frac{1}{\|\bSigma^{-1/2} \be_j\|} &&\mbox{by \eqref{eq:kronecker-mix-product}}\\
        &= - n^{-1} \frac{1}{\|\bSigma^{-1/2} \be_j\|} \trace\bigl((\be_t^\top\otimes \bX) \calM^{-1} \calD (\bF^\top  \otimes \be_j) \bigr).
    \end{align*}
    For the trace in the last line above, we have for each $j\in [p], t\in [T]$, the following holds, 

Since $(\bF^\top \otimes \be_j) = (\bI_T \otimes \be_j) \bF^\top$, whose rank is at most $T$, we have 
\begin{align*}
    &\quad \trace\bigl((\be_t^\top\otimes \bX) \calM^{-1} \calD (\bF^\top  \otimes \be_j) \bigr)\\
    &\le T \opnorm{(\be_t^\top\otimes \bX) \calM^{-1} \calD (\bF^\top  \otimes \be_j)} 
    \\
    &\le T \opnorm{(\be_t^\top\otimes \bX)} \opnorm{\calM^{-1}} \opnorm{\calD} 
    \opnorm{(\bF^\top  \otimes \be_j)}
    &&\mbox{submultiplicativity of $\|\cdot\|_{\rm op}$}
    \\
    &= \sqrt{T} \opnorm{\bX} \opnorm{\calM^{-1}} \opnorm{\calD} \opnorm{\bF}
    &&\mbox{by \eqref{eq:kronecker-norm}}\\
    &\le C(\zeta, T) \opnorm{\bX} (1 + \opnorm{\bX}^2/n)^{T-1} \fnorm{\bF}
    &&\mbox{by \eqref{eq:opnorm-D-J} and \eqref{eq:opnorm-M-N}}\\
    &\le C(\zeta, T) \sqrt{n} (1 + \opnorm{\bX}^2/n)^{T} \fnorm{\bF}.
    &&\mbox{by $\opnorm{\bX}/\sqrt{n} \le (1 + \opnorm{\bX}^2/n)/2$ }
\end{align*}

    Thus, using $\max_{j\in [p]} \tfrac{1}{\|\bSigma^{-1/2} \be_j\|^2} \le \opnorm{\bSigma}$ and above inequality for the trace, we have
    \begin{align*}
        &\quad \sum_{j=1}^p \|\Rem_{j}\|^2 \\
        &= \sum_{j=1}^p \sum_{t=1}^T (\Rem_{jt})^2\\
        &\le n^{-2} \opnorm{\bSigma}\sum_{j=1}^p \sum_{t=1}^T \trace\bigl((\be_t^\top\otimes \bX) \calM^{-1} \calD (\bF^\top  \otimes \be_j) \bigr)^2\\
        &\le n^{-1}p \opnorm{\bSigma} C(\zeta, T) (1 + \opnorm{\bX}^2/n)^{2T} \fnorm{\bF}^2\\
        &\le \gamma \kappa C(\zeta, T) (1 + \opnorm{\bX}^2/n)^{2T} \fnorm{\bF}^2. &&\mbox{by $p/n\le \gamma$ and $\opnorm{\bSigma}\le \kappa$ from \Cref{assu:design}}
    \end{align*}
    Therefore, Taking expectation using the Cauchy-Schwarz inequality and \Cref{lem:moment-bound-X,lem:moment-bound-F-H}, we have
    \begin{align*}
        \E [\sum_{j=1}^p \|\Rem_{j}\|^2] \le n C(\zeta, T, \kappa, \gamma)  \var(y_1).
    \end{align*}
    
    This completes the proof of \Cref{lem:moment-bound-Rem}. 
\end{proof}

\begin{lemma}
    \label{lem:moment-bound-F-X}
    Under \Cref{assu:design,assu:noise,assu:regime,assu:Lipschitz}, we have 
    \begin{align*}
        n^{-1}  \E \Big[\sum_{t=1}^T \sum_{i=1}^n \sum_{j=1}^p (\be_j^\top \bSigma^{-1} \be_j)^{-1} \fnorm*{\pdv{\bF \be_t}{x_{ij}} }^2\Big]
        \le C(\zeta, T, \kappa, \gamma) \var(y_1).
    \end{align*}
\end{lemma}

\begin{proof}[Proof of \Cref{lem:moment-bound-F-X}]

By \Cref{lem:opnorm-dF/dZ} the mapping $\R^{n\times p} \to \R^{n}: \bX \mapsto \bF \be_t$ is Lipschitz. 
Since the Frobenius norm of the Jacobian of an $L_*$-Lipschitz function valued in $\R^n$ is at most $L_* \sqrt{n}$, we have 
the Frobenius norm of the Jacobian of the mapping $\R^{n\times p} \to \R^{n}: \bX \mapsto \bF\be_t$ is bounded by
\begin{align*}
    \sqrt{n} C(\zeta, T) (1 + \opnorm{\bX}^2/n)^T 
    (\fnorm{\bH} + \fnorm{\bF}/\sqrt{n}).
\end{align*}
Therefore, using $\max_{j\in[p]} (\be_j^\top \bSigma^{-1} \be_j)^{-1} \le (\lam_{\min}(\bSigma^{-1}))^{-1} = \opnorm{\bSigma}$, 
we have 
\begin{align*}
    &\quad n^{-1} \sum_{t=1}^T \sum_{i=1}^n \sum_{j=1}^p (\be_j^\top \bSigma^{-1} \be_j)^{-1} \fnorm*{\pdv{\bF \be_t}{x_{ij}} }^2\\
    &\le n^{-1} \opnorm{\bSigma} \sum_{t=1}^T \sum_{i=1}^n \sum_{j=1}^p \fnorm*{\pdv{\bF \be_t}{x_{ij}} }^2\\
    &\le n^{-1} \opnorm{\bSigma} \sum_{t=1}^T  \fnorm*{\pdv{\bF \be_t}{\vec(\bX)}}^2\\
    &\le \opnorm{\bSigma} 
    C(\zeta, T) (1 + \opnorm{\bX}^2/n)^{2T} 
    (\fnorm{\bH} + \fnorm{\bF}/\sqrt{n})^2.
\end{align*}
Taking expectations, using the Cauchy-Schwarz inequality and \Cref{lem:moment-bound-X,lem:moment-bound-F-H}, we have
\begin{align*}
    &\quad n^{-1} \E \Big[\sum_{t=1}^T \sum_{i=1}^n \sum_{j=1}^p (\be_j^\top \bSigma^{-1} \be_j)^{-1} \fnorm*{\pdv{\bF \be_t}{x_{ij}} }^2\Big]\\
    &\le  
    C(\zeta, T, \kappa, \gamma)
    \var(y_1).
\end{align*}

\end{proof}

\begin{proof}[Proof of \Cref{thm:debias}]
    \label{proof:thm-debias}
    
    For any $j\in[p]$,  $\be_j\in \R^p$, 
    let $\bz_j = \bX \bSigma^{-1} \be_j /\|\bSigma^{-1/2} \be_j\|$, then $\bz_j \sim \mathsf{N}(\b0, \bI_n)$. 
    Let $\bF(\bz_j) = \bY - \bX \hbB \in \R^{n\times T}$.
    Since 
    $$\bX 
    = \bX(\bI_p - \frac{\bSigma^{-1} \be_j \be_j^\top}{\|\bSigma^{-1/2} \be_j\|^2}) + \bX \frac{\bSigma^{-1} \be_j \be_j^\top}{\|\bSigma^{-1/2} \be_j\|^2}
    := \bX \bPi_j + \bz_j \frac{\be_j^\top}{\|\bSigma^{-1/2} \be_j\|},
    $$
    where $\bPi_j = \bI_p - \frac{\bSigma^{-1} \be_j \be_j^\top}{\|\bSigma^{-1/2} \be_j\|^2}$. 

    We now show 
    $\vec(\bX \bPi_j)$ is independent of $\bz_j$. 
    Let $\bG = \bX\bSigma^{-1/2}$, we have 
    \begin{align*}
        \vec(\bX \bPi_j) 
        = \vec(\bG \bSigma^{1/2} \bPi_j)
        =(\bSigma^{1/2}\bPi_j \otimes \bI_n) \vec(\bG)
    \end{align*}
    and 
    \begin{align*}
        \bz_j 
        = \bG \bSigma^{-1/2} \be_j/\|\bSigma^{-1/2} \be_j\|
        = \frac{1}{\|\bSigma^{-1/2} \be_j\|}
        (\be_j^\top \bSigma^{-1/2} \otimes \bI_n) \vec(\bG).
    \end{align*}
    Since $\vec(\bG)$ is a standard Gaussian vector in $\R^{np}$,
    the above two vectors
    $\vec(\bX \bPi_j)$ and $\bz_j$ are two Gaussian vectors. 
    It suffices to show that $\vec(\bX \bPi_j)$ and $\bz_j$ are uncorrelated,
    which can be shown by verifying that 
    $$
    (\be_j^\top \bSigma^{-1/2} \otimes \bI_n) (\bSigma^{1/2}\bPi_j \otimes \bI_n)
    = 
    (\be_j^\top \bPi_j \otimes \bI_n) = \b0
    $$
    thanks to $ \bPi_j^\top \be_j = \b0$ from definition of $\bPi_j$.

    Applying \cite[Lemma S5.3]{tan2023multinomial} to the mapping: $\bz_j \mapsto \bF(\bz_j)/\sqrt{n}$, using $z_{ij} = \be_i^\top \bz_j$, we have
    \begin{equation}
        \label{SOS_normality}
    \E \Big[\norm[\Big]{ \frac{\bz_j ^\top \bF(\bz_j)}{\sqrt{n}} - \frac{1}{\sqrt n}\sum_{i=1}^n \pdv{\be_i^\top \bF(\bz_j)}{z_{ij}} - \frac{\bz_j^\top \bF(\tilde\bz_j)}{\sqrt{n}}}^2\Big]
        \le \frac3n \E \sum_{i=1}^n \fnorm*{\pdv{\bF(\bz_j)}{z_{ij}}}^2,
    \end{equation}
    where $\tilde \bz_j\sim \mathsf{N}(\boldzero, \bI_n)$
    is independent of $(\bX,\by)$.

    We want to compute the derivative of $\bF(\bz_j)$ with respect to $z_{ij}$. 
    Since $\bX= \bX\bPi_j+ \bz_j \frac{\be_j^\top}{\|\bSigma^{-1/2} \be_j\|}$. 
    Conditional on $\bX\bPi_j$ (that is, with $\bX\bPi_j$ held fixed), we have 
    $\pdv{\bF}{z_{ij}} = \pdv{\bF}{x_{ij}} \frac{1}{\|\bSigma^{-1/2} \be_j\|}$. 
    According to the expression of $\pdv{F_{lt}}{x_{ij}}$ in \Cref{lem:dot-F-4}, we have
    \begin{align}
        \sum_{i=1}^n\pdv{\be_i^\top \bF}{z_{ij}}
        &= \sum_{i=1}^n \pdv{\be_i^\top \bF}{x_{ij}} \frac{1}{\|\bSigma^{-1/2} \be_j\|}
        \nonumber\\
        &= \sum_{i=1}^n  \sum_{t=1}^T (D_{ij}^{it} + \Delta_{ij}^{it} ) \be_t^\top \frac{1}{\|\bSigma^{-1/2} \be_j\|}
        \nonumber\\
        & = \frac{-\be_j^\top (\hbB - \bB^*) (n \bI_T - \hbA)^\top}{\|\bSigma^{-1/2} \be_j\|} 
        + \Rem_{j} && \mbox{by \eqref{eq:sum-D}}
        ,
        \label{eq:expression_divergence_N01}
    \end{align}
    where $\Rem_{j} = \sum_{i=1}^n  \sum_{t=1}^T \Delta_{ij}^{it} \be_t^\top \frac{1}{\|\bSigma^{-1/2} \be_j\|}$.
    Define $\bw_j\in\R^T$ by
    $$\bw_j^\top
    \defas
    \frac{\be_j^\top  \bSigma^{-1}\bX^\top  (\bY - \bX \hbB) + \be_j^\top (\hbB - \bB^*)(n\bI_T - \hbA)^\top}{\sqrt{n}\|\bSigma^{-1/2} \be_j\|} - \frac{\bz_j^\top \bF(\tilde\bz_j)}{\sqrt{n}}
.
    $$
    Combining \eqref{SOS_normality} and \eqref{eq:expression_divergence_N01},
    by the triangle inequality, 
    \begin{align*}
        \sum_{j=1}^p\E \Big[\Big\|\bw_j^\top\Big\|^2\Big] 
        &=\sum_{j=1}^p\E \Big[\Big\|\frac{\bz_j ^\top \bF(\bz_j)}{\sqrt n} - \sum_{i} \pdv{\be_i^\top \bF(\bz_j)}{z_{ij}}/\sqrt{n} - \frac{\bz_j^\top \bF(\tilde\bz_j)}{\sqrt{n}}
        - \frac{\Rem_j}{\sqrt{n}}
        \Big\|^2\Big] \\
        &\le \frac2n \E \Big[ \sum_{j=1}^p\|\Rem_{j}\|^2\Big] + \frac6n \E \Big[\sum_{i=1}^n\sum_{j=1}^p \fnorm*{\pdv{\bF(\bz_j)}{z_{ij}}}^2\Big].
    \end{align*}
    We now bound the two terms in the last line.
    For the first term, it is bounded from above by $C(\zeta, T, \kappa, \gamma) \var(y_1)$ thanks to \Cref{lem:moment-bound-Rem}. 
    For the second term, we first rewrite the derivative with respect to $z_{ij}$ using chain rule of differentiation. 
    Recall $\bX= \bX\bPi_j+ \bz_j \frac{\be_j^\top}{\|\bSigma^{-1/2} \be_j\|}$. 
    Conditional on $\bX\bPi_j$, we have 
    $\pdv{\bF}{z_{ij}} = \pdv{\bF}{x_{ij}} \frac{1}{\|\bSigma^{-1/2} \be_j\|}$.
    Thus,
    \begin{align*}
        \frac1n \E \Big[\sum_{i=1}^n \sum_{j=1}^p\fnorm*{\pdv{\bF(\bz_j)}{z_{ij}}}^2\Big]
        =\frac1n  \E \Big[\sum_{t=1}^T \sum_{i=1}^n \sum_{j=1}^p (\be_j^\top \bSigma^{-1} \be_j)^{-1}\fnorm*{\pdv{\bF \be_t}{x_{ij}} }^2\Big].
    \end{align*}
    According to \Cref{lem:moment-bound-F-X}, the last line is bounded by $C(\zeta, T, \kappa, \gamma) (\|\bSigma^{1/2}b^*\| + \sigma^2)$.
    In summary, $\sum_{j=1}^p \E[\|\bw_j^\top\|^2]
        \le C(\zeta, T, \kappa, \gamma) \var(y_1).$
    Let
    $$\bL_n = (\bI_T - \hbA/n)^{-1}.$$
    Left multiplying $(\bI_T- \hbA/n)^{-1}$ inside the $\ell_2$ norm
    of $\bw_j$, 
    we have 
    \begin{equation*}
        \sum_{j=1}^p\E \Big[\Big\|\frac{\bL_n(\bY - \bX \hbB)^\top \bX \bSigma^{-1} \be_j +  n(\hbB - \bB^*)^\top \be_j}{\sqrt{n}\|\bSigma^{-1/2} \be_j\|} -  \frac{\bL_n\bF(\tilde\bz_j)^\top \bz_j}{\sqrt{n}}\Big\|^2\Big]
        =
        \sum_{j=1}^p\E \Big[\Big\|\bL_n\bw_j\Big\|^2\Big].
    \end{equation*}
    Since $\|(\bI_T- \hbA/n)^{-1}\|_{\rm op}\le C(\zeta, T) (1 + \opnorm{\bX}^2/n)^{T^2}$ from \Cref{lem:opnorm-Ahat},
    define the event 
    $\Omega = 
    \{\bX\in \R^{n\times p}: 
    \opnorm{\bX}/\sqrt{n} \le 2 + \sqrt{\gamma}\}$ as before,
    so that
    $$\sum_{j=1}^p
    \E[\|\bL_n\bw_j\|^2]
    \le 
    \sum_{j=1}^p
    C(\zeta,T,\gamma)
    \E\Bigl[\|\mathbb I\{\Omega\}\bw_j\|^2\Bigr]
    +
    \sum_{j=1}^p
    \E\Bigl[\|\mathbb I\{\Omega^c\}\bL_n\bw_j\|^2\Bigr]
    .
    $$
    For the first term with $\mathbb I\{\Omega\}$,
    the previously established bound  $\sum_{j=1}^p \E[\|\bw_j\|^2]
    \le C(\zeta,T,\kappa,\gamma) \var(y_1)$ bounds from above the first term.
    Each summand in the second term is exponentially small,
    using $\E[\mathbb I\{\Omega^c\} \|\bw_j\|^2]
    \le \P(\Omega^c)^{1/2} \E[\|\bw_j\|^4]^{1/2}$
    with $\P(\Omega^c)^{1/2} \le e^{-n/4}$,
    while 
    $\E[\|\bw_j\|^4]^{1/2}/\var(y_1)$ is at most polynomial
    in $(n,p)$ with multiplicate constant $C(\zeta,T,\kappa,\gamma)$
    thanks to several applications of the Cauchy-Schwarz
    inequality to combine the moment bounds in
    \Cref{lem:moment-bound-X,lem:moment-bound-F-H}
    and
    the equality in distribution $\bF(\tilde \bz_j)=^d \bY-\bX\hbB$.

    We wish to show that for
    the 2-Wasserstein distance,
    \begin{equation}
        \label{conv_W_2_announce_proof}
    \max_{j\in [p]}
    W_2\Bigl(
        \mathsf{Law}\bigl(\frac{\bL_n\bF(\tbz_j)^\top\bz_j}{\sqrt n} \bigr),
        ~~ \mathsf{N}\bigl(\boldzero_T, \bS+\E[\bH^\top\bSigma\bH]\bigr)
    \Bigr) \to 0.
    \end{equation}
    By \Cref{lem:moment-bound-F-H}, we may extract a subsequence
    of regression problems such that there exists psd matrices
    $\bK,\bar\bK\in\R^{T\times T}$ such that
    \eqref{eq:K_bar_K} holds.
    We will not adapt a particular notation for the extracted
    subsequence, but keep in mind that we are now working
    along this subsequence and that $n\to+\infty$ with $n$
    being restricted to this subsequence.

    By \eqref{eq:variance_to_0_F}
    and \eqref{eq:variance_to_0_H},
    \begin{equation}
    \E[\|\bF^\top\bF/n -\bK\|_{\rm F}^2]\to 0,
    \qquad
    \E[\|\bH^\top\bSigma\bH + \bS -\bar\bK\|_{\rm F}^2]\to 0.
    \end{equation}
    Note that \Cref{thm:generalization-error} further shows that
    \begin{equation}\E[\| \bL_n(\bF^\top\bF/n)\bL_n^\top - \bar\bK\|_{\rm F}]\to 0.
        \label{convergenceL_n_F_to_bar_K}
    \end{equation}

    Consider the LDLT decompositions $\bK = \bL\bD\bL^\top$
    and $\bar\bK = \bar\bL \bar \bD\bar \bL^\top$ for the psd matrices
    $\bK$ and $\bar\bK$, so that $\bD,\bar\bD$ are diagonal with
    non-negative entries and
    $\bL,\bar\bL$ are lower triangular with all diagonal entries equal to 1.
    Since \Cref{lem:opnorm-Ahat} gives $\|\bL_n\|_{\rm op} = O_P(1)$,
    \begin{equation}
        \bL_n \bK \bL_n^\top - \bar \bK\limp 0,
        \qquad
        \bar\bL^{-1} \bL_n \bL\bD - \bar \bD \bar\bL^\top (\bL_n^\top)^{-1}
        (\bL^\top)^{-1}
        \limp 0.
    \end{equation}
    On the right, we have the difference of a lower triangular matrix
    with diagonal $\bD$
    and an upper triangular matrix with diagonal $\bar\bD$. Thus
    the convergence to 0 in probability gives
    $\bD=\bar\bD$ and $\bar \bL^{-1}\bL_n\bL \bD \limp \bD$.
    We also have
    $\bar \bL^{-1}\bL_n\bL \sqrt{\bD} \limp \sqrt{\bD}$
    by multiplying on the right by the pseudo-inverse of
    $\sqrt{\bD}$.
    Now let $\bL_\infty = \bar\bL \bL^{-1}$, so that
    \begin{align*}
        (\bL_n - \bL_\infty)\bL\sqrt{\bD}
        &=
        \bar\bL \bar\bL^{-1}(\bL_n - \bL_\infty)\bL\sqrt{\bD}
      \\&= \bar\bL( \bar\bL^{-1}\bL_n \bL \sqrt{\bD} -\sqrt{\bD})
      \limp 0.
    \end{align*}
    By the continuous mapping theorem, $(\bL_n-\bL_\infty)\bK(\bL_n-\bL_\infty)^\top\limp0$, 
    and since \Cref{lem:opnorm-Ahat} gives $\|\bL_n\|_{\rm op} = O_P(1)$,
    \begin{equation}
        (\bL_n-\bL_\infty)(\bF^\top\bF/n)(\bL_n-\bL_\infty)^\top
        \limp 0,
        \qquad
        \|(\bL_n-\bL_\infty)\bF^\top/\sqrt n\|\limp 0.
        \label{eq:convergence_L_infty}
    \end{equation}
    Although we have not shown that $\bL_n\limp \bL_\infty$,
    this means that we can still replace $\bL_n$ by the deterministic
    $\bL_\infty$ up to when multiplied on the right by $\bF^\top/\sqrt n$.
    The previous display also holds with $\bF$ replaced by
    $\bF(\tbz_j)$, uniformly over $j\in[p]$ since
    \begin{align}
        &
        \E
        \|
        (\bL_n-\bL_\infty)\tfrac{\bF^\top\bF}{n}(\bL_n-\bL_\infty)^\top
        -
        (\bL_n-\bL_\infty)\tfrac{\bF(\tbz_j)^\top\bF(\tbz_j)}{n}(\bL_n-\bL_\infty)^\top
        \|_{\rm op}
        \label{bound_L_n_L_infty_with_tilde_F}
      \\&\le
        \E[
        \|\bL_n-\bL_\infty\|_{\rm op}^2
        \|\bF^\top\bF/n - \bF(\tbz_j)^\top\bF(\tbz_j)/n\|_{\rm op}
        ]
        \nonumber
      \\&\le
        \E[
        \|\bL_n-\bL_\infty\|_{\rm op}^4]^{1/2}
        2
        \E[\|\bF^\top\bF/n-\E[\bF^\top\bF/n]\|_{\rm op}^2]
        \to 0
        \nonumber
    \end{align}
    thanks to the Cauchy-Schwarz inequality and
    the equality in distribution $\bF=\bF(\tbz_j)$
    for the last inequality.
    The last line converges to 0 thanks to 
    \Cref{lem:opnorm-Ahat} to show that
    $\E[ \|\bL_n-\bL_\infty\|_{\rm op}^4]^{1/2}$ is bounded,
    and thanks to
    \eqref{eq:variance_to_0_F} which gives
    $\E[\|\bF^\top\bF/n-\E[\bF^\top\bF/n]\|_{\rm op}^2]\to 0$.
    This gives that
    $\|\bL_n\bF(\tbz_j)/\sqrt n\|_{\rm op} \limp 0$
    uniformly over $j\in[p]$, so that
    $$\bL_n\bF(\tbz_j)^\top\bz_j/\sqrt n =
    \bL_\infty \bF(\tbz_j)^\top \bz_j /\sqrt n + o_P(1).$$
    Since $\bL_\infty$ is deterministic,
    the distribution $\bL_\infty \bF(\tbz_j)^\top \bz_j /\sqrt n$
    is independent of $j$ and is the same as the distribution
    of $\bL_\infty\bF^\top \bg/\sqrt n$ where $\bg\sim \mathsf{N}(\boldzero,
    \bI_n)$ is independent of $\bF$.
    For any $\bw\in\R^T$,
    the characteristic function of $\bL_\infty\bF^\top \bg/\sqrt n$ 
    evaluated at $\bw$ is
    \begin{align*}
        \E [\exp (\sqrt{-1} \bw^\top \bF^\top\bg /\sqrt n )] 
        &=\E \big[\E [\exp (\sqrt{-1} \bw^\top\bL_\infty \bF^\top\bg/\sqrt n) \mid \bg]\big]\\
        &=\E \big[ \exp(-\tfrac1{2n} \bw^\top
        \bL_\infty \bF^\top\bF\bL_\infty^\top \bw)\big].
    \end{align*} 
    This converges to $\exp(-\frac12\bw^\top\bar\bK\bw)$
    by \eqref{eq:convergence_L_infty}
    and \eqref{convergenceL_n_F_to_bar_K}.
    We have established the weak convergence
    $\bL_n\bF(\tbz_j)^\top \bz_j/\sqrt n\limd \mathsf{N}(\boldzero,\bar\bK)$
    uniformly over $j\in[p]$.

    To prove convergence in 2-Wasserstein distance to $\mathsf{N}(\boldzero,\bar\bK)$, it is enough to establish convergence in distribution 
    and convergence in the
    second moment \cite[Def. 6.8 and Theorem 6.9]{villani2009optimal}.
    Since we have already established convergence in distribution,
    it is enough to prove
    \begin{equation}
        \E[\|\bL_n\bF(\tbz_j)^\top\bz_j/\sqrt n\|^2]\to \trace \bar\bK = \E[\|\mathsf{N}(\boldzero,\bar\bK)\|^2]
        .
    \end{equation}
    Since $\bF$ is equal in distribution to $\bF(\tbz_j)$,
    and $\tbz_j$ is independent of $\bF(\tbz_j)$, we have
    \[
        \E[
        \|\bL_\infty \bF(\tbz_j)^\top\bz_j/\sqrt n\|^2 
        ]
        =
        \trace[\bL_\infty\E[\bF^\top\bF/n]\bL_\infty^\top]
        \to
        \trace[\bL_\infty \bK \bL_\infty^\top]
        =
        \trace \bar \bK.
    \]
    where we used that $\bD=\bar\bD$ and $\bL_\infty=\bar\bL \bL^{-1}$
    for the last inequality.
    By the triangle inequality in $L^2(\R^T)$,
    $$
    \Big|
        \E\Bigl[
        \|\bL_\infty \bF(\tbz_j)^\top\bz_j/\sqrt n\|^2 
        \Bigr]^{\frac12}
        -
        \E\Bigl[\|\bL_n\bF(\tbz_j)^\top\bz_j/\sqrt n\|^2\Bigr]^{\frac12}
    \Big|
    \le
    \E\Bigl[\|(\bL_n-\bL_\infty)\bF(\tbz_j)^\top\bz_j/\sqrt n\|^2\Bigr]^{\frac12}.
    $$
    Let $\tbQ\in\R^{n\times n}$ be the orthogonal projector
    on the image of $\bF(\tbz_j)$, which is of rank at most $T$
    and such that $\bF(\tbz_j)^\top\tbQ = \bF(\tbz_j)$.
    Let $\chi^2_T = \|\tbQ\bz_j\|^2$ which has chi-square distribution
    with at most $T$ degrees of freedom.
    We may bound from above the integrand in the right-hand side
    by $\|\bL_n-\bL_\infty\bF(\tbz_j)/\sqrt n\|^2 \chi^2_T$.
    Since $T$ is bounded, $\chi^2_T= O_P(1)$ and
    $\|\bL_n-\bL_\infty\bF(\tbz_j)/\sqrt n\|\limp 0$ by
    \eqref{bound_L_n_L_infty_with_tilde_F}.
    Since the integrand has finite second moment by
    \Cref{lem:moment-bound-F-H,lem:opnorm-Ahat}
    and $\E[(\chi^2_T)^4]<\infty$, it is uniformly
    integrable so the last display converges to 0.

    We have established convergence in distribution and in the second
    moment, hence
    by \cite[Def. 6.8 and Theorem 6.9]{villani2009optimal},
    \begin{equation}
        \label{conv_W_2_barK}
    \max_{j\in [p]}
    W_2\Bigl(
        \mathsf{Law}\Bigl(\frac{\bL_n\bF(\tbz_j)^\top\bz_j}{\sqrt n} \Bigr),
        ~~ \mathsf{N}(0, \bar\bK)
    \Bigr) \to 0.
    \end{equation}
    Since $\bS+\E[\bH^\top\bSigma\bH]\to \bar \bK$ 
    and
    $\bC\mapsto \mathsf{N}(\boldzero,\bC)$ is continuous
    in 2-Wasserstein distance,
    so that $W_2(\mathsf{N}(\b0, \bar\bK), \bS+\E[\bH^\top\bSigma\bH])\to 0$.
    Combined with \eqref{conv_W_2_barK} and
    \cite[Corollary 6.11]{villani2009optimal},
    we obtain
    \eqref{conv_W_2_announce_proof} along any subsequence
    such that \eqref{eq:K_bar_K} holds.

    Actually, the convergence \eqref{conv_W_2_announce_proof}
    holds without extracting a subsequence for the following reason.
    Since both distributions inside $W_2$ in
    \eqref{conv_W_2_announce_proof} have uniformly bounded second
    (thanks to \Cref{lem:moment-bound-F-H,lem:opnorm-Ahat}),
    the $W_2$ distance in \eqref{conv_W_2_announce_proof}
    is uniformly bounded. To prove \eqref{conv_W_2_announce_proof},
    it is sufficient to prove that 0 is the only limit
    point of the sequence in the left-hand side of \eqref{conv_W_2_announce_proof}. For any subsequence such that the left-hand side of
    \eqref{conv_W_2_announce_proof} converges to a limit
     $\nu$, we may extract further a subsequence such that 
     \eqref{eq:K_bar_K} and the above argument holds, showing
     that the limit point must be $\nu =0$.
     Since the bounded sequence in the left-hand side of \eqref{conv_W_2_announce_proof} has a unique limit point equal to 0,
     the convergence \eqref{conv_W_2_announce_proof} holds.
    \end{proof}

\section{Proof of Corollary~\ref{cor:debias}}\label{sec:proof_cor_debias}
\begin{proof}[Proof of \Cref{cor:debias}]
        Let $J_{n,p}\subset [p]$ be the set of $j\in[p]$ such that
        \begin{equation}
            \label{eq:previous_display_W2_proof}
        \E\Bigl[
        \Bigl(
            \sqrt{\frac{n}{\|\bSigma^{-1/2} \be_j\|^2}}
            \Bigl(\hbb^{t, \rm{debias}}_j - \bb^*_j\Bigr)
            - \bzeta_{jt}
        \Bigr)^2
        \Bigr] 
        \le \frac{1}{a_p} C(\zeta,T,\kappa,\gamma)\var(y_1) 
        \end{equation}
        for the same constant $C(\zeta,T,\kappa,\gamma)$ as in \eqref{eq:debias-normality}. Bounding the left-hand side of
        \eqref{eq:debias-normality} from below by the
        sum indexed over $[p]\setminus J_{n,p}$, we get
        $\frac{1}{a_p} |[p]\setminus J_{n,p}| \le 1$.
        On the one hand,
        $W_2(\mathsf{Law}(\bzeta_{jt}), \mathsf{N}(0, \E[r_t]))
            \to 0$ by \eqref{convergence_Zeta_j_W2_to_0},
            uniformly over $j\in [p]$.
        On the other hand,
            \eqref{eq:previous_display_W2_proof}
            provides a vanishing upper bound on the $W_2$ distance
            between the law of $\bzeta_{jt}$ and the law
            of 
            $\frac{\sqrt n}{\|\bSigma^{-1/2} \be_j\|}
            (\hbb^{t, \rm{debias}}_j - \bb^*_j)$,
            uniformly over $j\in J_{n,p}$.
        The triangle inequality for the 2-Wasserstein distance
        completes the proof of
        \eqref{eq:Zj}.
    \end{proof}

\section{Asymptotic normality and state evolution in the separable case}
\label{sec:appendix_state_ev}
\begin{theorem}
    \label{thm:state_ev}
    Let \Cref{assu:design,assu:Lipschitz,assu:noise,assu:regime}
    be fulfilled with $\bSigma=\bI_p$,
    assume $\var(y_1)\le v_0$,
    and
    assume that 
    $T,v_0$ are fixed constant as \( n,p\to+\infty \).
    Then there exists deterministic weights $\bar w_{s,t}$ (possibly
    depending on $n,p$) such that
    \begin{equation}
        \max_{t=1,...,T}
        \Bigl(
        \frac1n
        \|\sum_{s=1}^t (\hat w_{t,s} - \bar w_{t,s}) (\by - \bX\hbb^s)\|^2
        \Bigr)
        \limp 0.
        \label{eq:convergence_deterministic_weights}
    \end{equation}
    Given these deterministic weights $\bar w_{t,s}$, define the
    debiased estimate $\bar\bb^{T, \rm{debias}}\in\R^p$ by
    $\bar\bb_j^{T, \rm{debias}}
    = \hbb^T_j + n^{-1} (\bX\be_j)^\top\sum_{s=1}^T \bar w_{T,s}
    (\by-\bX\hbb^s)$.

    Let $(\eta_{j})_{j\in [p]}$ be $\zeta$-Lipschitz functions.
    Choose the $T+1$-th nonlinear function $\bg_{T+1}$
    in \eqref{eq:iterates_all_previous} at $t=T+1$
    by applying the functions $\eta_j$ componentwise to 
    $\bar\bb_j^{T, \rm{debias}}$,
    \ie, $\hbb^{T+1}_j
        =
        \eta_j(\bar\bb_j^{T, \rm{debias}}).$ Then the following asymptotic recursion holds between
    $\E[r_T]$ and
    $\E[r_{T+1}]$,
    \begin{align}
        \label{eq:risk_evolution}
    \Bigl(
    \sigma^2
    +
    \E_{Z_j\sim \mathsf{N}(0,\E[r_t])}\Bigl[\sum_{j=1}^p\Bigl\{
            \eta_j\bigl(\bb_j^* + \frac{Z_j}{\sqrt n}\bigr)
    -\bb_j^*
\Bigr\}^2\Bigr]
\Bigr)^{1/2}
    -
    \E\Bigl[r_{T+1}\Bigr]^{1/2}
    &\to 0
    \end{align}
    as $n,p\to+\infty$, as well as
    $
    (
    \sigma^2
    +
    \E[\sum_{j=1}^p\{
            \eta_j(\hbb^{t, \rm{debias}}_j)
    -\bb_j^*
\}^2]
    )^{1/2}
    -
    \E[r_{T+1}]^{1/2}
    \to 0
    $.
\end{theorem}

\begin{proof}[Proof of \Cref{thm:state_ev}]
    Without loss of generality, assume that the sequence
    of regression problems and data $(\bX,\by)$ is defined
    on the same probability space $(\Omega,\mathcal A,\mathbb P)$.
    We have already proved in \Cref{lem:moment-bound-F-H,lem:opnorm-Ahat,lem:var-F-H}
    and \Cref{thm:generalization-error} that the event $\Omega_n$ defined as
    \begin{align}
        \|(\bI_T-\hbA/n)^{-1}\|_{\rm op} + \|\bI_T-\hbA/n\|_{\rm op}\le C(\zeta,T,\gamma,v_0),
        \\
    \|(\bI_T-\hbA/n)^{-1}\bF^\top\bF/n(\bI_T-\hbA/n)^{-1}
    -(\sigma^2 \bd1_T\bd1_T^\top+\bH^\top\bH)\|_{\rm F}\le n^{-0.49}
    \\
    \|\bF^\top\bF/n - \E[\bF^\top\bF/n]\|_{\rm F}
    +
    \|\bH^\top\bH - \E[\bH^\top\bH]\|_{\rm F}\le n^{-0.49}
    \end{align}
    for a large enough constant $C(\zeta,T,\gamma,v_0)$
    has probability approaching one. In particular, this event
    $\Omega_n\in \mathcal A$ is non-empty and contains some outcome $\omega_n$.
    For each $s,t\le T$,
    for a given regression problem in the sequence (indexed by $n$),
    we may define the weight $\bar w_{t,s}$ as the random variable
    $\hat w_{t,s}=\be_t^\top(\bI_T-\hbA/n)^{-1}\be_s$
    taken as the outcome $\omega_n$, so that
    $\bar w_{t,s}$ is deterministic.
    Let us emphasize that this weight implicitly depends on
    $n$, \ie, it depends on $n,p,\bb^*,(\bg_{t})_{t\le T}$
    and all other parameters that are allowed to change as $n,p\to+\infty$.
    Since $\omega_n$ is an outcome in $\Omega_n$, the above inequalities
    give that $\bar \bW = (\bar w_{t,s})_{t\in T, s\in [T]}$ satisfies
    \begin{equation}
        \|\bar\bW \E[\bF^\top\bF/n] \bar\bW^\top - 
        (\sigma^2 \bd1_T\bd1_T^\top+\E[\bH^\top\bH])\|_{\rm F} 
        \le C(\zeta,T,\gamma,v_0) n^{-0.49}
            \label{bar_W_limit}
    \end{equation}
    in the sense of convergence of a deterministic sequence.
    Let us now show that \eqref{eq:convergence_deterministic_weights}
    holds.
    Denote by $\hbW=(\bI_T-\hbA/n)^{-1}$ so that
    \eqref{eq:convergence_deterministic_weights} is equivalent to
    $\|(\bar\bW - \hbW)\bF^\top/\sqrt n\|_{\rm F}^2\limp 0$.
    Since convergence in probability to 0 for the sequence
    $U_n$
    is equivalent to convergence
    of $\E[1 \wedge |U_n|]\to 0$, the convergence in probability
    \eqref{eq:convergence_deterministic_weights} is equivalent to
    $u_n \defas
    \E[1\wedge \|(\bar\bW - \hbW)\bF^\top/\sqrt n\|_{\rm F}^2]\to 0$.
    Consider a converging
    subsequence $(u_{n_k})$ of $(u_n)_{n\ge 1}$ with limit point $u_\infty\in[0,1]$.
    By \eqref{eq:K_bar_K},
    we may further extract a subsequence $(u_{n_{k_m}})$, such that in this
    nested subsequence
    $u_{n_{k_{m}}}\to u_\infty$ and $\E[\bF^\top\bF/n]\to \bK$ both holds.
    In the event $\Omega_n$, 
    \eqref{bar_W_limit} is also satisfied
    for $\hbW$, hence in this subsequence and in $\Omega_n$,
    \begin{equation}
        \|\bar\bW \bK \bar \bW^\top - \hbW \bK \hbW^\top\|_{\rm F}
        \le C(\zeta,T,\gamma,v_0)(n^{-0.49} + \|\bK - \E[\bF^\top\bF/n]\|_{\rm F}).
    \end{equation}
    Let $\bK=\bL\bD\bL^\top$ be the LDLT decomposition of $\bK$
    with $\bL$ triangular with diagonal elements all equal to 1
    and $\bD$ diagonal with non-negative entries.
    Since the operator norms of $\bar\bW$  and $\hbW$ are bounded
    by $C(\zeta,T,\gamma,v_0)$ in $\Omega_n$, multiplying on the left
    by $\bL^{-1}\hbW^{-1}$ and on the right by $(\bL^\top\bar\bW^\top)^{-1}$,
    we get
    \begin{multline*}
        \|(\hbW\bL)^{-1}\bar\bW \bL \bD -  \bD \bL^\top \hbW^\top (\bL^\top\bar\bW^\top)^{-1}\|_{\rm F}
        \\\le C(\zeta,T,\gamma, v_0,\|\bL\|_{\rm op})(n^{-0.49} + \|\bK - \E[\bF^\top\bF/n]\|_{\rm F}).
    \end{multline*}
    Now $(\hbW\bL)^{-1}\bar\bW \bL \bD$ is lower triangular
    while 
    $\bD \bL^\top \hbW^\top (\bL^\top\bar\bW^\top)^{-1}$
    is upper triangular, so the left-hand side is bounded from below
    by
    $
    \|(\hbW\bL)^{-1}\bar\bW \bL \bD -  \bD\|_{\rm F}$
    (we may bound from below the full Frobenius norm by the Frobenius norm
    of the lower triangular part only). Thus on $\Omega_n$,
    \begin{equation*}
        \|(\hbW\bL)^{-1}\bar\bW \bL \bD -  \bD\|_{\rm F}
        \le C(\zeta,T,\gamma,v_0, \|\bL\|_{\rm op})(n^{-0.49} + \|\bK - \E[\bF^\top\bF/n]\|_{\rm F}).
    \end{equation*}
    Multiplying by $\hbW\bL$ on the left again, and by
    $\bL^\top(\hbW - \bar\bW)^\top$ on the right,
    we have proved
    that on $\Omega_n$,
    $$
    \|(\bar \bW \bL \bD  - \hbW \bL \bD) \bL^\top(\hbW - \bar\bW)^\top
    \|_{\rm F}
    \le C(\zeta,T,\gamma,v_0, \|\bL\|_{\rm op})(n^{-0.49} + \|\bK - \E[\bF^\top\bF/n]\|_{\rm F}).
    $$
    Finally, we may replace $\bK=\bL\bD\bL^\top$
    by $\bF^\top\bF/n$ on $\Omega_n$ in the left-hand side
    by enlarging the right-hand side by a constant if necessary.
    This implies that in $\Omega_n$,
    $$
    \|(\bar \bW - \hbW)(\bF^\top\bF/n)(\hbW - \bar\bW)^\top
    \|_{\rm F}
    \le C(\zeta,T,\gamma,v_0, \|\bL\|_{\rm op})(n^{-0.49} + \|\bK - \E[\bF^\top\bF/n]\|_{\rm F})
    $$
    and since $\P(\Omega_n)\to 1$, this shows that $u_\infty=0$
    is the unique limit point of the sequence $(u_n)_{n\ge 1}$.
    Since $0$ is the unique limit point, $(u_n)_{n\ge 1}\to 0$
    and \eqref{eq:convergence_deterministic_weights} holds.

    Recall that $\E\bigl[r_{T+1}\bigr] = \sigma^2 + \E[\|\bg_{T+1}(\bar\bb^{T, \rm{debias}}) - \bb^*\|^2]$ by definition of $\hbb^{T+1}$.
    By the triangle inequality for $\E[\sigma^2  + \|\cdot\|^2]^{1/2}$ we have
    \begin{equation}\label{eq:risk_evolution-1}
        \begin{aligned}
            &\big|
            \bigl(
            \sigma^2
            +
            \E[
            \|
            \bg_{T+1}(\hbb^{T, \rm{debias}})
            -\bb^*
        \|^2]
            \bigr)^{1/2}
            -
            \E\bigl[r_{T+1}\bigr]^{1/2}
            \big|
            \\&\le
            \E[\|
            \bg_{T+1}(\hbb^{T, \rm{debias}})
            -\bg_{T+1}(\bar\bb^{T, \rm{debias}})
            \|^2]^{1/2}
          \\&\le \zeta
            \E[\|
            \hbb^{T, \rm{debias}}
            -\bar\bb^{T, \rm{debias}}
            \|^2]^{1/2}
          \\&= \zeta
            \E[\|
            n^{-1}\bX^\top\bF(\bar\bW -\hbW)
            \|_{\rm F}^2]^{1/2}
          \\&\le \zeta
              \E[\|\bX^\top\bX/n\|_{\rm op}
                \|n^{-1/2}\bF(\bar\bW -\hbW)
            \|_{\rm F}^2]^{1/2}
            \end{aligned}
    \end{equation}
    
    since $\hbb^{T, \rm{debias}}$ and
    $\bar\bb^{T, \rm{debias}}$ only differ in the weights used
    for the columns of $\bF$.
    The random variable inside the final expectation
    converges to 0 in probability by the previous argument,
    and has uniformly bounded fourth moment by
    \Cref{lem:moment-bound-F-H,lem:opnorm-Ahat,lem:moment-bound-X}.
    By dominated convergence, the left-hand side converges to 0.
    Similarly, \Cref{thm:debias} gives
    $$
    \E[
    \|
    \hbb^{T, \rm{debias}}
    -(\bb^* + n^{-1/2}\bzeta_{\cdot,T})
    \|^2] \le C(\zeta,\gamma,T,v_0) /n
    $$
    where $\bzeta_{\cdot,T}\in\R^p$ is the random vector
    with components $(\bzeta_{j,T}){j\in [p]}$ from \Cref{thm:debias}.
    By the triangle inequality, similarly to the above display,
    \begin{equation}\label{eq:risk_evolution-2}
    \begin{aligned}
    &\big|
    \bigl(
    \sigma^2
    +
    \E[
    \|
    \bg_{T+1}(\hbb^{T, \rm{debias}})
    -\bb^*
\|^2]
    \bigr)^{1/2}
    -
    \bigl(
    \sigma^2
    +
    \E[
    \|
    \bg_{T+1}(\bb^* + n^{-1/2}\bzeta_{\cdot,T})
    -\bb^*
\|^2]
    \bigr)^{1/2}
    \big|
  \\&\le
  \zeta \E[\|\hbb^{T, \rm{debias}} - (\bb^*+n^{-1/2}\bzeta_{\cdot,T})\|^2]^{1/2}
  \\&\le C(\zeta,T,\gamma,v_0) / \sqrt n.
    \end{aligned}
\end{equation}
    Finally, by 
    \eqref{convergence_Zeta_j_W2_to_0},
    for each $j\in[p]$
    there exists a coupling $(Z_j,\zeta_{j,T})$
    with $Z_j\sim \sf N(0,\E[r_t])$ such that
    $\E[(Z_j-\zeta_{j,T})^2]^{1/2}\le
    \eqref{convergence_Zeta_j_W2_to_0} \to 0$
    where $\eqref{convergence_Zeta_j_W2_to_0}$
    the maximum over $j\in[p]$ of the 2-Wasserstein distance
    in the left-hand side of \eqref{convergence_Zeta_j_W2_to_0}.
    Since $\bg_{T+1}:\R^p\to\R^p$ is separable, acting componentwise
    with the $\zeta$-Lipschitz functions $\eta_j:\R\to\R$,
    \begin{equation} \label{eq:risk_evolution-3}
    \begin{aligned}
    &\Big|
    \Bigl(
    \sigma^2
    +
    \sum_{j=1}^p
    \E\Bigl[
    \Bigl(
    \eta_j\Bigl(b_j^* + \frac{Z_j}{\sqrt n}\Bigr)
    - b_j^*
    \Bigr)^2\Bigr]
    \Bigr)^{1/2}
    -
    \Bigl(
    \sigma^2
    +
    \E\Bigl[
    \Big\|
    \bg_{T+1}\Bigl(\bb^* + \frac{\bzeta_{\cdot,T}}{\sqrt n}\Bigr)
    -\bb^*
\Big\|^2\Bigr]
    \Bigr)^{1/2}
    \Big|
  \\&\le
  \zeta \E\Bigl[
  \sum_{j=1}^p
  \Bigl(
  \frac{Z_j}{\sqrt n}
  -
  \frac{\zeta_{jT}}{\sqrt n}
  \Bigr)^2
  \Bigr]^{1/2} \le \zeta
  \sqrt{\frac{p}{n}} \eqref{convergence_Zeta_j_W2_to_0}
  \to 0.
    \end{aligned}
\end{equation}
Combining the inequalities \eqref{eq:risk_evolution-1}-\eqref{eq:risk_evolution-3} gives  \eqref{eq:risk_evolution}.

The result 
$
    (
    \sigma^2
    +
    \E[\sum_{j=1}^p\{
            \eta_j(\hbb^{t, \rm{debias}}_j)
    -\bb_j^*
\}^2]
    )^{1/2}
    -
    \E[r_{T+1}]^{1/2}
    \to 0
    $
    is a direct consequence of \eqref{eq:risk_evolution} and the triangle inequality.
    This concludes the proof of \Cref{thm:state_ev}.
\end{proof}

\end{document}